\documentclass[journal]{IEEEtran}

\ifCLASSINFOpdf
\else
\fi

\usepackage{enumerate}
\usepackage{nicefrac}
\usepackage{rotating}
\usepackage{url}
\usepackage{csquotes}

\usepackage{amsthm}
\usepackage{amsmath}            %
\usepackage{amssymb}            %

\usepackage[10pt]{moresize}

\usepackage{tikz}

\newcommand{\cfbox}[2]{%
\setlength{\fboxrule}{1.4pt}%
    \colorlet{currentcolor}{.}%
    {\color{#1}%
    \fbox{\color{currentcolor}#2}}%
}

\newcommand{\cfboxR}[3]{%
\setlength{\fboxrule}{#1}%
    \colorlet{currentcolor}{.}%
    {\color{#2}%
    \fbox{\color{currentcolor}#3}}%
}

\definecolor{green1}{rgb}{.0,0,.7}
\definecolor{blue1}{rgb}{.3,.3,1}
\definecolor{red1}{rgb}{.35,.65,.25}
\definecolor{black}{rgb}{0,0,0}

\newcommand{\bx}{\mathbf{x}}

\def\algoname{\textbf{Proposed method}}

\newtheorem{claim}{Claim}

\usepackage[vlined,linesnumbered,ruled]{algorithm2e}
\SetKwInOut{Input}{$\,\,\,\,\,\,$Input}
\SetKwInOut{Output}{Output}
\SetKwComment{Comment}{}{}

\SetCommentSty{mycmtsty}

\newcommand{\mdG}[1]{  #1}  %

\begin{document}

\title{Removing Camera Shake via \\Weighted Fourier Burst Accumulation}

\author{Mauricio~Delbracio %
        and~Guillermo~Sapiro%
\IEEEcompsocitemizethanks{
\IEEEcompsocthanksitem This work was partially funded by: 
ONR, ARO, NSF, NGA, and AFOSR.
\IEEEcompsocthanksitem The authors are with the
Department of Electrical and Computer Engineering at Duke University.

e-mail: \{mauricio.delbracio, guillermo.sapiro\}@duke.edu

}%
}%

\maketitle

\begin{abstract}
Numerous recent approaches attempt to remove image blur due to camera shake, 
either with one or multiple input images, by
explicitly solving an inverse and inherently ill-posed deconvolution problem.
If the photographer takes a burst of images, a modality available in virtually 
all modern digital cameras, we show that it is possible to combine them
to get a clean sharp version.
This is done without explicitly solving any blur estimation and subsequent  inverse problem. 
The proposed algorithm is strikingly simple: it performs a weighted average in the Fourier domain,
with weights depending on the Fourier spectrum magnitude.
\mdG{The method can be seen as a generalization of the \emph{align and average} procedure, with a weighted average, 
motivated by hand-shake physiology and theoretically supported, taking place in the Fourier domain. }
The method's rationale is that camera shake has a random nature and therefore 
each image in the burst is generally blurred differently. 
Experiments with real camera data, and extensive comparisons, show that the proposed Fourier Burst Accumulation (FBA) algorithm
 achieves state-of-the-art results an order of magnitude faster, with simplicity for on-board implementation on camera phones.
 Finally, we also present experiments in real high dynamic range (HDR) scenes, showing how the method can be 
 straightforwardly extended to HDR photography.

\end{abstract}


\IEEEpeerreviewmaketitle

\section{Introduction}
\label{sec:introduction}

\IEEEPARstart{O}{ne} of the most challenging experiences in photography is taking images in low-light environments.
The basic principle of photography is the accumulation of photons in the sensor during a given exposure time. 
In general, the more photons reach the surface the better the quality of the final image, as the photonic noise is reduced. However, this
basic principle requires the photographed scene to be static and that there is no relative motion between the camera and the scene.
Otherwise,  the photons  will be accumulated in neighboring pixels, generating a loss of sharpness (blur). 
This problem is significant when shooting with hand-held cameras, the most popular photography device today, in dim light conditions. 

Under reasonable hypotheses, the 
camera shake can  be modeled mathematically as a convolution,
\begin{equation}
v = u \star k + n,
\label{eq:convolution}
\end{equation}
where $v$ is the noisy blurred observation, $u$ is the latent sharp image, $k$ is an unknown blurring kernel and $n$ is additive white noise.
For this model to be accurate, the camera movement has to be essentially a rotation in its optical axis   with negligible
in-plane rotation, e.g.,~\cite{whyte2012non}. 
The kernel $k$ results from several blur sources: light diffraction due to the finite aperture,
out-of-focus, light integration in the photo-sensor,  and relative motion between the camera and 
the scene during the exposure. To get enough photons per pixel in a typical low light scene, the camera needs 
to capture light for a period of tens to hundreds of milliseconds. 
In such a situation (and assuming that the scene is static and the user/camera has correctly set the focus),
the dominant contribution to the blur kernel is the camera shake ---mostly due to hand tremors.

Current cameras can take a burst of images, this being popular also in camera phones. This has been exploited in several approaches for 
accumulating photons in the different images and then forming an image with less noise (mimicking a longer exposure 
time a posteriori, see e.g.,~\cite{buades2009note}). However, this principle is disturbed if the images in the burst are blurred. 
The classical mathematical formulation  as a multi-image deconvolution, seeks to
solve an inverse problem where the unknowns are the multiple blurring operators and the underlying sharp image.  
This procedure, although it produces good results~\cite{zhang2013multi},  is computationally very expensive, 
and very sensitive to a good estimation of the blurring kernels, an extremely challenging task by itself. 
Furthermore, since the inverse problem is ill-posed it relies on priors either, or both, for the calculus of the blurs and the latent sharp image.

Camera shake originated from hand tremor vibrations is essentially random~\cite{carignan2010quantifying,gavant2011physiological,xiao2006camera}. 
This implies that the movement of the camera in an individual image of the burst is  independent 
of the movement in another one.  Thus, the blur in one frame will be different from the one in another image of the burst. 
Our work is built on this basic principle. We present an algorithm that aggregates a burst of images (or more than one burst for high dynamic range), taking what is less blurred of each frame 
to build an image that is sharper and less noisy than all the images in the burst.
The algorithm is straightforward to implement and conceptually simple. It takes as input a series of registered images and computes a weighted average of the Fourier coefficients of the images in the burst. 
Similar ideas have been explored by Garrel et al.~\cite{garrel2012highly} in the context of astronomical images,
where a sharp clean image is produced from a video affected by atmospheric turbulence blur.

With the availability of accurate gyroscope and accelerometers in, for example, phone cameras, the registration can be obtained ``for free,''
 rendering the whole algorithm very efficient for on-board implementation. Indeed, one could envision a mode transparent to the user, where 
 every time he/she takes a picture, it is actually a burst or multiple bursts with different parameters each. 
 The set is then processed on the fly and only the result is saved. Related modes are 
 currently available in ``permanent open'' cameras.
The explicit computation of the blurring kernel, as commonly done in the literature, is completely avoided. This is not only an unimportant 
hidden variable for the task at hand, but as mentioned above, still leaves the ill-posed and computationally very expensive task of solving the 
inverse problem.

Evaluation through synthetic and real experiments shows that the final image quality is significantly improved. This is done without explicitly 
performing deconvolution, which generally introduces artifacts and also a significant overhead.
Comparison to state-of-the-art multi-image deconvolution algorithms shows that our approach produces similar or better results while being 
orders of magnitude faster and simpler.
The proposed algorithm does not assume any prior on the latent image; exclusively relying  on the randomness of hand tremor. 

A preliminary short version of this work was  submitted to a conference~\cite{delbracio2015burst}. The present version incorporates
a more detailed analysis of the burst aggregation algorithm and its implementation. We also introduce a detailed comparison to lucky image 
selection techniques~\cite{fried1978probability},  
where from a series of short exposure images, the sharpest ones are selected, aligned and averaged into a single frame.
Additionally, we present experiments in real high dynamic range (\textsc{hdr}) scenes showing how the method can be  extended to 
hand-held \textsc{hdr} photography.

The remaining of the paper is organized as follows.
Section~\ref{sec:relatedWork} discusses the related work and the substantial differences to what we propose.
Section~\ref{sec:fba} explains how a burst can be combined in the Fourier domain to recover a sharper image, while in 
Section~\ref{sec:fbaAnalysis} we present an empirical analysis of the algorithm performance through the simulation of camera
shake kernels. Section~\ref{sec:algodetails}  details the algorithm implementation and in Section~\ref{sec:results} we present results 
of the proposed aggregation procedure in real data, comparing
the algorithm to state-of-the-art  multi-image deconvolution methods. 
Section~\ref{sec:hdr} presents experiments in real high dynamic range (HDR) scenes, 
showing how the method can be  straightforwardly extended to HDR photography.
Conclusions are finally summarized in Section~\ref{sec:conclusions}.

\section{Related Work}
\label{sec:relatedWork}

Removing  camera shake blur is one of the most challenging problems in image processing.
Although in the last decade several image restoration algorithms have emerged giving outstanding performance, 
their success is still very dependent on the scene. Most image deblurring algorithms 
cast the problem as a deconvolution with either a known (non-blind) or an unknown blurring kernel (blind).
See e.g., the review by Kundur and Hatzinakos~\cite{kundur1996blind}, where a discussion of 
the most classical methods is presented.\\

\noindent \textbf{Single image blind deconvolution.}
Most blind deconvolution algorithms  try to estimate the latent image without any other input than the 
noisy blurred image itself. A representative work is the one by Fergus {\it et al.}~\cite{fergus2006removing}.  
This variational method sparked many competitors  seeking to combine natural image priors, assumptions 
on the blurring operator,  and complex optimization frameworks, to simultaneously estimate both the blurring 
kernel and the sharp image~\cite{shan2008high,cai2009blindsingle, krishnan2011blind,xu2013unnatural,michaeli2014blind}.
Fergus et al.~\cite{fergus2006removing} approximated the  heavy-tailed distribution of the gradient of natural
images using a Gaussian mixture. In \cite{shan2008high}, the authors exploited the use of  sparse priors for both the
sharp image and the blurring kernel. Cai et al.~\cite{cai2009blindsingle} proposed a joint optimization framework,
that simultaneously maximizes the sparsity of the blur kernel and the sharp image in a curvelet and a framelet 
systems respectively. Krishnan et al. ~\cite{krishnan2011blind} introduced as a prior the ratio
between the $\ell_1$ and the $\ell_2$ norms on the high frequencies of an image. This normalized sparsity measure
gives low cost for the sharp image.
In~\cite{xu2013unnatural} the authors discussed \emph{unnatural}  sparse representations of the image that
 mainly retain edge information. This representation is used to estimate the motion kernel.
Michaeli and Irani~\cite{michaeli2014blind} recently proposed to use as an image prior the recurrence of small 
natural image patches across different scales. The idea is that the cross-scale patch occurrence should be 
maximal for sharp images.

Several attempt to first estimate the degradation operator and then applying a non-blind deconvolution algorithm.
For instance, \cite{cho2009fast} accelerates the kernel estimation step by using fast image filters for explicitly detecting and  
restoring strong edges in the latent sharp image.
Since the blurring kernel has typically a very small support, the kernel estimation problem is better conditioned than estimating 
the kernel and the sharp image together. In~\cite{levin2009understanding,levin2011efficient},
the authors concluded that it is better to  first solve a maximum a posteriori estimation of the kernel than 
of the latent image and the kernel simultaneously.
 However, even in non-blind deblurring, i.e., when the blurring kernels are known, the problem is 
generally ill-posed, because the blur introduces zeros in the frequency domain. 
Thereby avoiding explicit inversion, as here proposed,  becomes critical.\\

\noindent \textbf{Multi-image blind deconvolution.}
Two or more input images can improve the estimation of both the underlying image 
and the blurring kernels.  Rav-Acha  and Peleg~\cite{rav2005two} claimed that 
\emph{``Two motion-blurred images are better than one,''}  if the direction of the 
blurs are different.
In~\cite{yuan2007image} the authors proposed to capture two photographs:  
one having a short exposure time, noisy but sharp, and one with a long exposure, 
blurred but with low noise. The two acquisitions are complementary, and the sharp 
one is used to estimate the motion kernel of the blurred one.

Close to our work are papers on multi-image blind deconvolution~\cite{zhang2013multi,cai2009blind,chen2008robust,
sroubek2012robust,zhu2012deconvolving}. 
In~\cite{cai2009blind} the authors showed that given multiple observations, the sparsity of the image under 
a tight frame is a good measurement of the clearness of the recovered image. Having multiple input images
improves the accuracy of identifying the motion blur kernels, reducing the illposedness of the problem.
Most of these multi-image algorithms introduce cross-blur penalty functions
between image pairs. However this has the problem of growing combinatorially with the number 
of images in the burst.
This idea is extended in~\cite{zhang2013multi} using a Bayesian framework for coupling all the unknown
 blurring kernels and the latent image in a unique prior. Although this prior has numerous good mathematical
 properties, its optimization is very slow. The algorithm  produces very good results but it may take several
 minutes or even hours  for a typical burst of 8-10 images of several megapixels.
The very recent work by Park and  Levoy~\cite{park2014gyro} relies on an attached gyroscope, now present
 in many phones and tablets, to align all the input images and to get an estimation of the blurring kernels. 
 Then, a multi-image non-blind deconvolution algorithm is applied.

By taking a burst of images, the multi-image deconvolution problem becomes less ill-posed allowing the use of simpler priors.
This is explored in~\cite{ito2014blurburst} where the authors adopted a total variation prior on the underlying sharp image. 

All these papers propose kernel estimation and to solve an inverse problem of image deconvolution. 
The main inconvenience of tackling this problem as a deconvolution, on top of the computational burden, 
is that if the convolution model is not accurate or the kernel is not accurately estimated, 
the restored image will contain strong artifacts (such as ringing). \\

\noindent \textbf{Lucky imaging.}
A popular technique in astronomical photography, known as \emph{lucky imaging} or \emph{lucky exposures}, 
is to take a series of thousands of short-exposure images and then select and fuse only the sharper ones~\cite{law2006lucky}. 
Fried~\cite{fried1978probability} showed that the probability of getting a sharp lucky short-exposure 
image through turbulence follows a negative exponential. Thus, when the captured series or video is 
sufficiently long, there will exist such a frame with high probability.  

Classical selection techniques are based on the brightness of the 
brightest speckle~\cite{law2006lucky}.
The number of selected frames is chosen to optimize the tradeoff
between sharpness and signal-to-noise ratio required in the application.

Others propose to measure the local sharpness from the norm of the gradient or the image 
Laplacian~\cite{john2005multiframe,aubailly2009automated,joshi2010seeing,haro2012photographing}.
Joshi and Cohen~\cite{joshi2010seeing} engineered a weighting scheme 
to balance noise reduction and sharpness preservation.
The sharpness is measured through the intensity of the  image Laplacian. They also proposed a local selectivity
 weight to reflect the fact that more averaging should be done in smooth regions.
Haro and colleagues~\cite{haro2012photographing} explored similar ideas to
fusion different acquisitions of painting images. The weights for combining the input images rely on a 
local sharpness measure based on the energy of the image gradient.
The main disadvantage of these approaches is that they only rely on sharpness measures 
(which by the way is not necessarily trivial to estimate)  and do not 
profit the fact that camera shake blur can be in different directions in different frames.

Garrel et al.~\cite{garrel2012highly} introduced a selection scheme for astronomic images, based on the relative 
strength of signal for each spatial frequency in the Fourier domain. 
From a series of realistic image simulations, the authors showed
that this procedure produces images of higher resolution and better signal to noise ratio
than traditional lucky image fusion schemes.
This procedure makes a much more efficient use of the information contained in each frame.

Our paper is based on similar ideas but in a different scenario. 
The idea is to fuse all the images in the burst without 
explicitly estimating the blurring kernels and subsequent inverse problem approach, 
but taking the information that is less degraded from each image in the burst.
The estimation of the ``less degraded'' information is done in a trivial fashion as explained next.
The entire algorithm is based on physical properties of the camera (hand) shake and not on priors 
or assumptions on the image and/or kernel.

\section{Fourier Burst Accumulation}
\label{sec:fba}
\subsection{Rationale}

\begin{figure}
\centering

\begin{minipage}[c]{.258\columnwidth}
\begin{center}

\begin{tikzpicture}
    \node[anchor=north west,inner sep=0] (image) at (0,0) {\includegraphics[width=\textwidth]{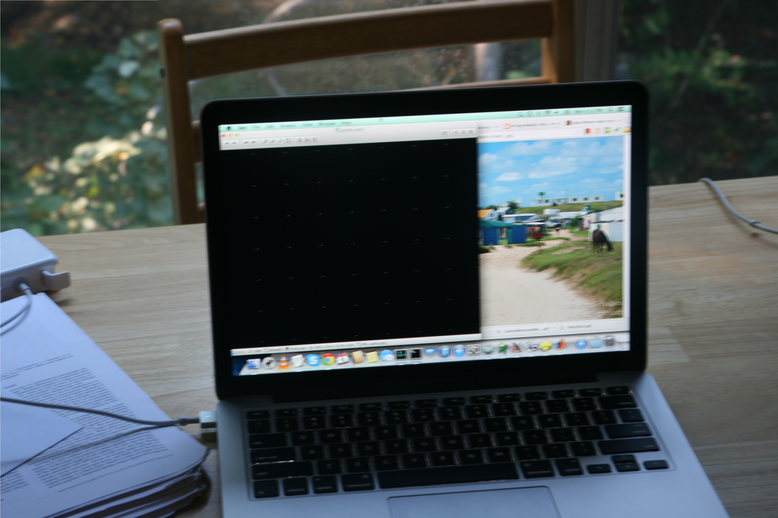}};
    \begin{scope}[x={(image.north east)},y={(image.south west)}]
        \draw[red, thick] (0.3971,0.3971) rectangle (0.6029,0.6029); %
    \end{scope}
\end{tikzpicture}

\vspace{.15em}

\begin{tikzpicture}
    \node[anchor=north west,inner sep=0] (image) at (0,0) {\cfbox{red}{\includegraphics[width=.96\textwidth]{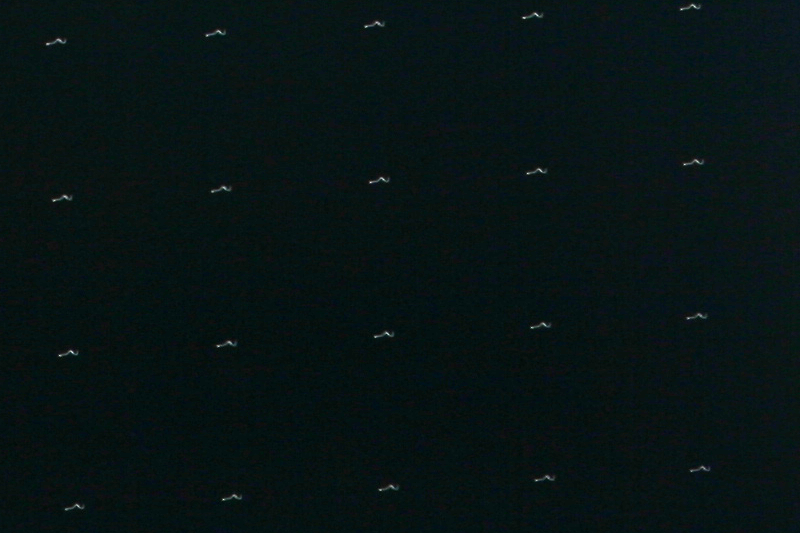}}};
    \begin{scope}[x={(image.north east)},y={(image.south west)}]
        \draw[blue1, thick] (0.4363,0.5619) rectangle (0.5238,0.6932); %
    \end{scope}
\end{tikzpicture}
\end{center}

\end{minipage}
\hspace{-0.5em}
\begin{minipage}[c]{.734\columnwidth}
\cfboxR{0.8pt}{blue1}{\includegraphics[width=.150\textwidth]{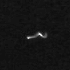}}\hspace{-.2em}
\cfboxR{0.8pt}{blue1}{\includegraphics[width=.150\textwidth]{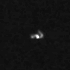}}\hspace{-.2em}
\cfboxR{0.8pt}{blue1}{\includegraphics[width=.150\textwidth]{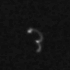}}\hspace{-.2em}
\cfboxR{0.8pt}{blue1}{\includegraphics[width=.150\textwidth]{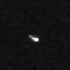}}\hspace{-.2em}
\cfboxR{0.8pt}{blue1}{\includegraphics[width=.150\textwidth]{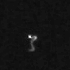}}\hspace{-.2em}
\cfboxR{0.8pt}{blue1}{\includegraphics[width=.150\textwidth]{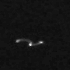}}

\vspace{0.1em}

\cfboxR{0.8pt}{blue1}{\includegraphics[width=.150\textwidth]{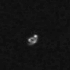}}\hspace{-.2em}
\cfboxR{0.8pt}{blue1}{\includegraphics[width=.150\textwidth]{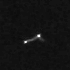}}\hspace{-.2em}
\cfboxR{0.8pt}{blue1}{\includegraphics[width=.150\textwidth]{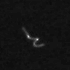}}\hspace{-.2em}
\cfboxR{0.8pt}{blue1}{\includegraphics[width=.150\textwidth]{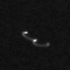}}\hspace{-.2em}
\cfboxR{0.8pt}{blue1}{\includegraphics[width=.150\textwidth]{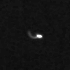}}\hspace{-.2em}
\cfboxR{0.8pt}{blue1}{\includegraphics[width=.150\textwidth]{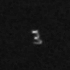}}

\vspace{0.1em}

\cfboxR{0.8pt}{blue1}{\includegraphics[width=.150\textwidth]{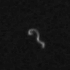}}\hspace{-.2em}
\cfboxR{0.8pt}{blue1}{\includegraphics[width=.150\textwidth]{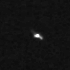}}\hspace{-.2em}
\cfboxR{0.8pt}{blue1}{\includegraphics[width=.150\textwidth]{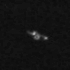}}\hspace{-.2em}
\cfboxR{0.8pt}{blue1}{\includegraphics[width=.150\textwidth]{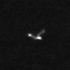}}\hspace{-.2em}
\cfboxR{0.8pt}{blue1}{\includegraphics[width=.150\textwidth]{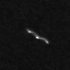}}\hspace{-.2em}
\cfboxR{0.8pt}{blue1}{\includegraphics[width=.150\textwidth]{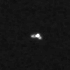}}

\end{minipage}

\caption{
When the camera is set to a burst mode, several photographs are captured sequentially. 
Due to the random nature of  hand tremor,  the camera shake blur is mostly 
independent from one frame to the other. An image consisting
of white dots was photographed with a \textsc{dslr} handheld camera to depict
the camera motion kernels. The kernels are mainly unidimensional regular trajectories 
that are not completely random (perfect random walk) nor uniform. 
}
\label{fig:realKernelsLaptop}
\end{figure}

\begin{figure*}
\centering

\begin{minipage}{.5em}
\begin{sideways}

{\ssmall
$\text{imag}(\hat{k})$\hspace{1.2em}
$\text{real}(\hat{k})$\hspace{2em}
$k$\hspace{4em}
$v$\hspace{2em}
}

\end{sideways}
\end{minipage}
\begin{minipage}{.0528\textwidth}
\centering

{\tiny input 1} \vspace{.12em}

\cfboxR{0.8pt}{black}{\includegraphics[width=.93\textwidth]{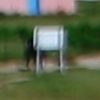}} \vspace{-1.05em}

\cfboxR{0.8pt}{black}{\includegraphics[width=.93\textwidth]{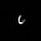}} \vspace{-1.05em}

\cfboxR{0.8pt}{black}{\includegraphics[width=.93\textwidth]{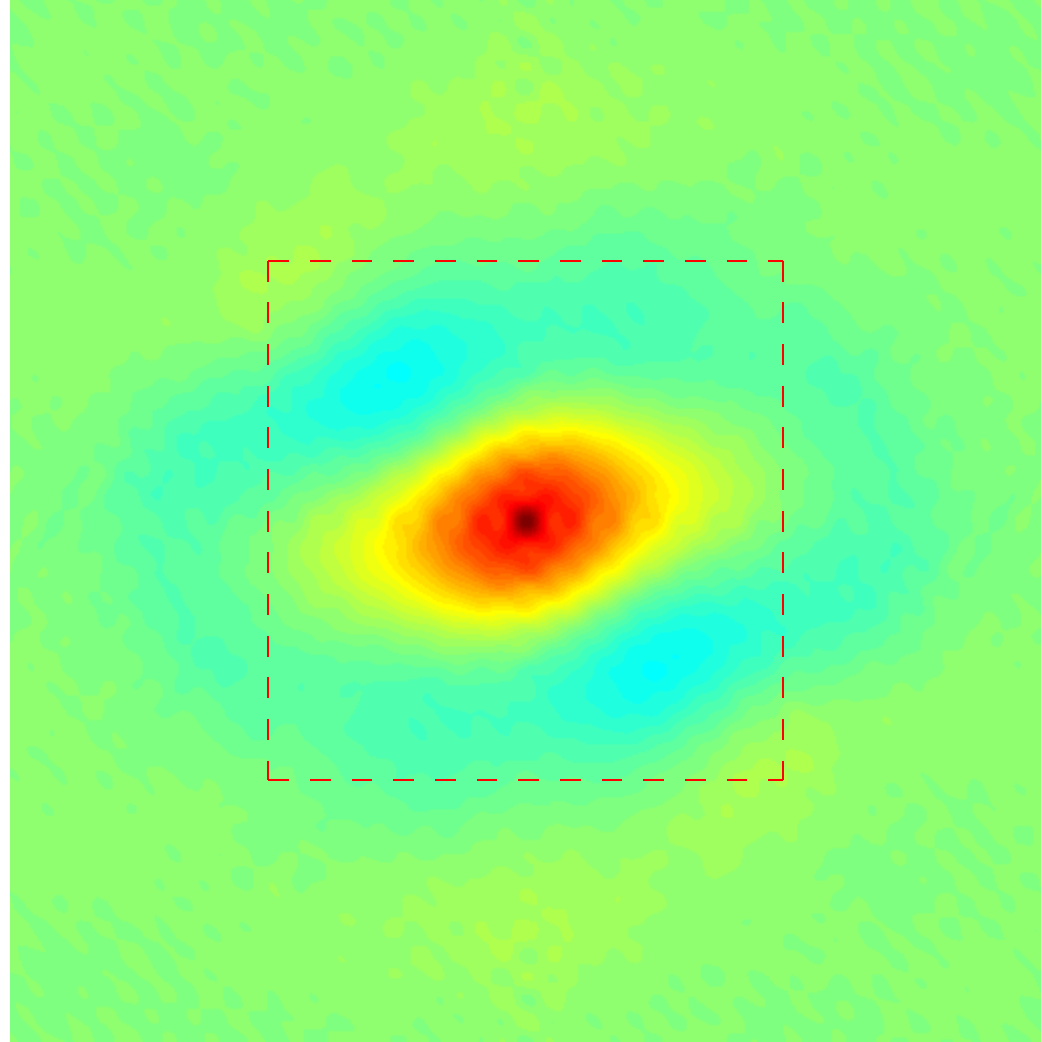}} \vspace{-1.05em}

\cfboxR{0.8pt}{black}{\includegraphics[width=.93\textwidth]{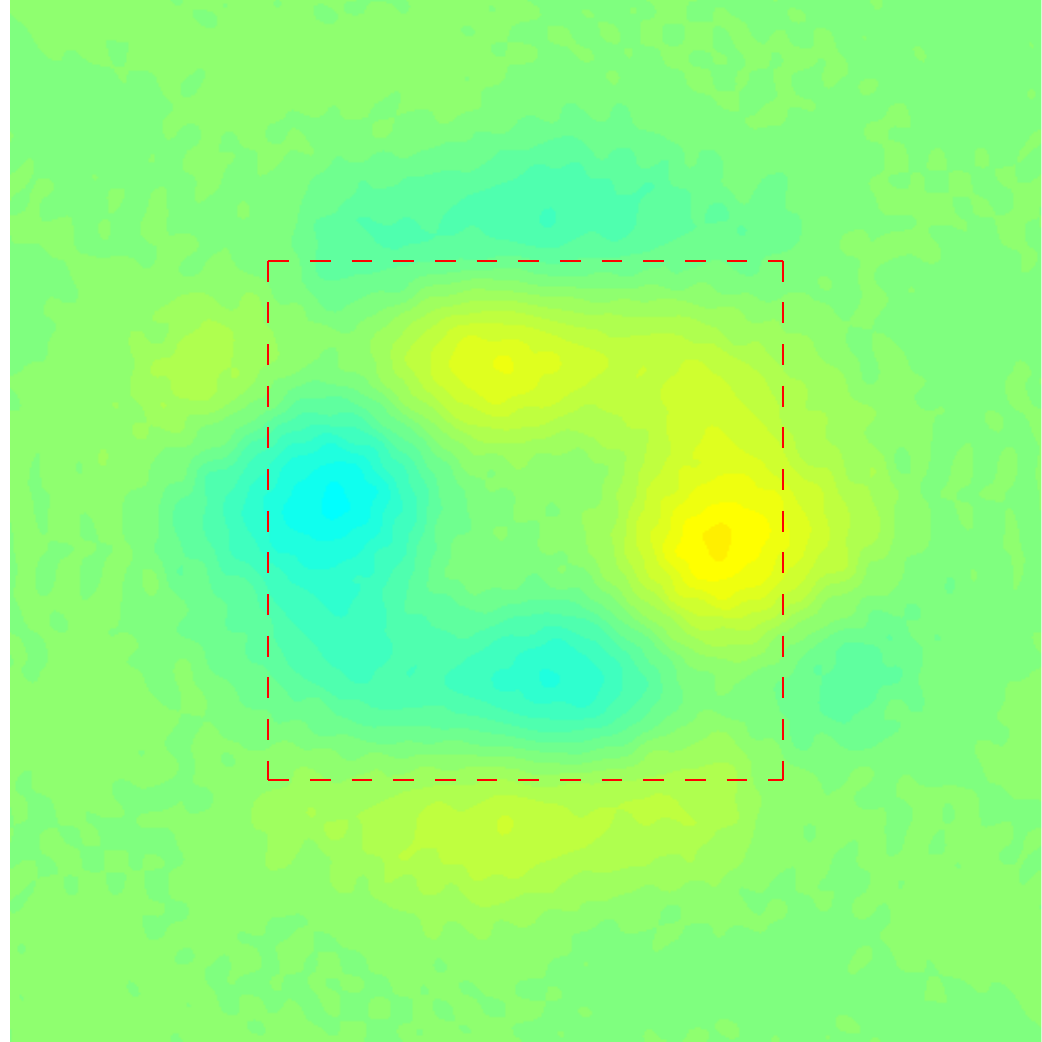}}

\end{minipage}
\hspace{-.27em}%
\begin{minipage}{.0528\textwidth}
\centering
{\tiny input 2} \vspace{.12em}

\cfboxR{0.8pt}{black}{\includegraphics[width=.93\textwidth]{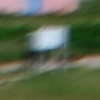}} \vspace{-1.05em}

\cfboxR{0.8pt}{black}{\includegraphics[width=.93\textwidth]{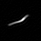}} \vspace{-1.05em}

\cfboxR{0.8pt}{black}{\includegraphics[width=.93\textwidth]{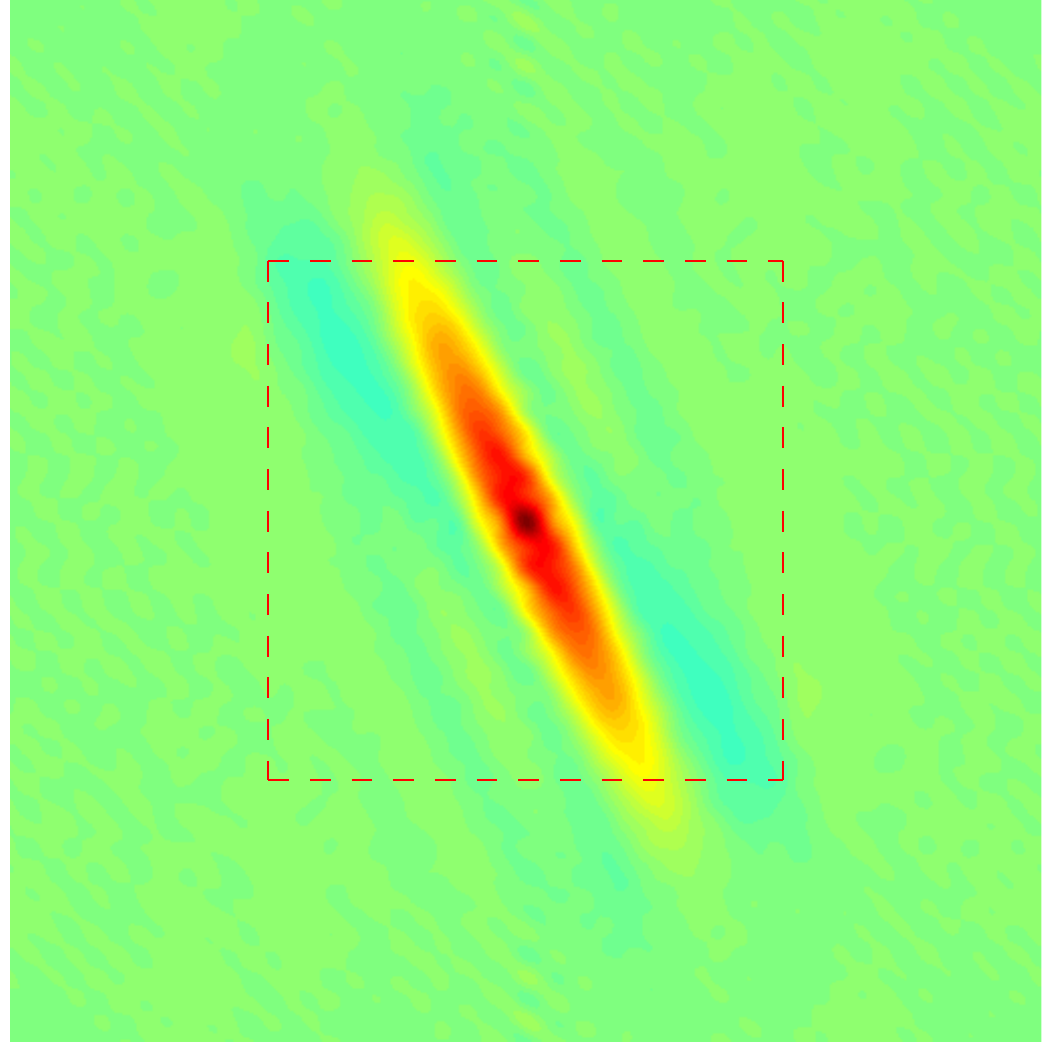}} \vspace{-1.05em}

\cfboxR{0.8pt}{black}{\includegraphics[width=.93\textwidth]{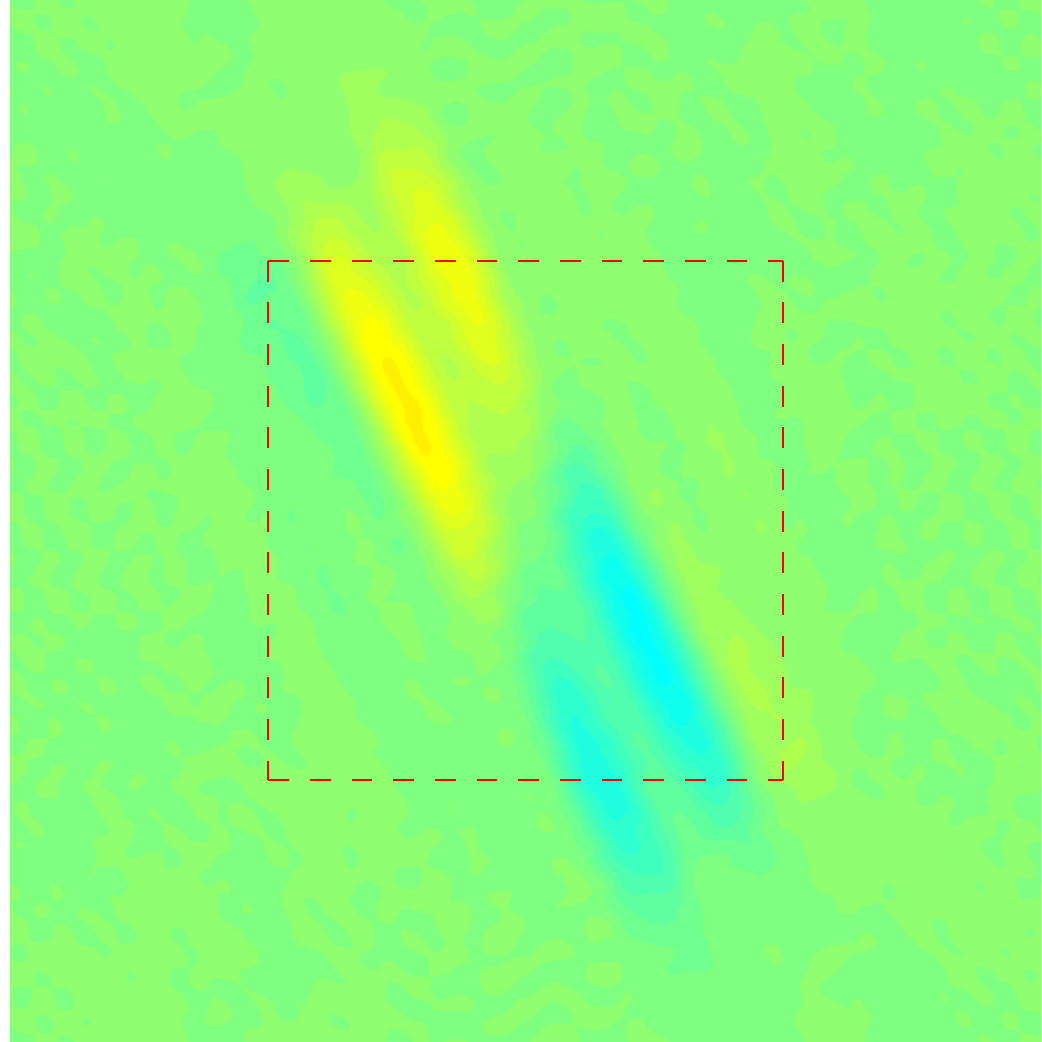}}

\end{minipage}
\hspace{-.60em}
\begin{minipage}{.0528\textwidth}
\centering
{\tiny input 3} \vspace{.12em}

\cfboxR{0.8pt}{black}{\includegraphics[width=.93\textwidth]{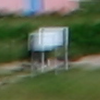}} \vspace{-1.05em}

\cfboxR{0.8pt}{black}{\includegraphics[width=.93\textwidth]{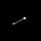}} \vspace{-1.05em}

\cfboxR{0.8pt}{black}{\includegraphics[width=.93\textwidth]{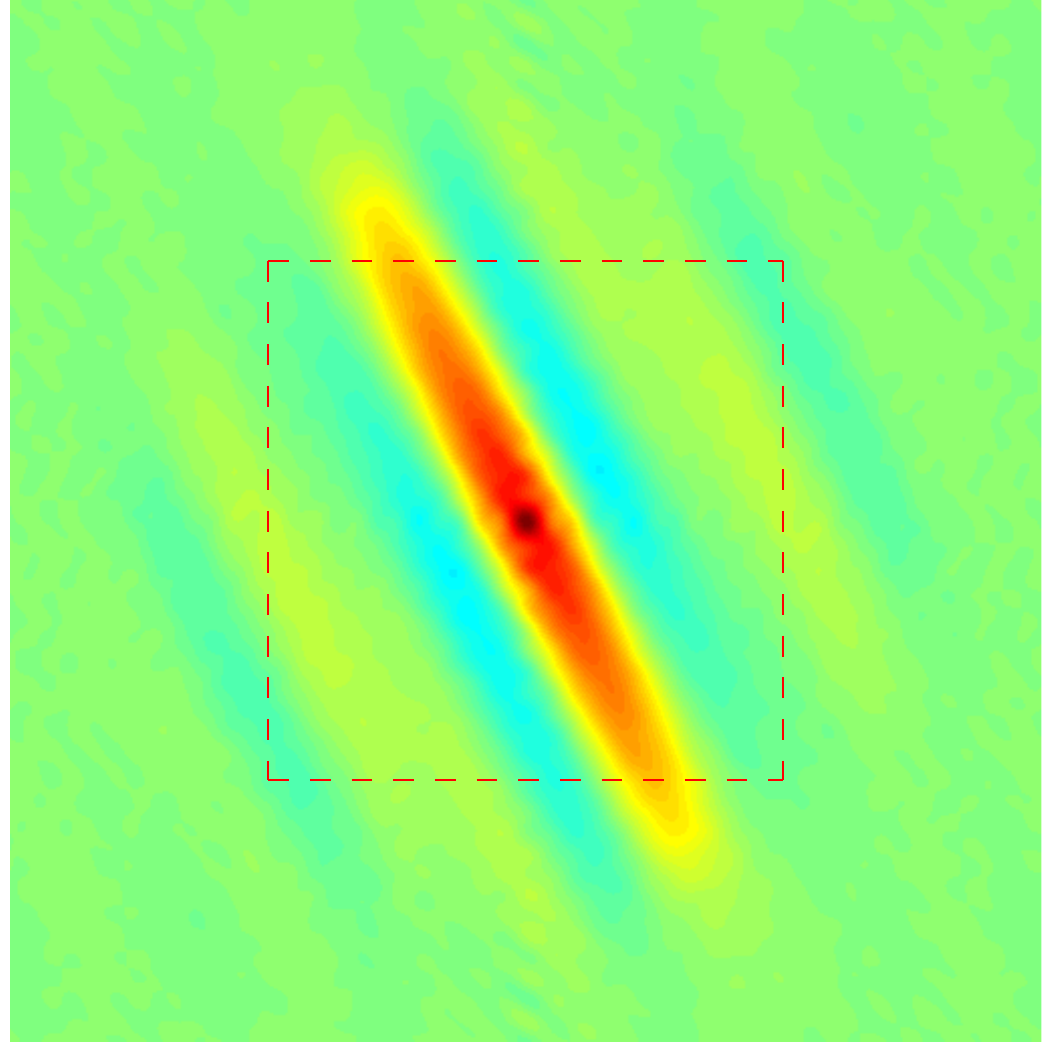}} \vspace{-1.05em}

\cfboxR{0.8pt}{black}{\includegraphics[width=.93\textwidth]{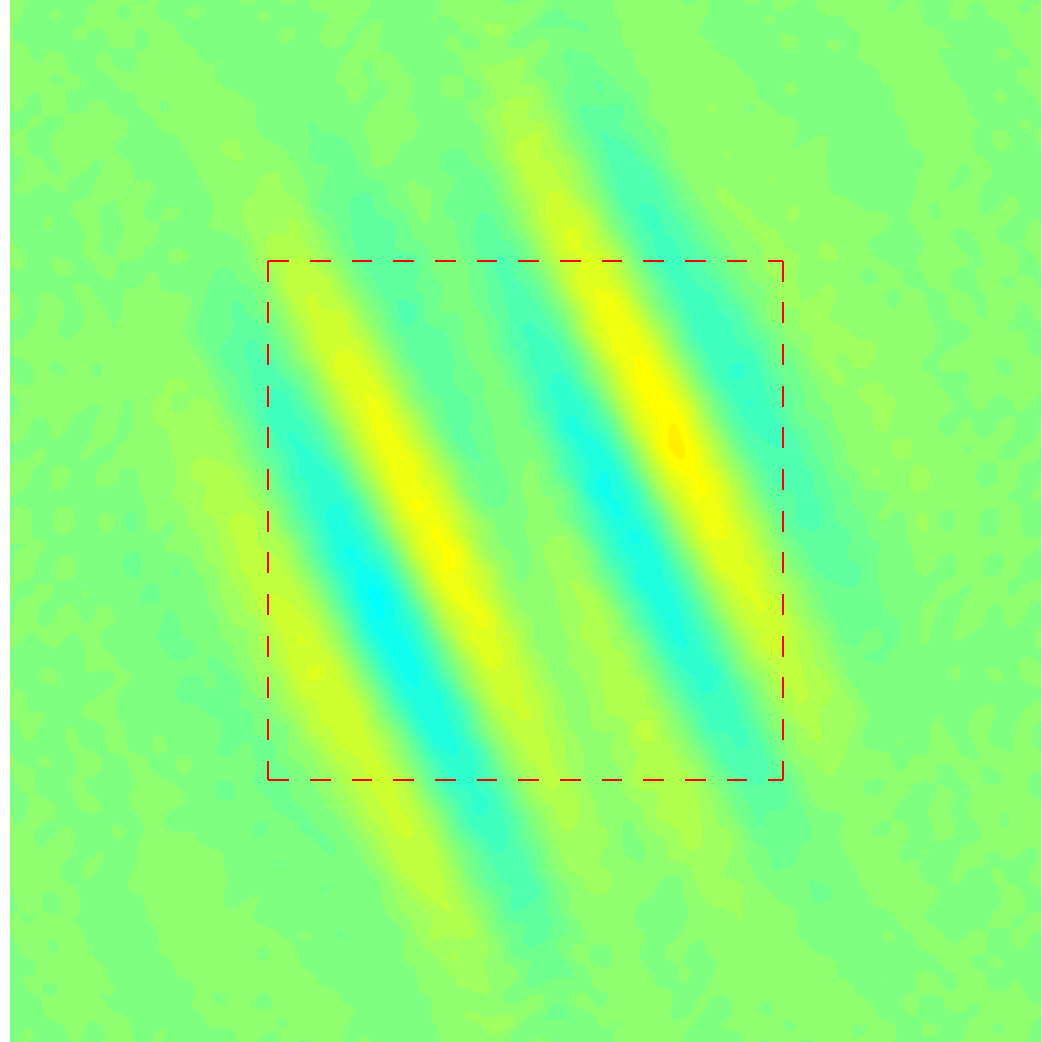}}

\end{minipage}
\hspace{-.60em}
\begin{minipage}{.0528\textwidth}
\centering
{\tiny input 4} \vspace{.12em}

\cfboxR{0.8pt}{black}{\includegraphics[width=.93\textwidth]{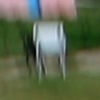}} \vspace{-1.05em}

\cfboxR{0.8pt}{black}{\includegraphics[width=.93\textwidth]{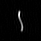}} \vspace{-1.05em}

\cfboxR{0.8pt}{black}{\includegraphics[width=.93\textwidth]{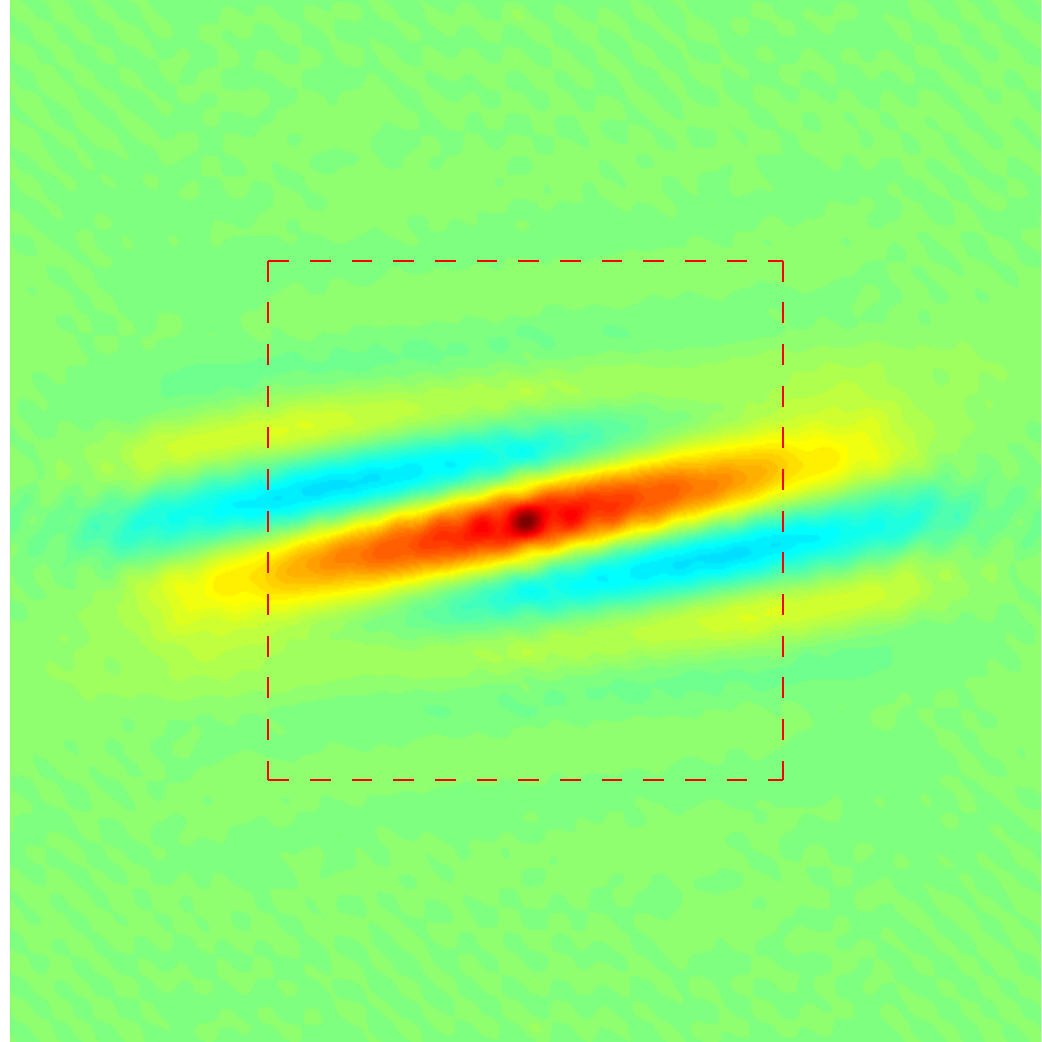}} \vspace{-1.05em}

\cfboxR{0.8pt}{black}{\includegraphics[width=.93\textwidth]{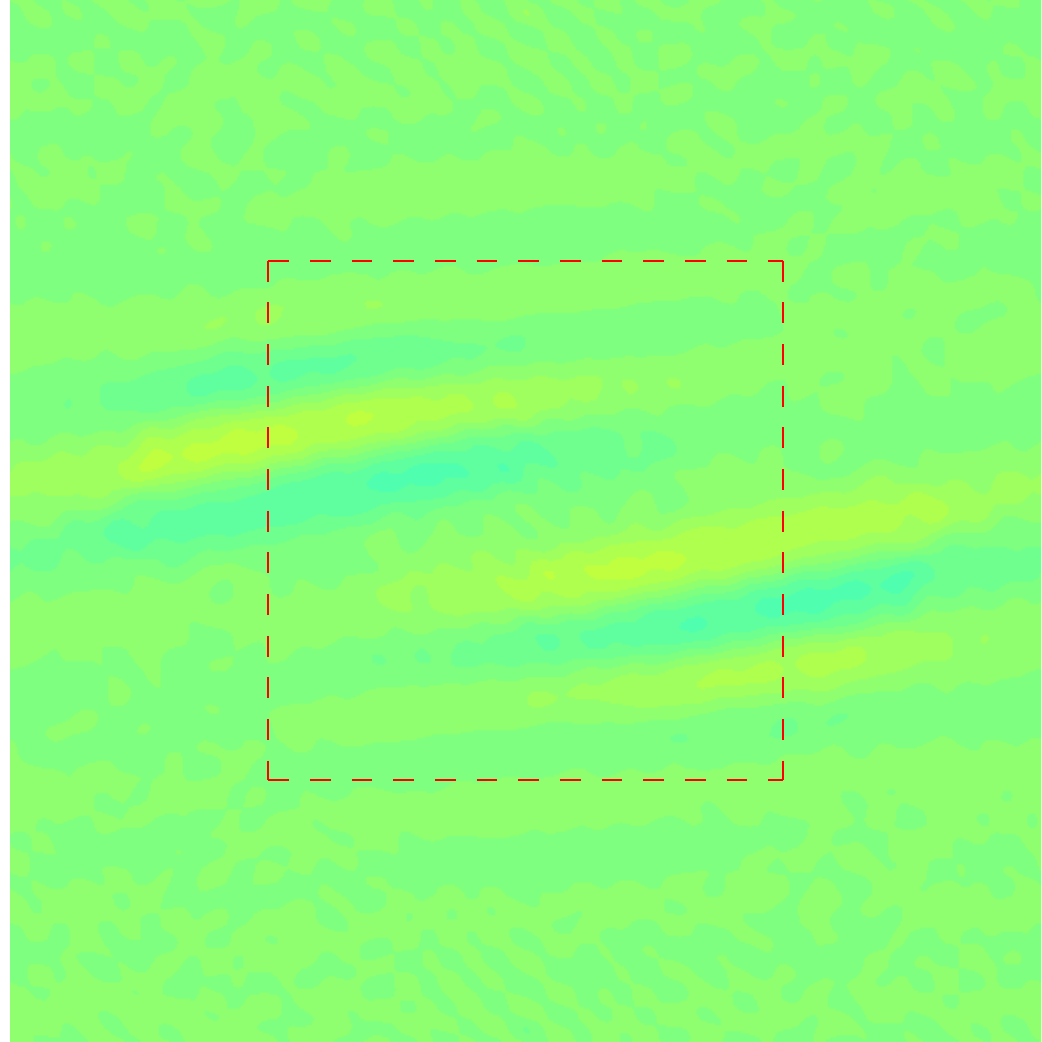}}

\end{minipage}
\hspace{-.60em}
\begin{minipage}{.0528\textwidth}
\centering
{\tiny input 5} \vspace{.12em}

\cfboxR{0.8pt}{black}{\includegraphics[width=.93\textwidth]{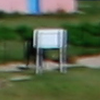}} \vspace{-1.05em}

\cfboxR{0.8pt}{black}{\includegraphics[width=.93\textwidth]{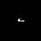}} \vspace{-1.05em}

\cfboxR{0.8pt}{black}{\includegraphics[width=.93\textwidth]{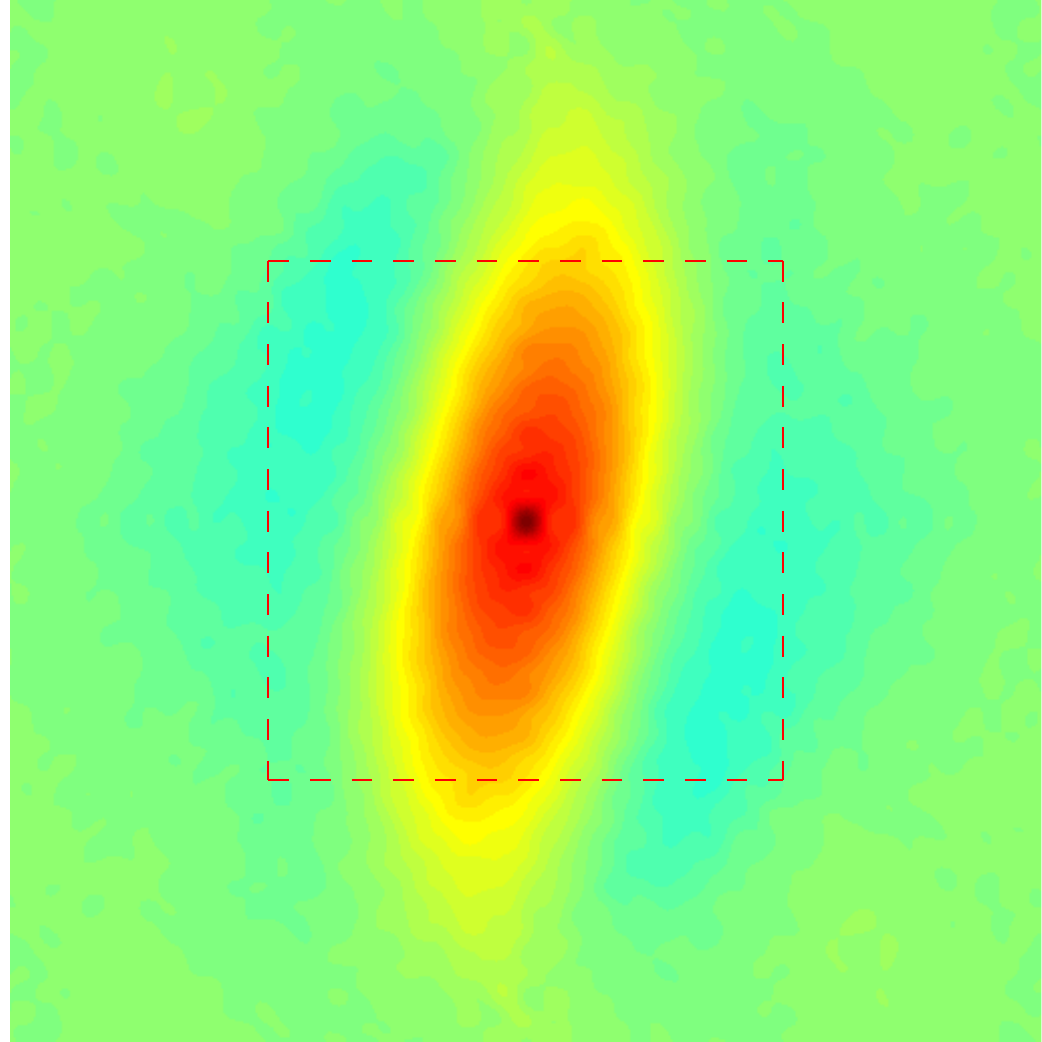}} \vspace{-1.05em}

\cfboxR{0.8pt}{black}{\includegraphics[width=.93\textwidth]{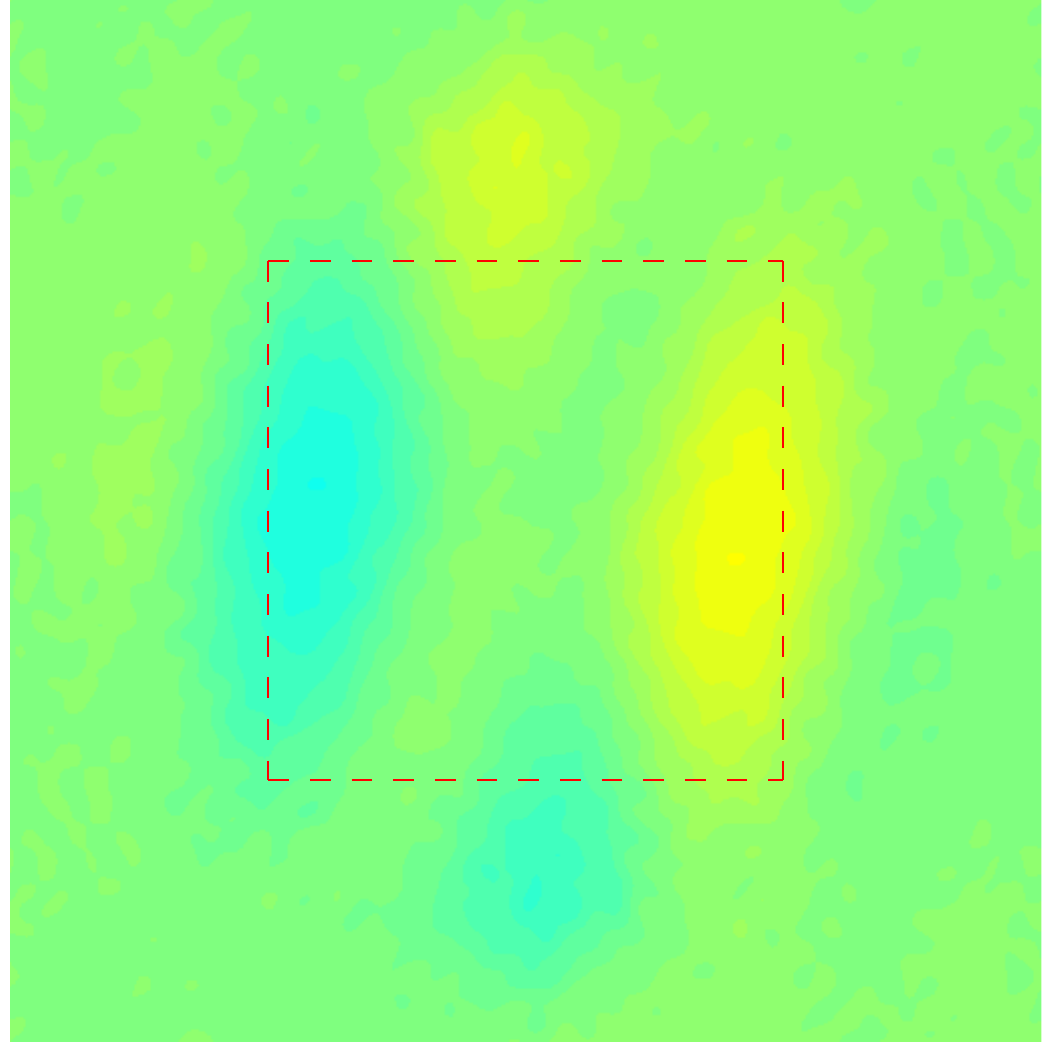}}

\end{minipage}
\hspace{-.60em}
\begin{minipage}{.0528\textwidth}
\centering
{\tiny input 6} \vspace{.12em}

\cfboxR{0.8pt}{black}{\includegraphics[width=.93\textwidth]{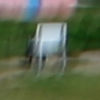}} \vspace{-1.05em}

\cfboxR{0.8pt}{black}{\includegraphics[width=.93\textwidth]{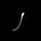}} \vspace{-1.05em}

\cfboxR{0.8pt}{black}{\includegraphics[width=.93\textwidth]{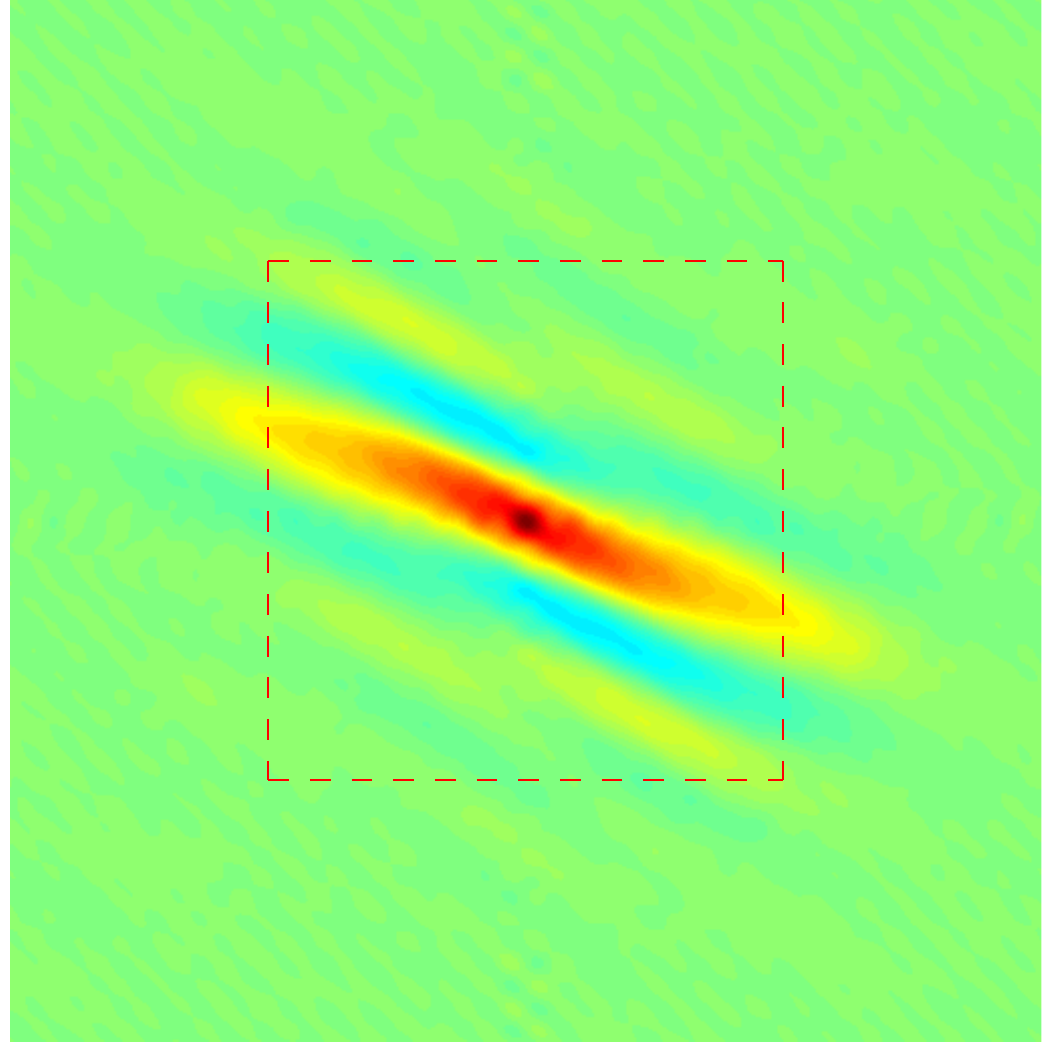}} \vspace{-1.05em}

\cfboxR{0.8pt}{black}{\includegraphics[width=.93\textwidth]{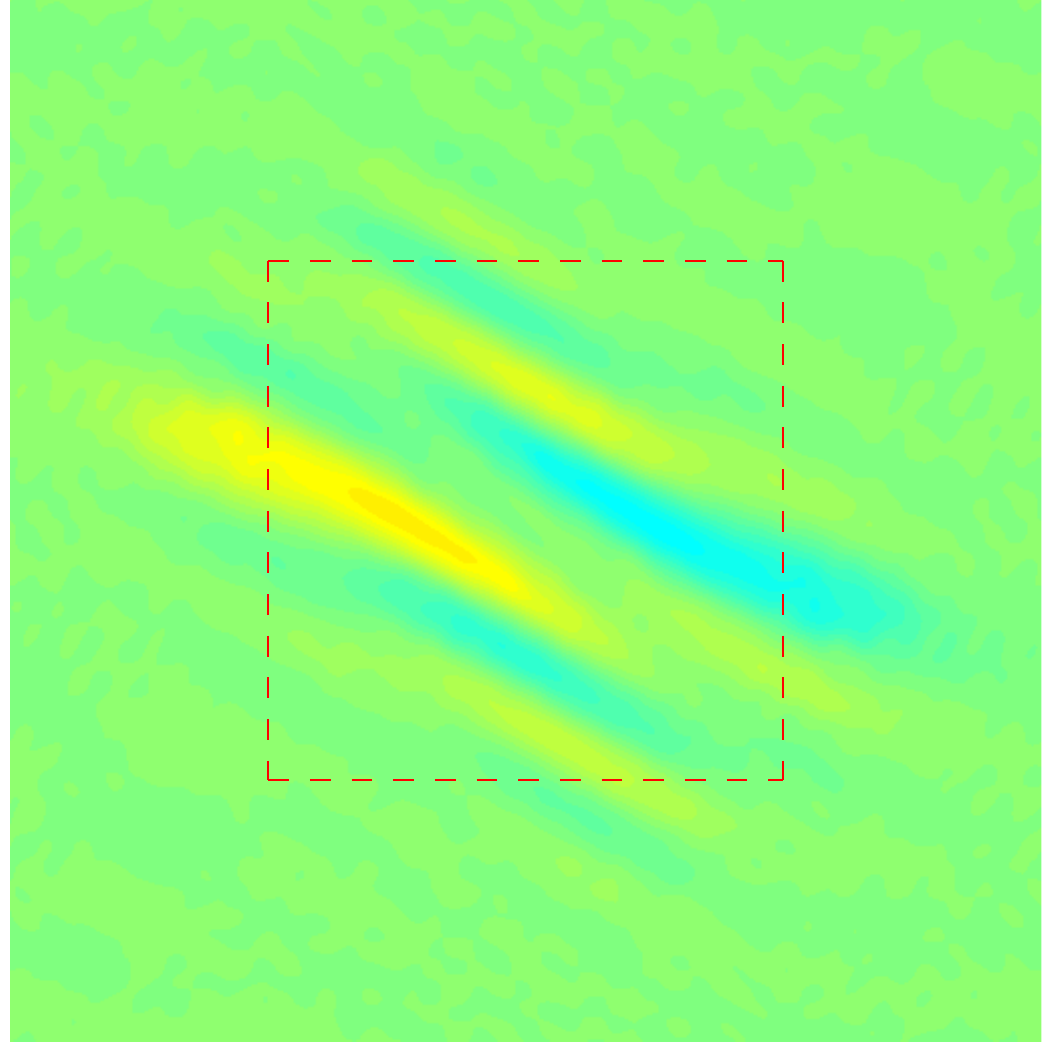}}
\end{minipage}
\hspace{-.60em}
\begin{minipage}{.0528\textwidth}
\centering
{\tiny input 7} \vspace{.12em}

\cfboxR{0.8pt}{black}{\includegraphics[width=.93\textwidth]{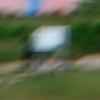}} \vspace{-1.05em}

\cfboxR{0.8pt}{black}{\includegraphics[width=.93\textwidth]{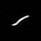}} \vspace{-1.05em}

\cfboxR{0.8pt}{black}{\includegraphics[width=.93\textwidth]{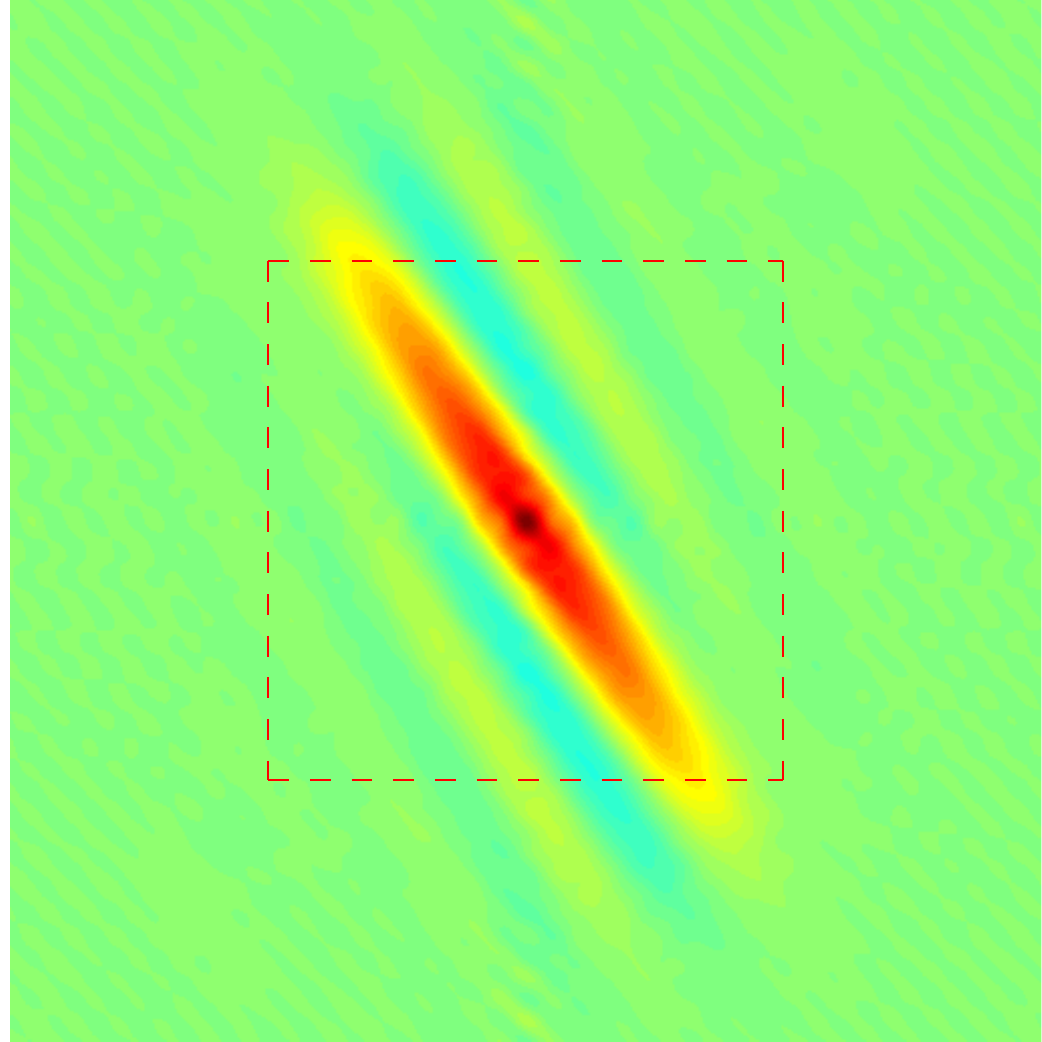}} \vspace{-1.05em}

\cfboxR{0.8pt}{black}{\includegraphics[width=.93\textwidth]{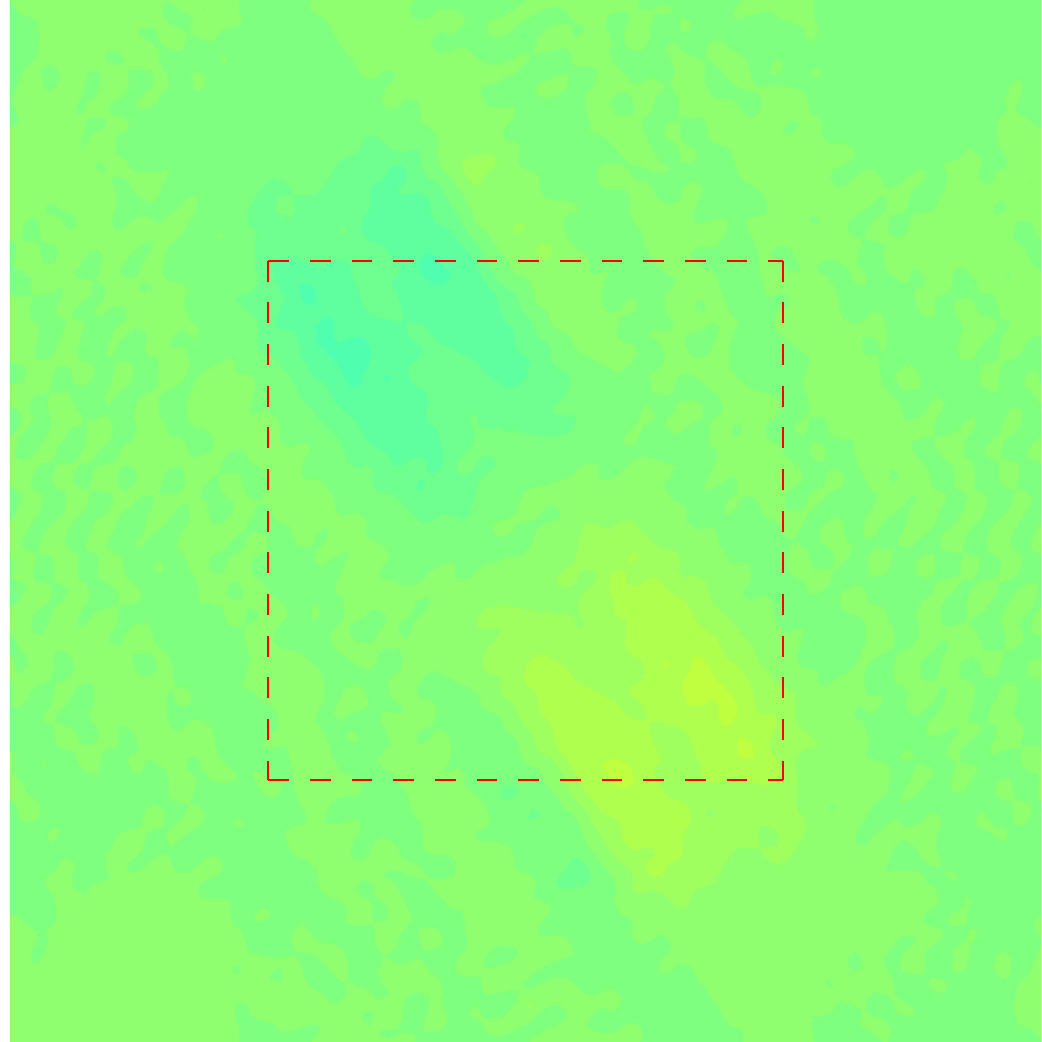}}
\end{minipage}
\hspace{-.60em}
\begin{minipage}{.0528\textwidth}
\centering
{\tiny input 8} \vspace{.12em}

\cfboxR{0.8pt}{black}{\includegraphics[width=.93\textwidth]{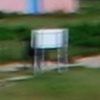}} \vspace{-1.05em}

\cfboxR{0.8pt}{black}{\includegraphics[width=.93\textwidth]{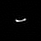}} \vspace{-1.05em}

\cfboxR{0.8pt}{black}{\includegraphics[width=.93\textwidth]{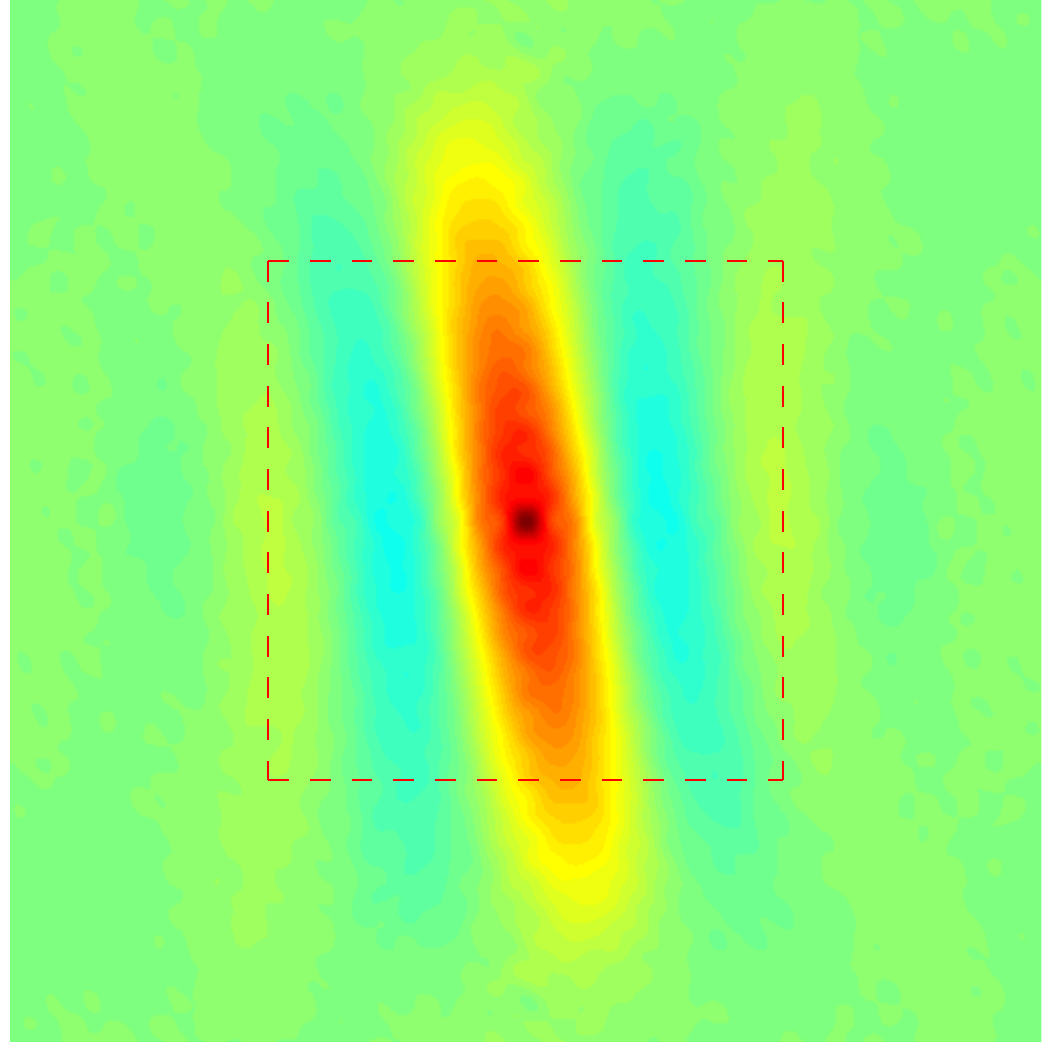}} \vspace{-1.05em}

\cfboxR{0.8pt}{black}{\includegraphics[width=.93\textwidth]{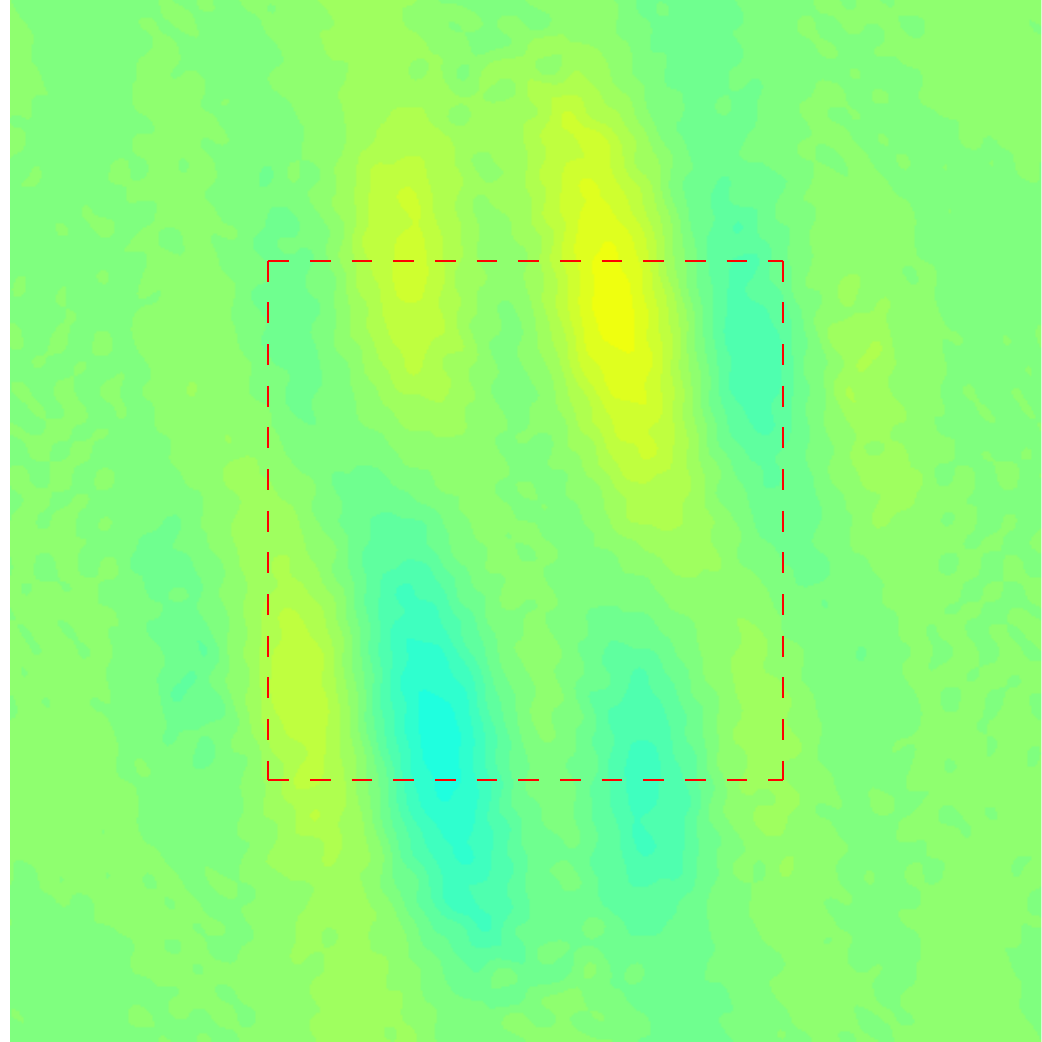}}
\end{minipage}
\hspace{-.60em}
\begin{minipage}{.0528\textwidth}
\centering
{\tiny input 9} \vspace{.12em}

\cfboxR{0.8pt}{black}{\includegraphics[width=.93\textwidth]{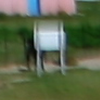}} \vspace{-1.05em}

\cfboxR{0.8pt}{black}{\includegraphics[width=.93\textwidth]{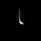}} \vspace{-1.05em}

\cfboxR{0.8pt}{black}{\includegraphics[width=.93\textwidth]{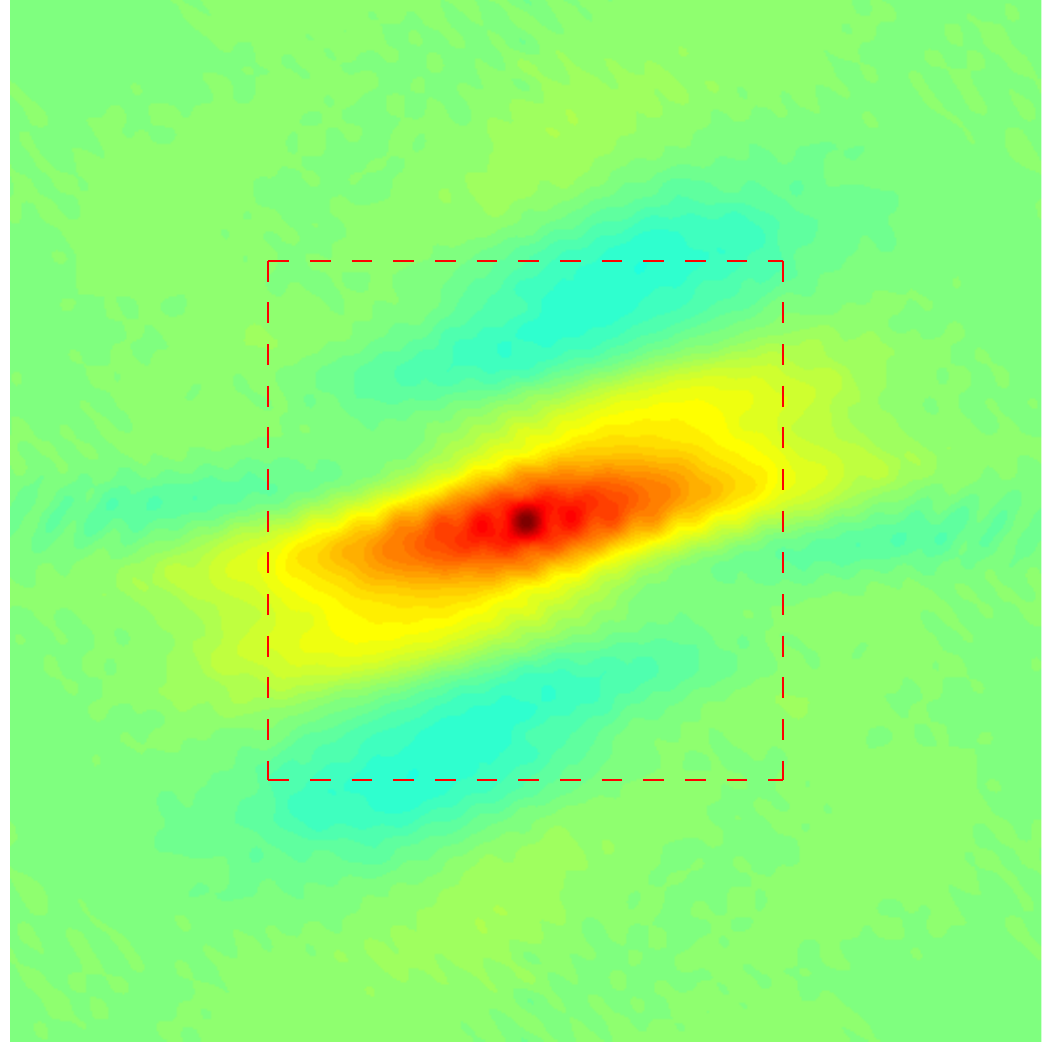}} \vspace{-1.05em}

\cfboxR{0.8pt}{black}{\includegraphics[width=.93\textwidth]{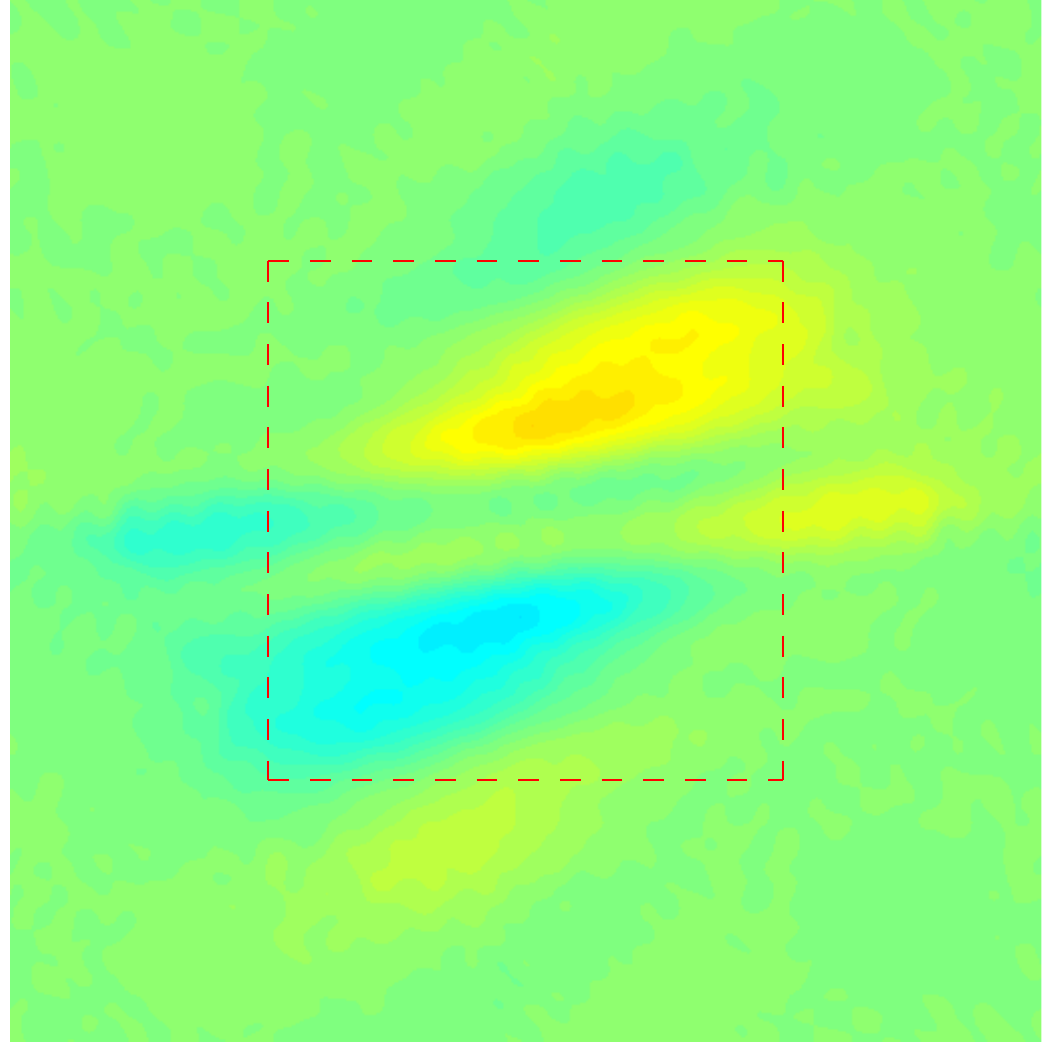}}
\end{minipage}
\hspace{-.60em}
\begin{minipage}{.0528\textwidth}
\centering
{\tiny input 10} \vspace{.12em}

\cfboxR{0.8pt}{black}{\includegraphics[width=.93\textwidth]{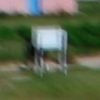}} \vspace{-1.05em}

\cfboxR{0.8pt}{black}{\includegraphics[width=.93\textwidth]{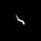}} \vspace{-1.05em}

\cfboxR{0.8pt}{black}{\includegraphics[width=.93\textwidth]{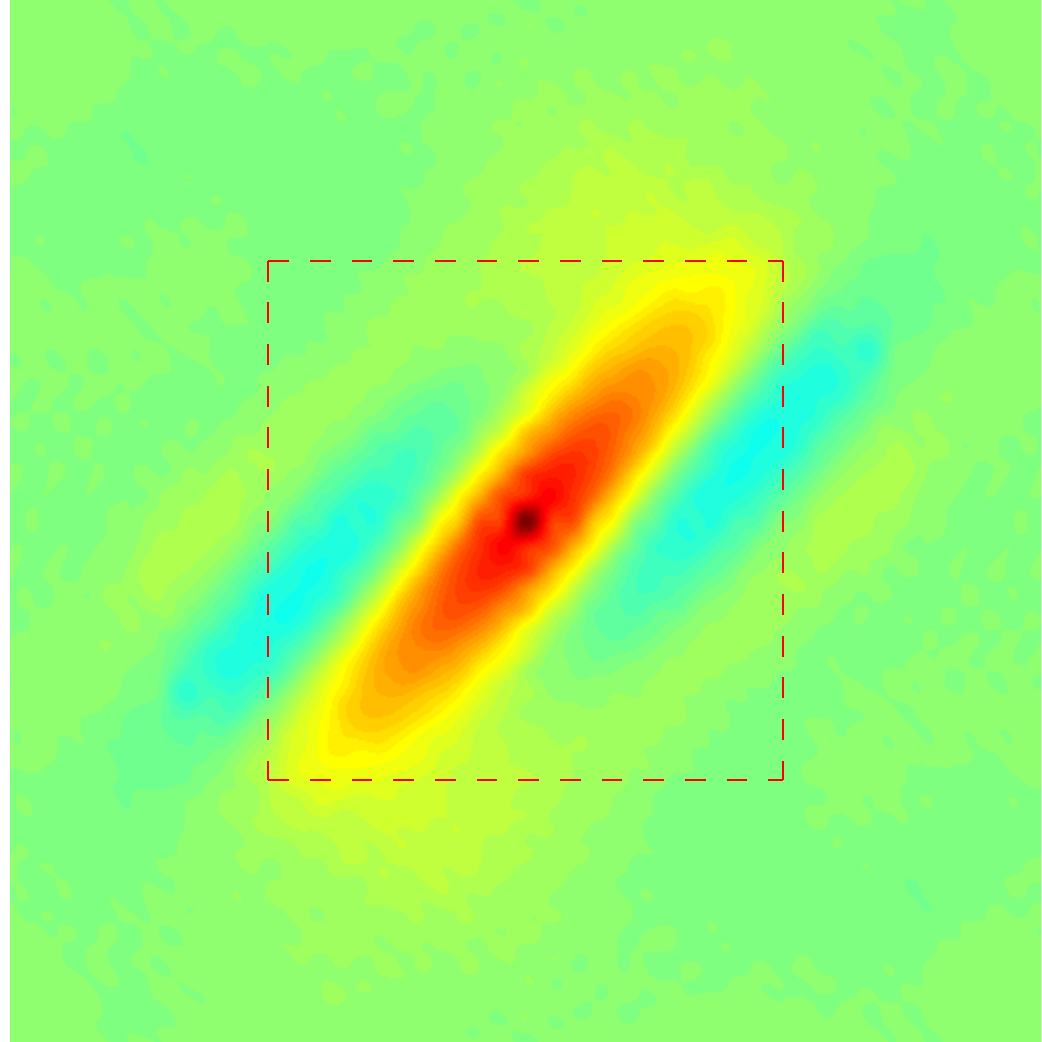}} \vspace{-1.05em}

\cfboxR{0.8pt}{black}{\includegraphics[width=.93\textwidth]{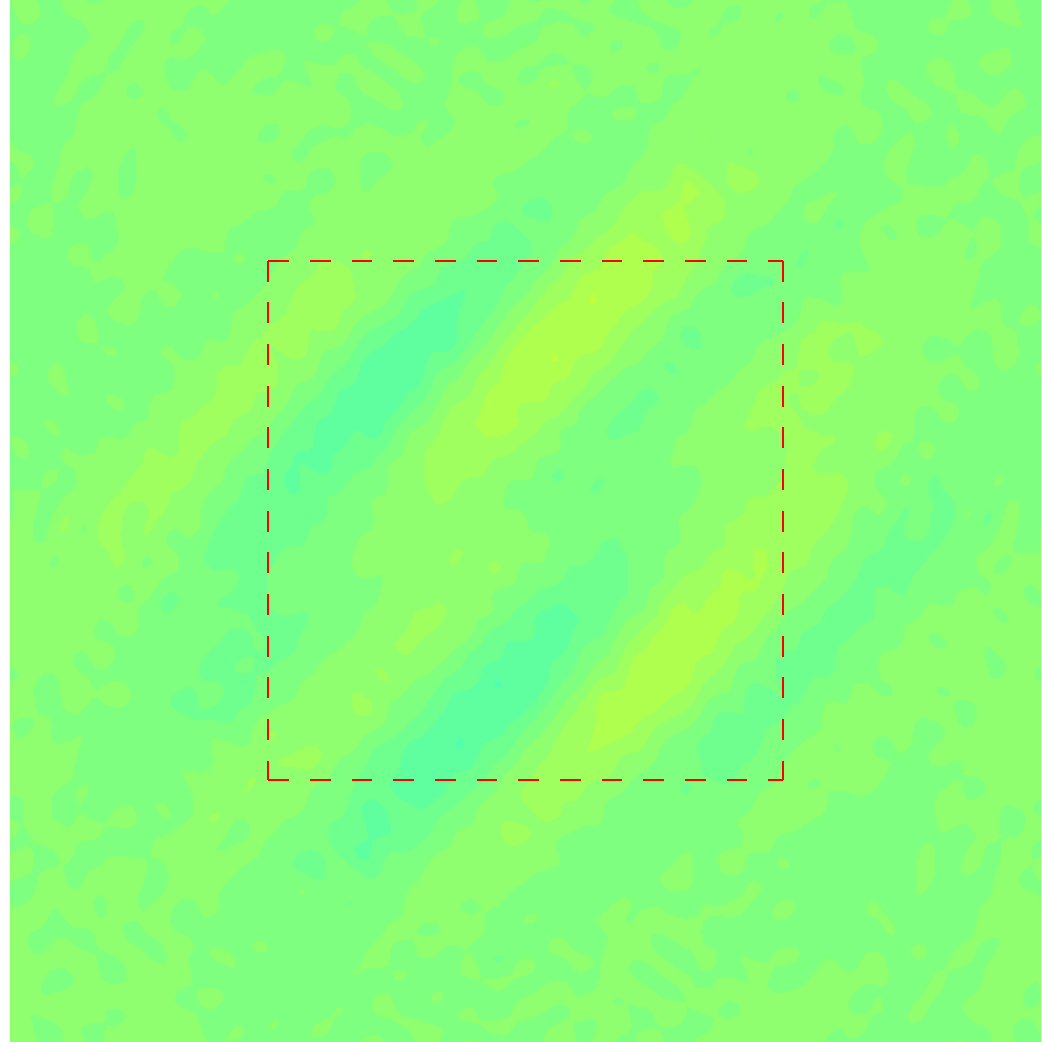}}
\end{minipage}
\hspace{-.60em}
\begin{minipage}{.0528\textwidth}
\centering
{\tiny input 11} \vspace{.12em}

\cfboxR{0.8pt}{black}{\includegraphics[width=.93\textwidth]{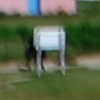}} \vspace{-1.05em}

\cfboxR{0.8pt}{black}{\includegraphics[width=.93\textwidth]{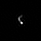}} \vspace{-1.05em}

\cfboxR{0.8pt}{black}{\includegraphics[width=.93\textwidth]{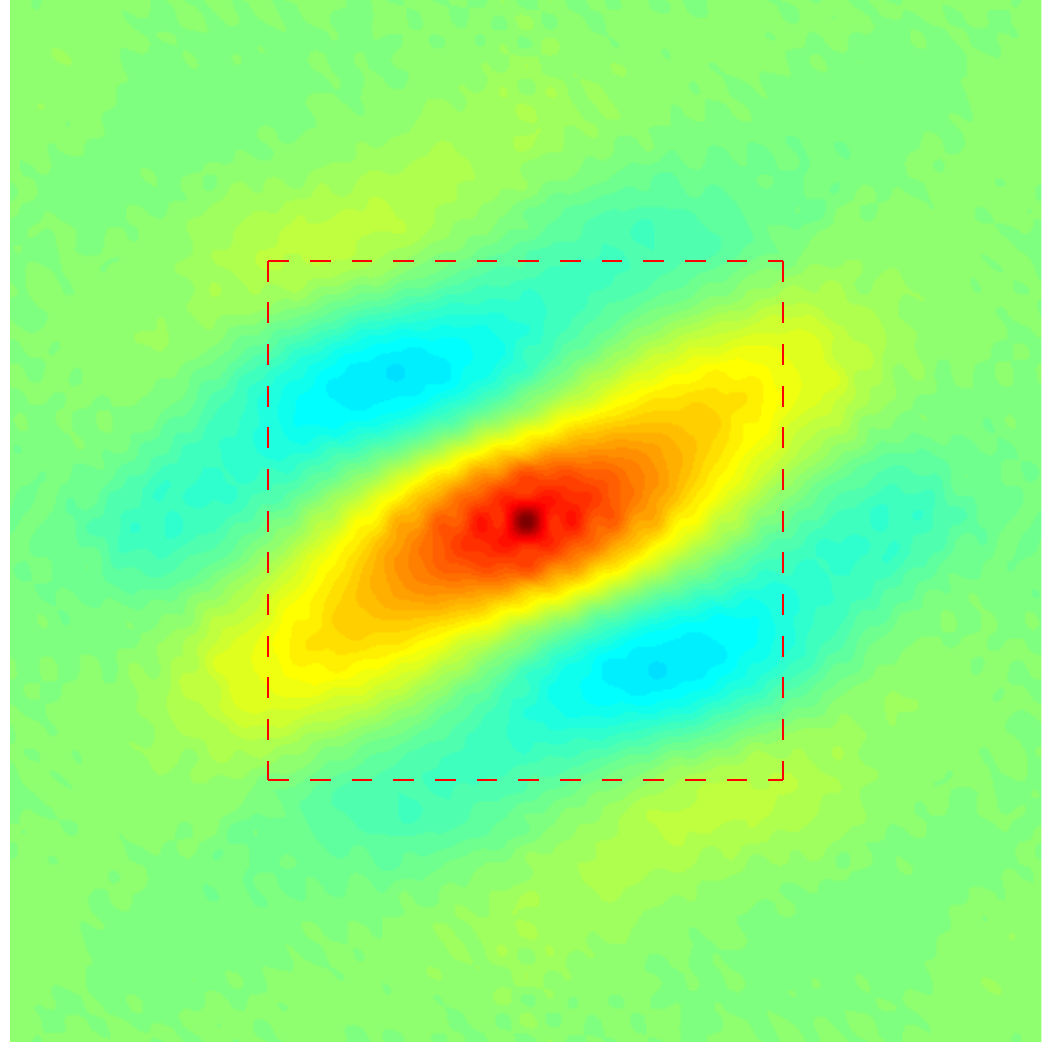}} \vspace{-1.05em}

\cfboxR{0.8pt}{black}{\includegraphics[width=.93\textwidth]{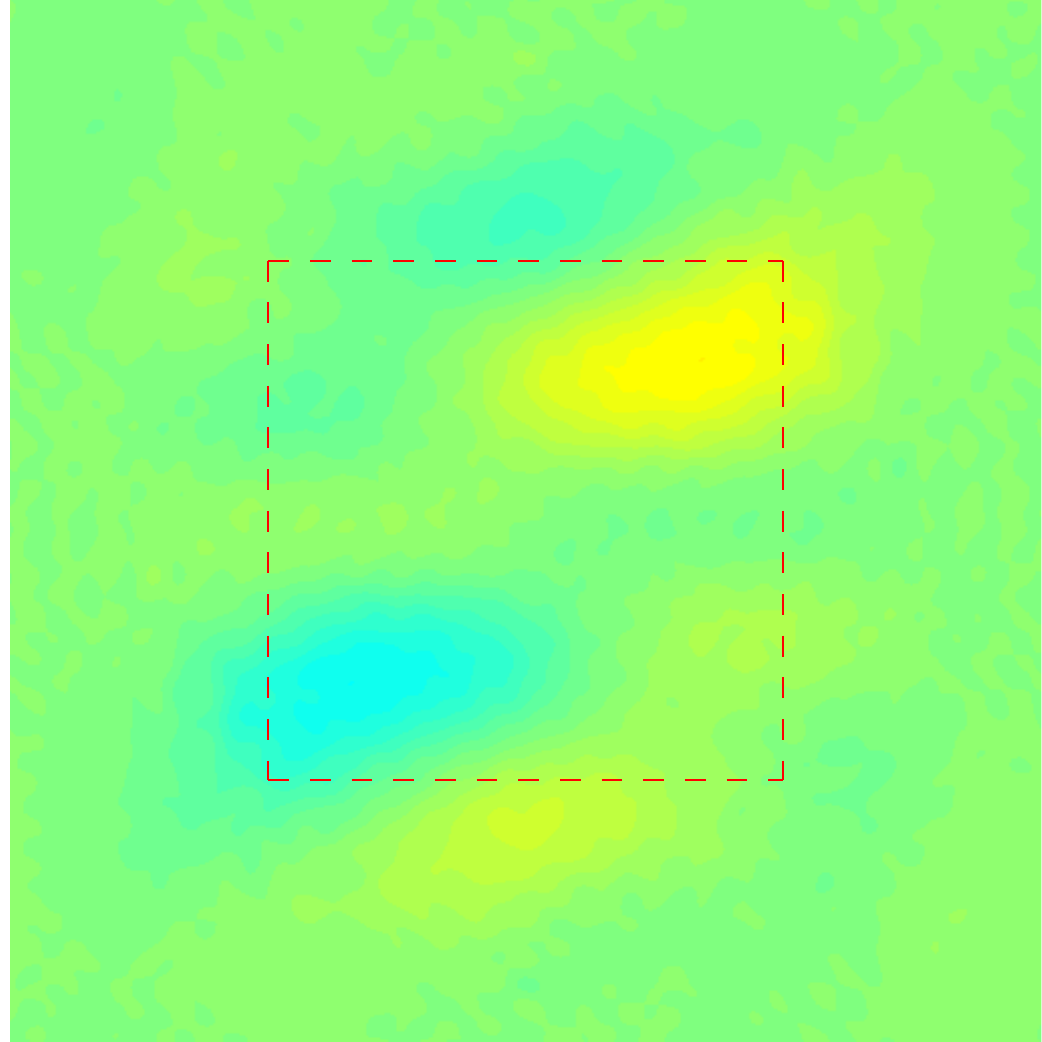}}
\end{minipage}
\hspace{-.60em}
\begin{minipage}{.0528\textwidth}
\centering
{\tiny input 12} \vspace{.12em}

\cfboxR{0.8pt}{black}{\includegraphics[width=.93\textwidth]{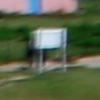}} \vspace{-1.05em}

\cfboxR{0.8pt}{black}{\includegraphics[width=.93\textwidth]{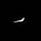}} \vspace{-1.05em}

\cfboxR{0.8pt}{black}{\includegraphics[width=.93\textwidth]{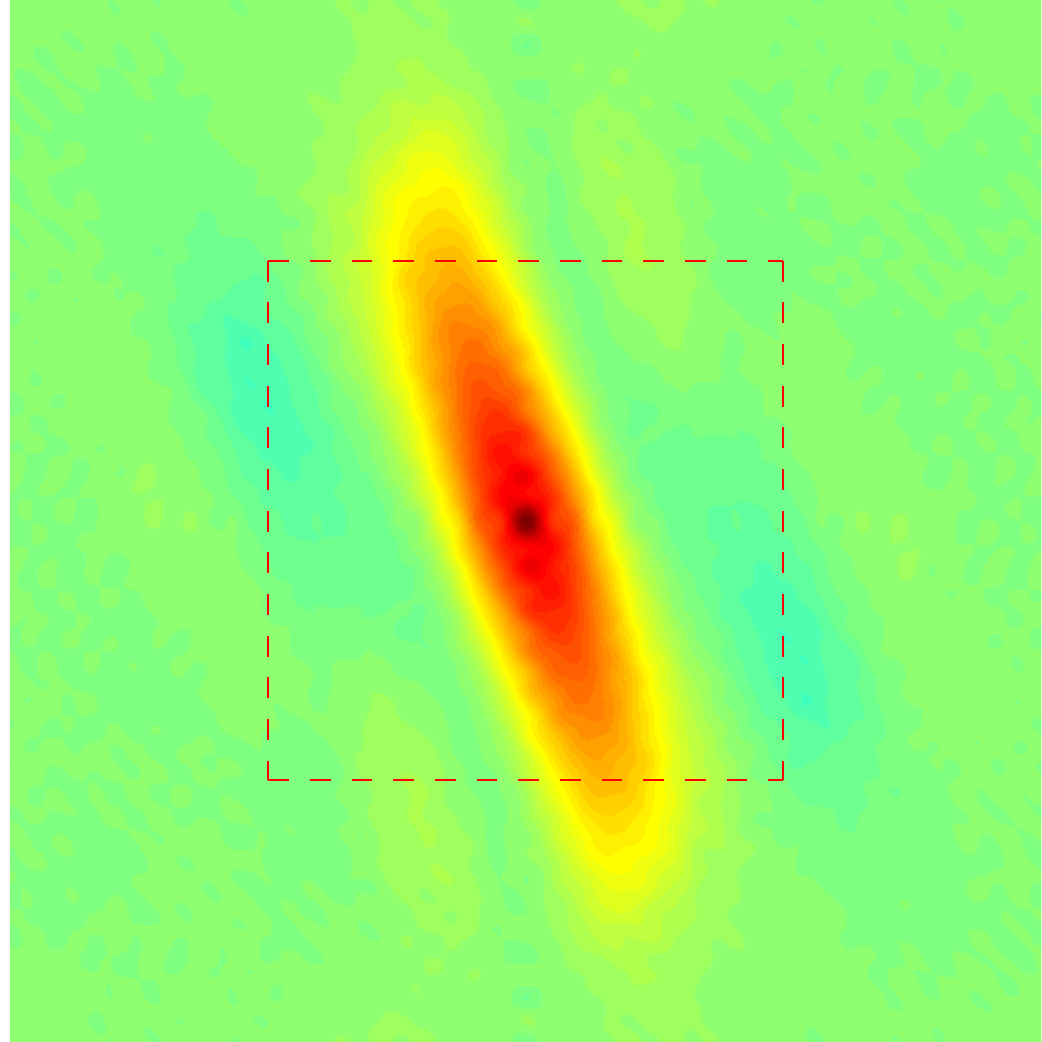}} \vspace{-1.05em}

\cfboxR{0.8pt}{black}{\includegraphics[width=.93\textwidth]{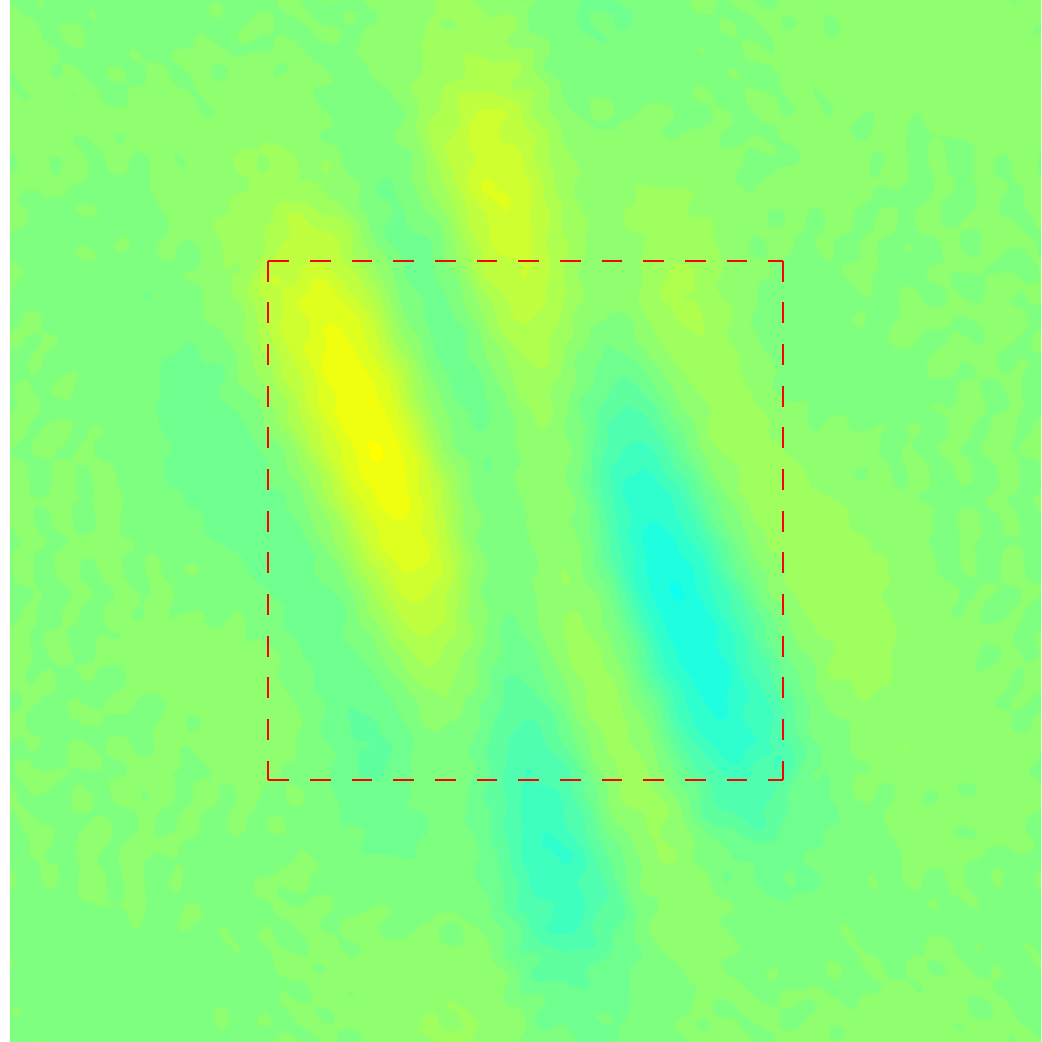}}
\end{minipage}
\hspace{-.60em}
\begin{minipage}{.0528\textwidth}
\centering
{\tiny input 13} \vspace{.12em}

\cfboxR{0.8pt}{black}{\includegraphics[width=.93\textwidth]{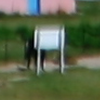}} \vspace{-1.05em}

\cfboxR{0.8pt}{black}{\includegraphics[width=.93\textwidth]{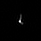}} \vspace{-1.05em}

\cfboxR{0.8pt}{black}{\includegraphics[width=.93\textwidth]{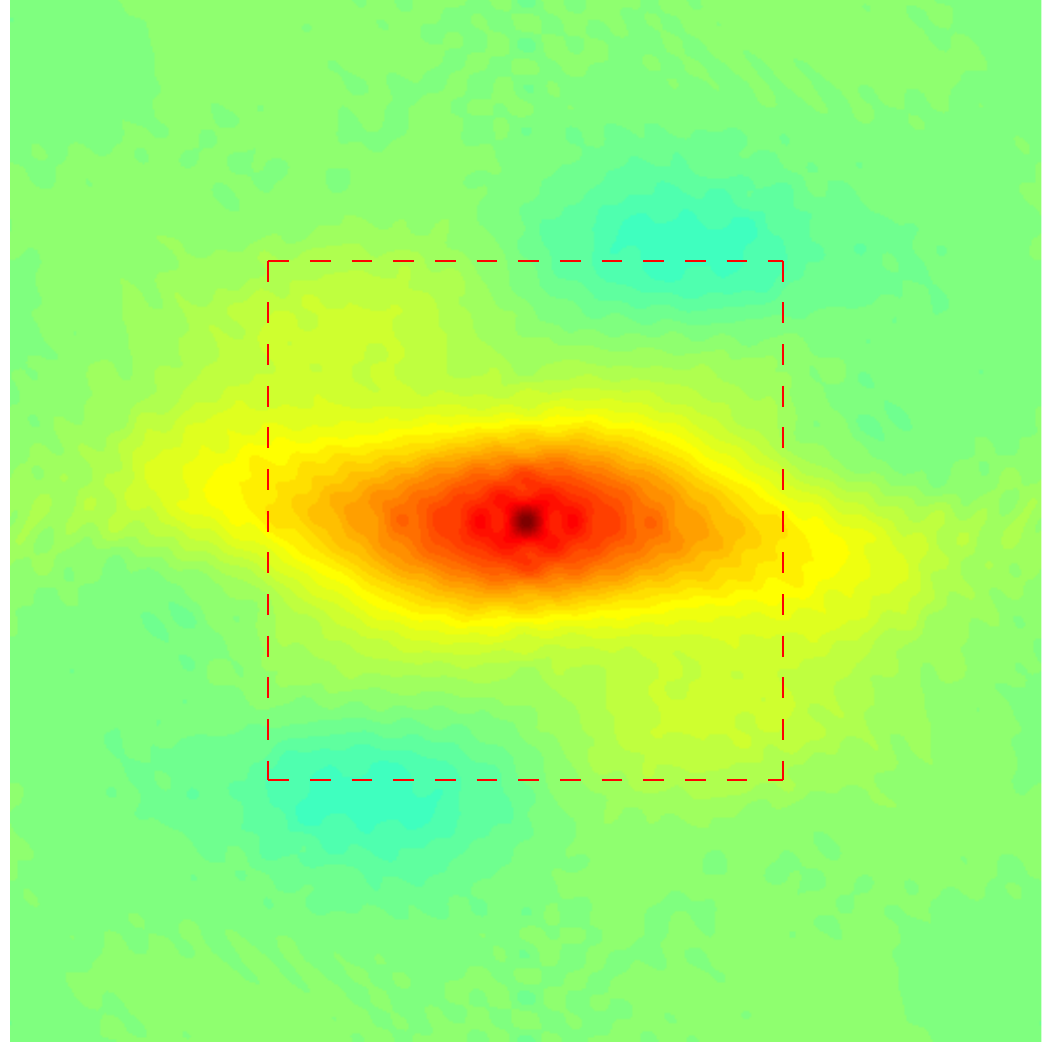}} \vspace{-1.05em}

\cfboxR{0.8pt}{black}{\includegraphics[width=.93\textwidth]{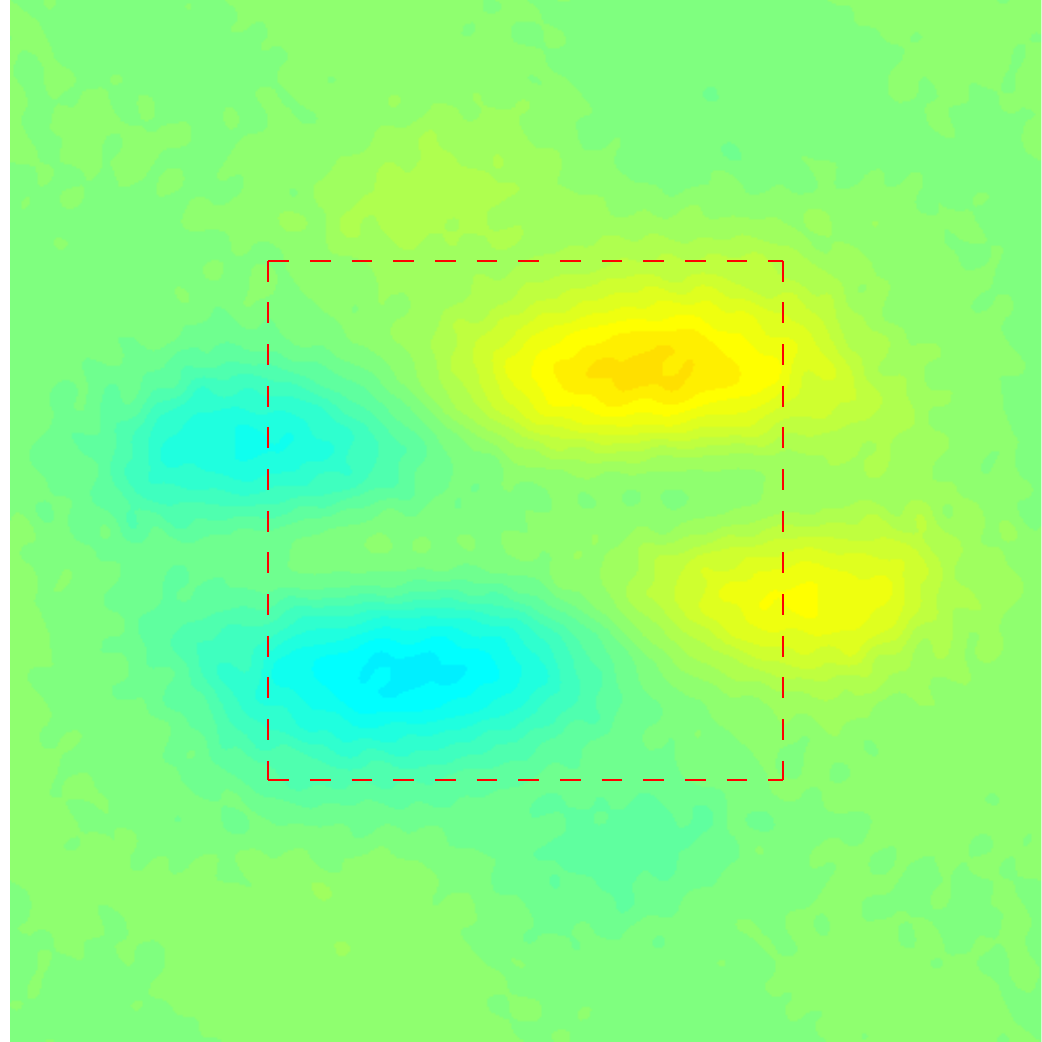}}
\end{minipage}
\hspace{-.60em}
\begin{minipage}{.0528\textwidth}
\centering
{\tiny input 14} \vspace{.12em}

\cfboxR{0.8pt}{black}{\includegraphics[width=.93\textwidth]{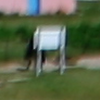}} \vspace{-1.05em}

\cfboxR{0.8pt}{black}{\includegraphics[width=.93\textwidth]{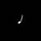}} \vspace{-1.05em}

\cfboxR{0.8pt}{black}{\includegraphics[width=.93\textwidth]{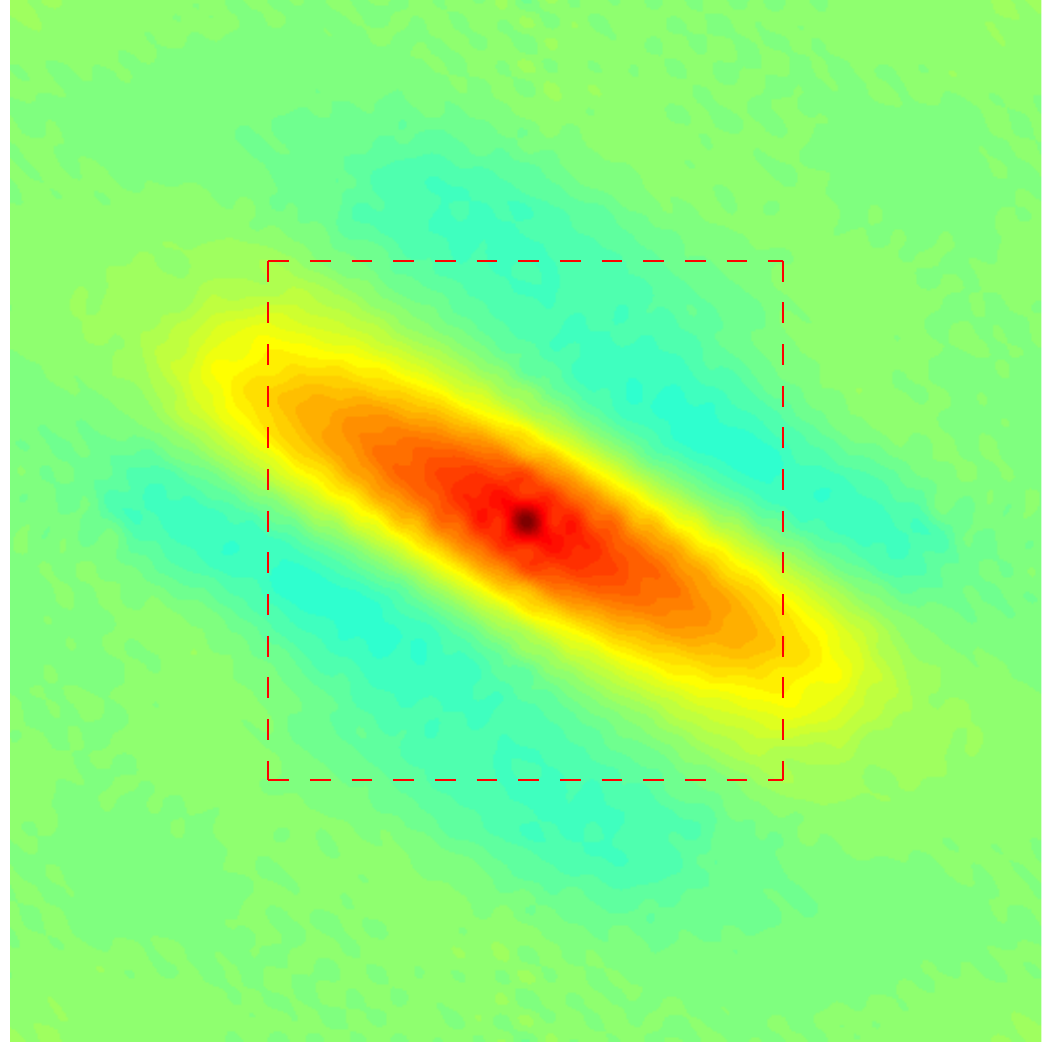}} \vspace{-1.05em}

\cfboxR{0.8pt}{black}{\includegraphics[width=.93\textwidth]{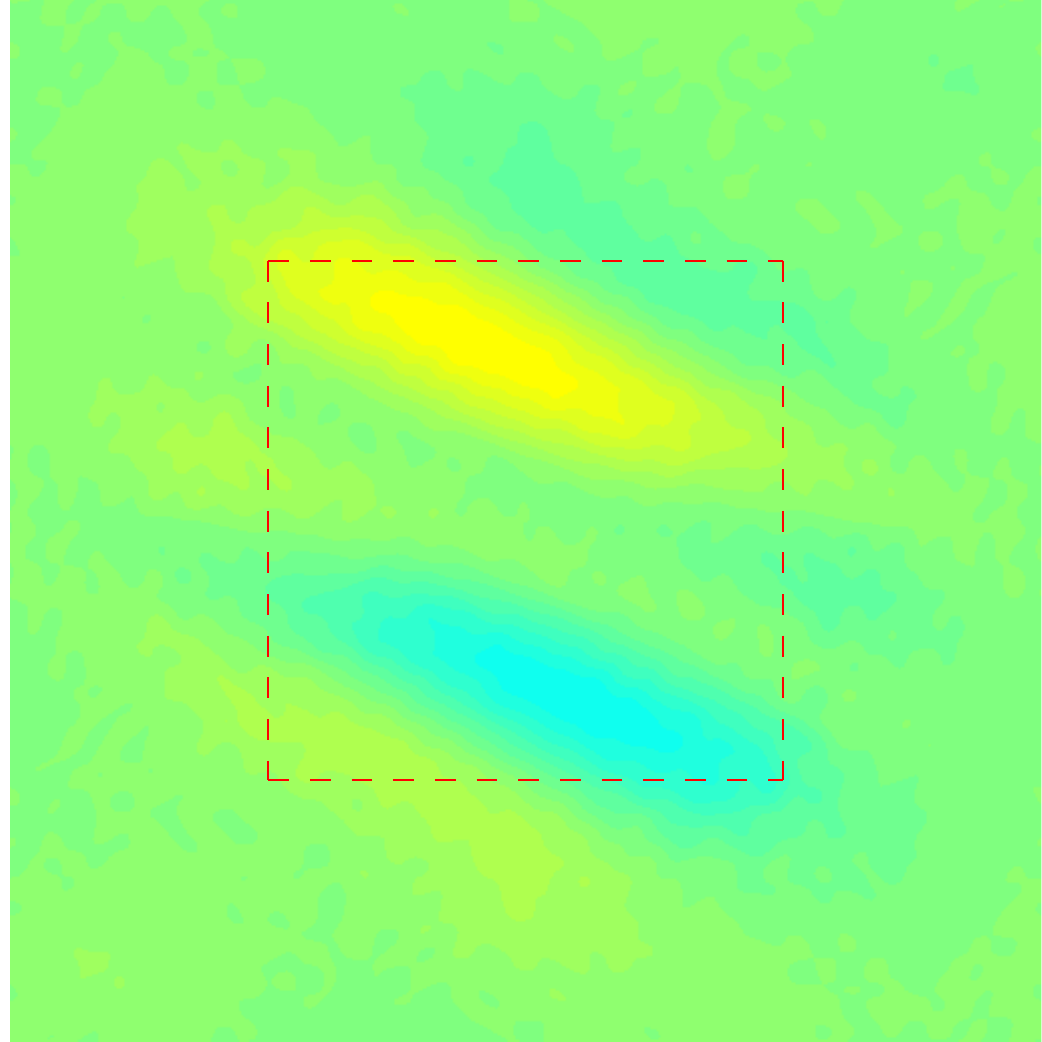}}
\end{minipage}
\hspace{-.60em}
\begin{minipage}{.0528\textwidth}
\centering
{\tiny $p=0$} \vspace{.12em}

\cfboxR{0.8pt}{blue1}{\includegraphics[width=.93\textwidth]{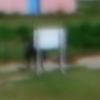}} \vspace{-1.05em}

\cfboxR{0.8pt}{blue1}{\includegraphics[width=.93\textwidth]{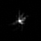}} \vspace{-1.05em}

\cfboxR{0.8pt}{blue1}{\includegraphics[width=.93\textwidth]{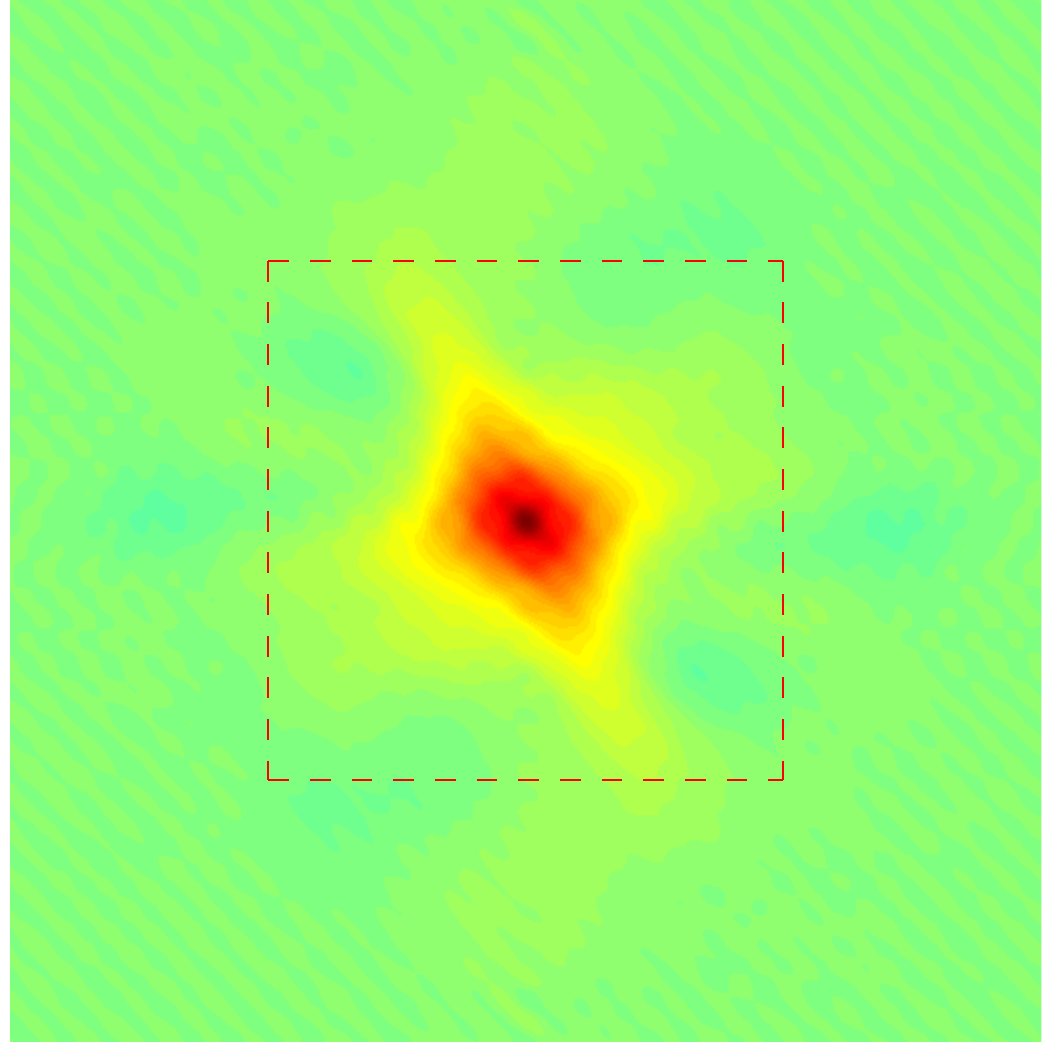}} \vspace{-1.05em}

\cfboxR{0.8pt}{blue1}{\includegraphics[width=.93\textwidth]{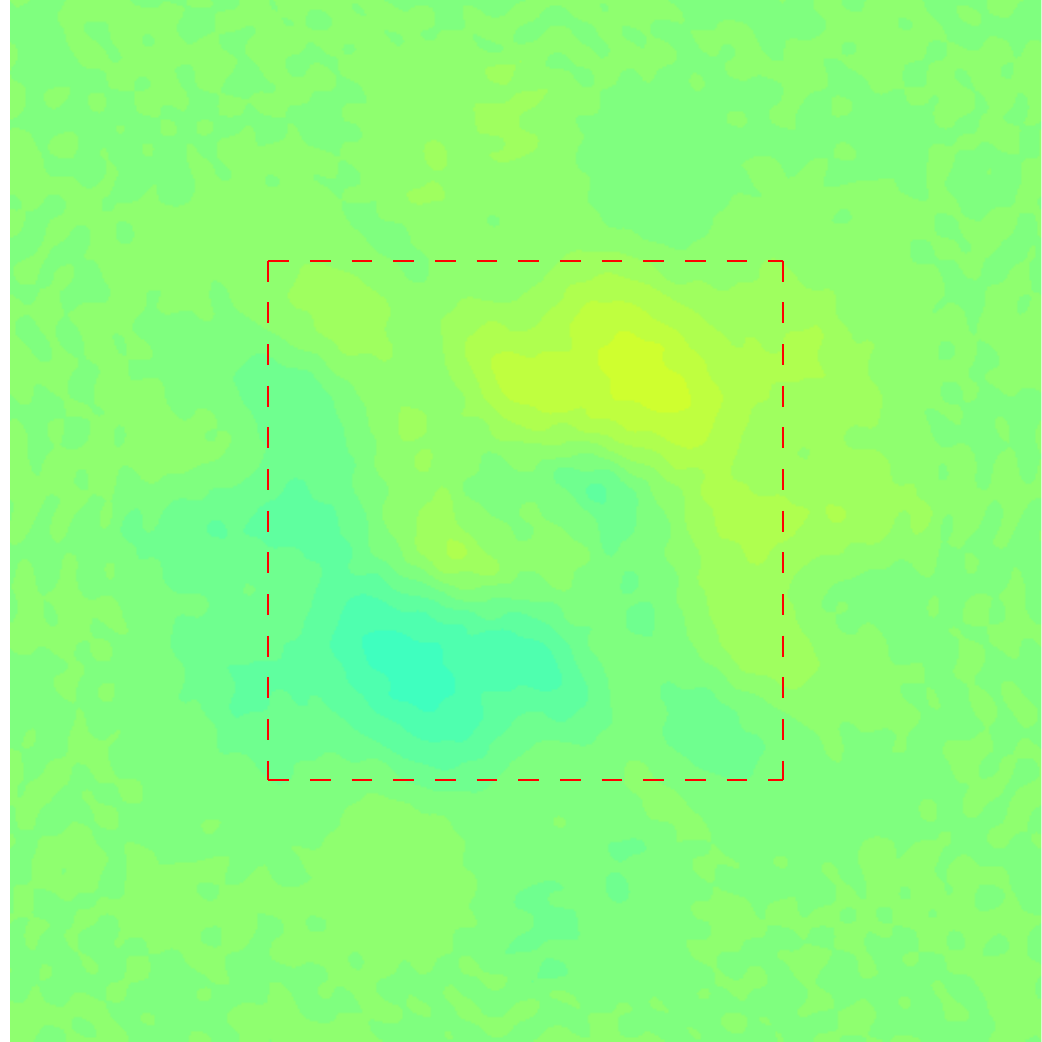}}

\end{minipage}
\hspace{-.60em}
\begin{minipage}{.0528\textwidth}
\centering
{\tiny $p=3$} \vspace{.12em}

\cfboxR{0.8pt}{blue1}{\includegraphics[width=.93\textwidth]{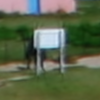}} \vspace{-1.05em}

\cfboxR{0.8pt}{blue1}{\includegraphics[width=.93\textwidth]{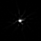}} \vspace{-1.05em}

\cfboxR{0.8pt}{blue1}{\includegraphics[width=.93\textwidth]{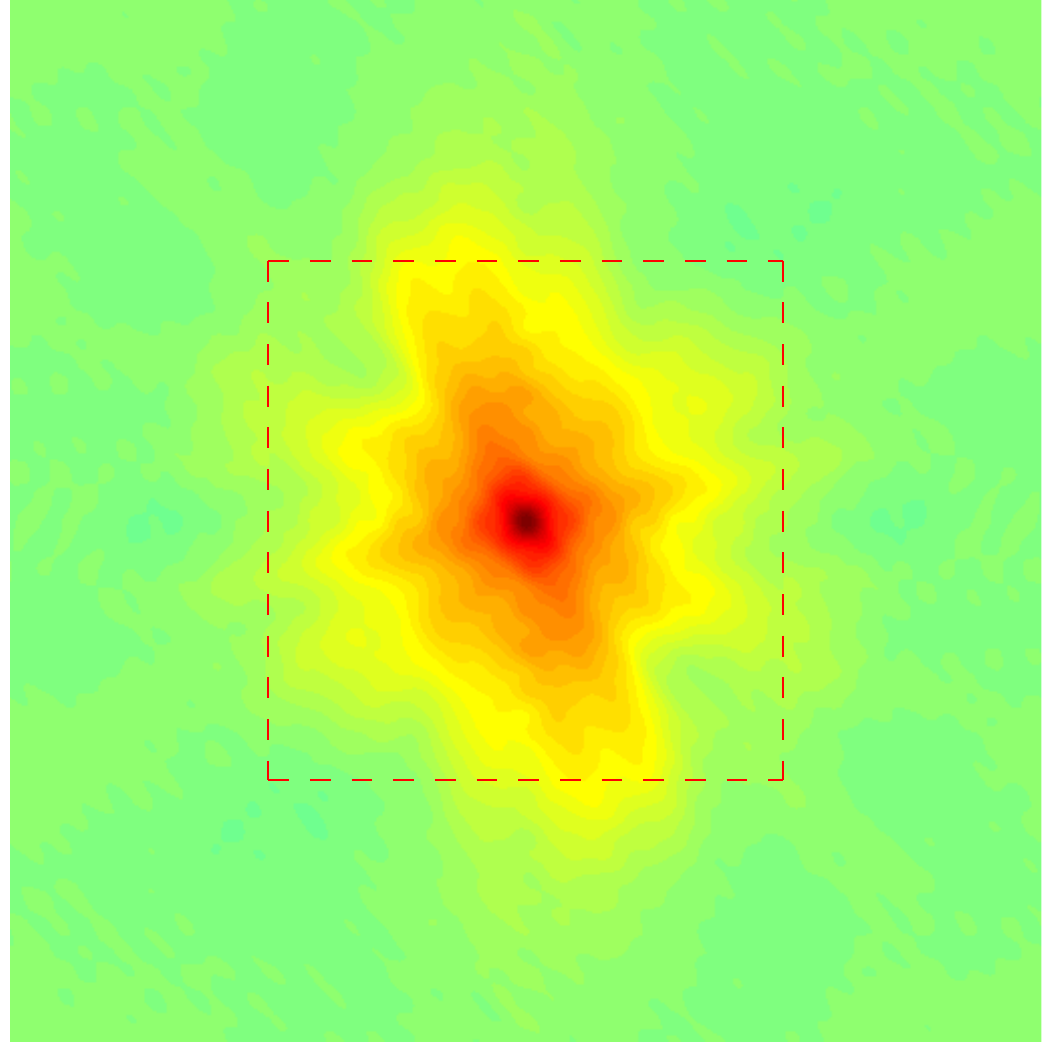}} \vspace{-1.05em}

\cfboxR{0.8pt}{blue1}{\includegraphics[width=.93\textwidth]{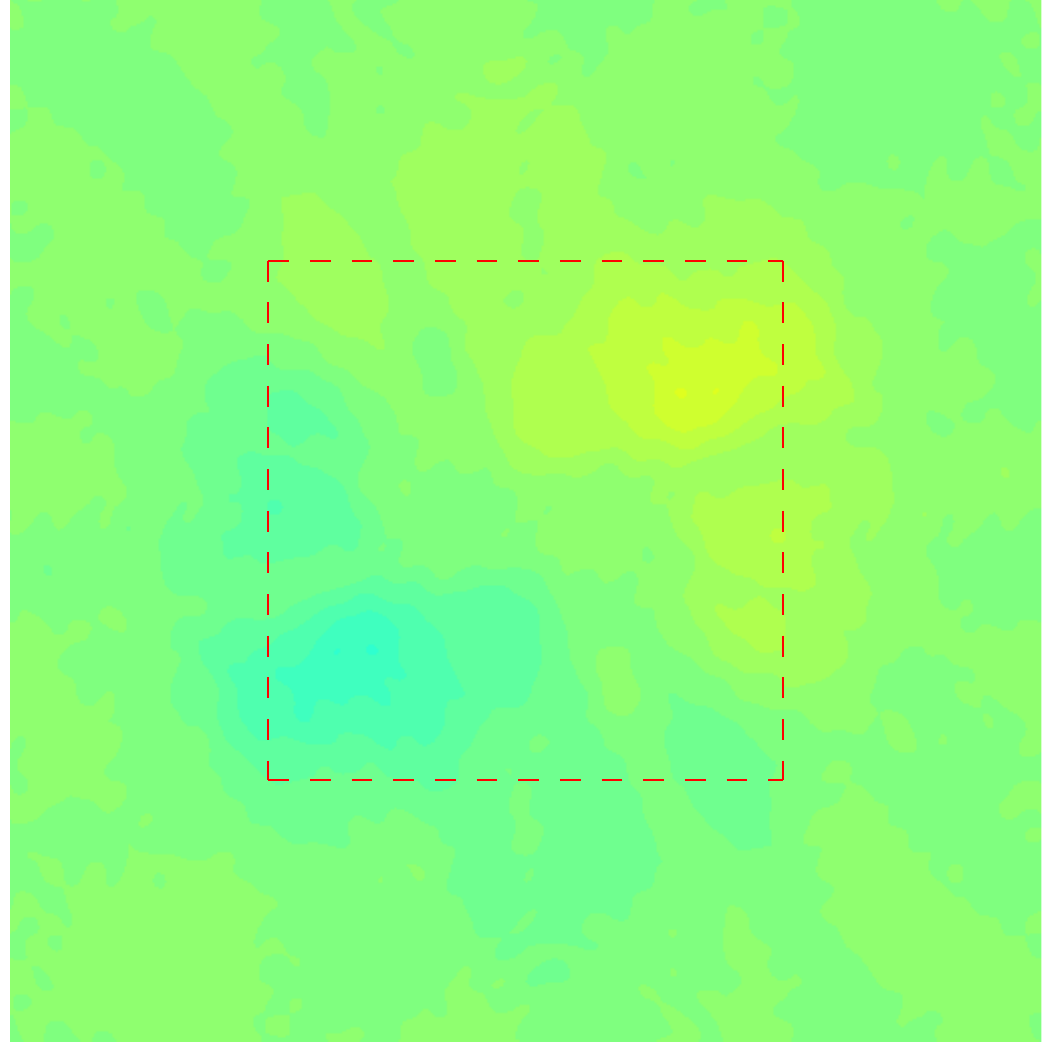}}
\end{minipage}
\hspace{-.60em}
\begin{minipage}{.0528\textwidth}
\centering
{\tiny $p=11$} \vspace{.12em}

\cfboxR{0.8pt}{blue1}{\includegraphics[width=.93\textwidth]{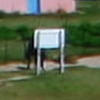}} \vspace{-1.05em}

\cfboxR{0.8pt}{blue1}{\includegraphics[width=.93\textwidth]{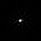}} \vspace{-1.05em}

\cfboxR{0.8pt}{blue1}{\includegraphics[width=.93\textwidth]{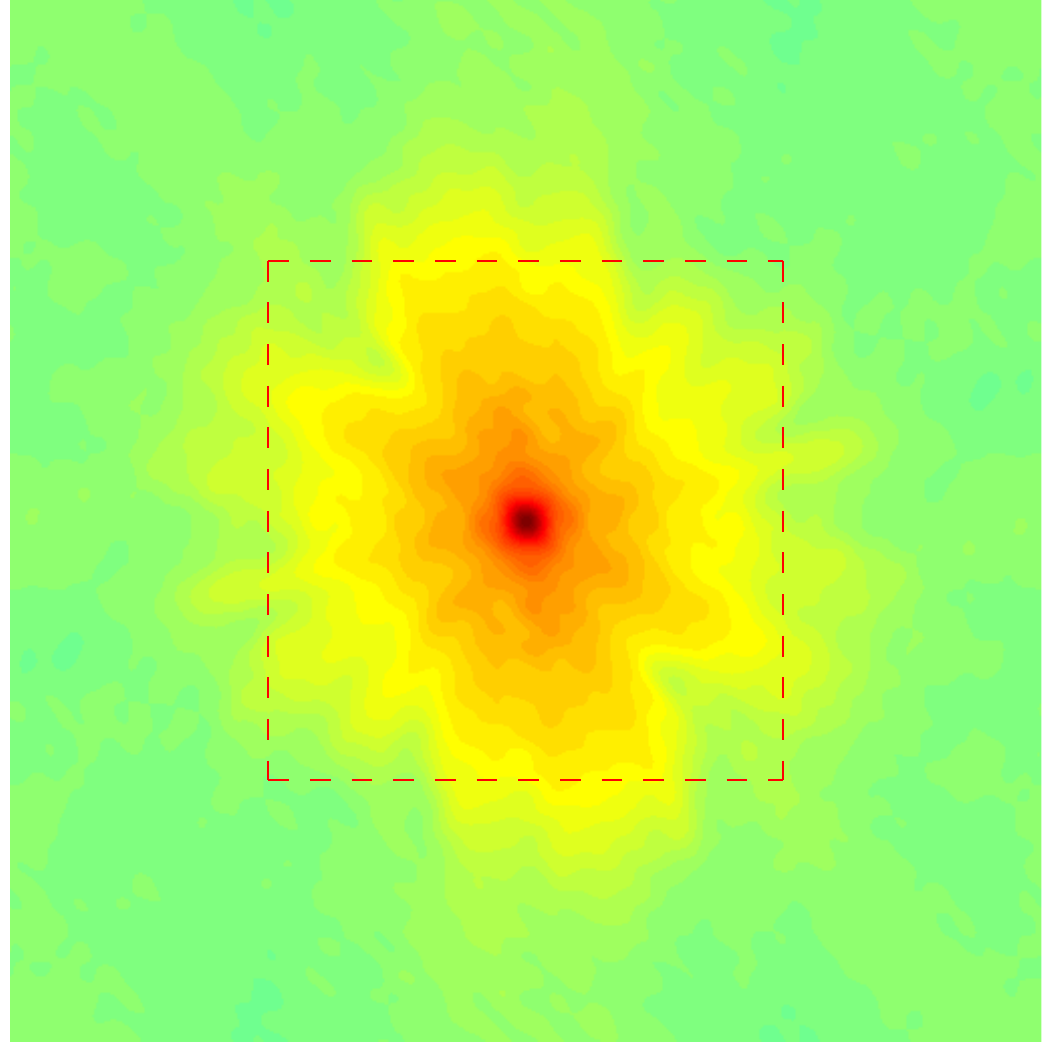}} \vspace{-1.05em}

\cfboxR{0.8pt}{blue1}{\includegraphics[width=.93\textwidth]{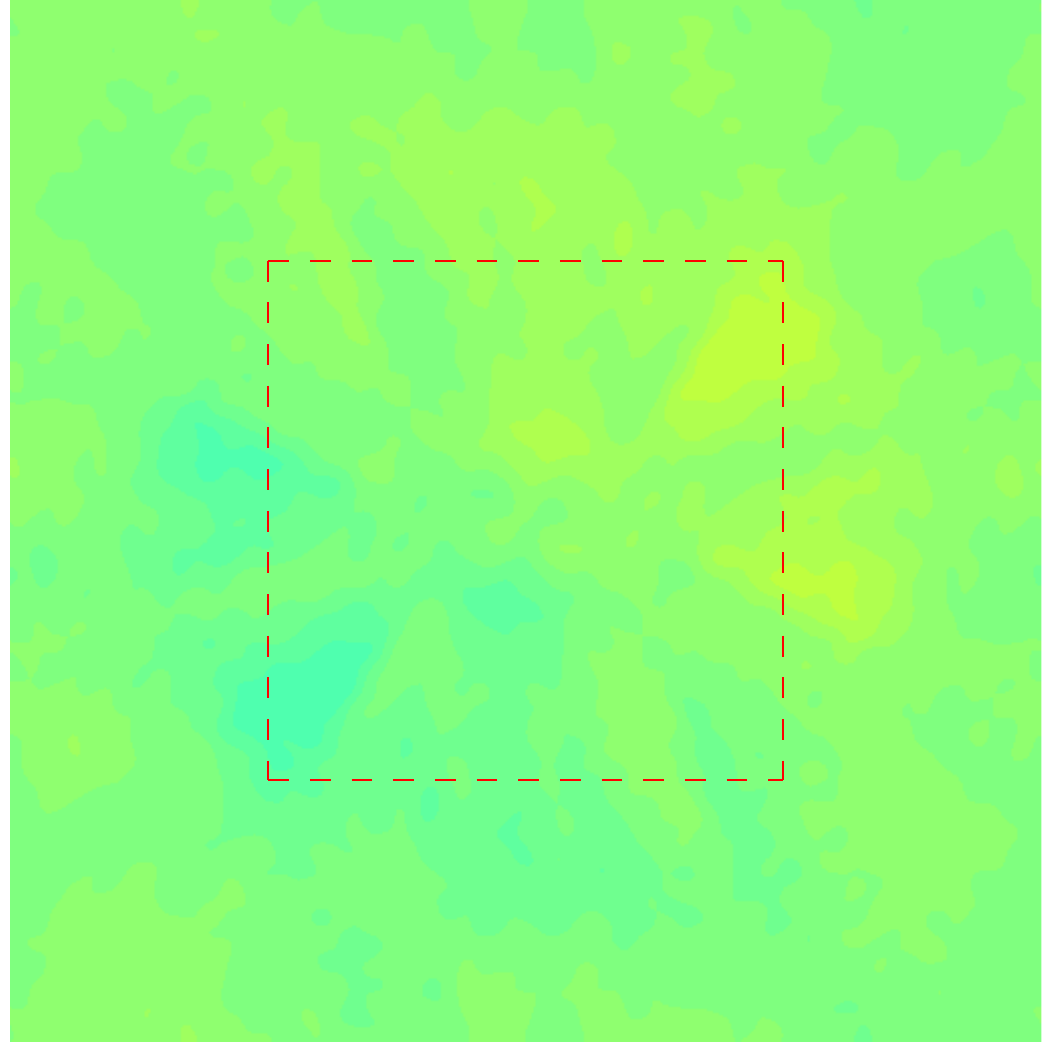}}
\end{minipage}
\hspace{-.60em}
\begin{minipage}{.0528\textwidth}
\centering
{\tiny $p=25$} \vspace{.12em}

\cfboxR{0.8pt}{blue1}{\includegraphics[width=.93\textwidth]{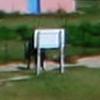}} \vspace{-1.05em}

\cfboxR{0.8pt}{blue1}{\includegraphics[width=.93\textwidth]{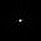}} \vspace{-1.05em}

\cfboxR{0.8pt}{blue1}{\includegraphics[width=.93\textwidth]{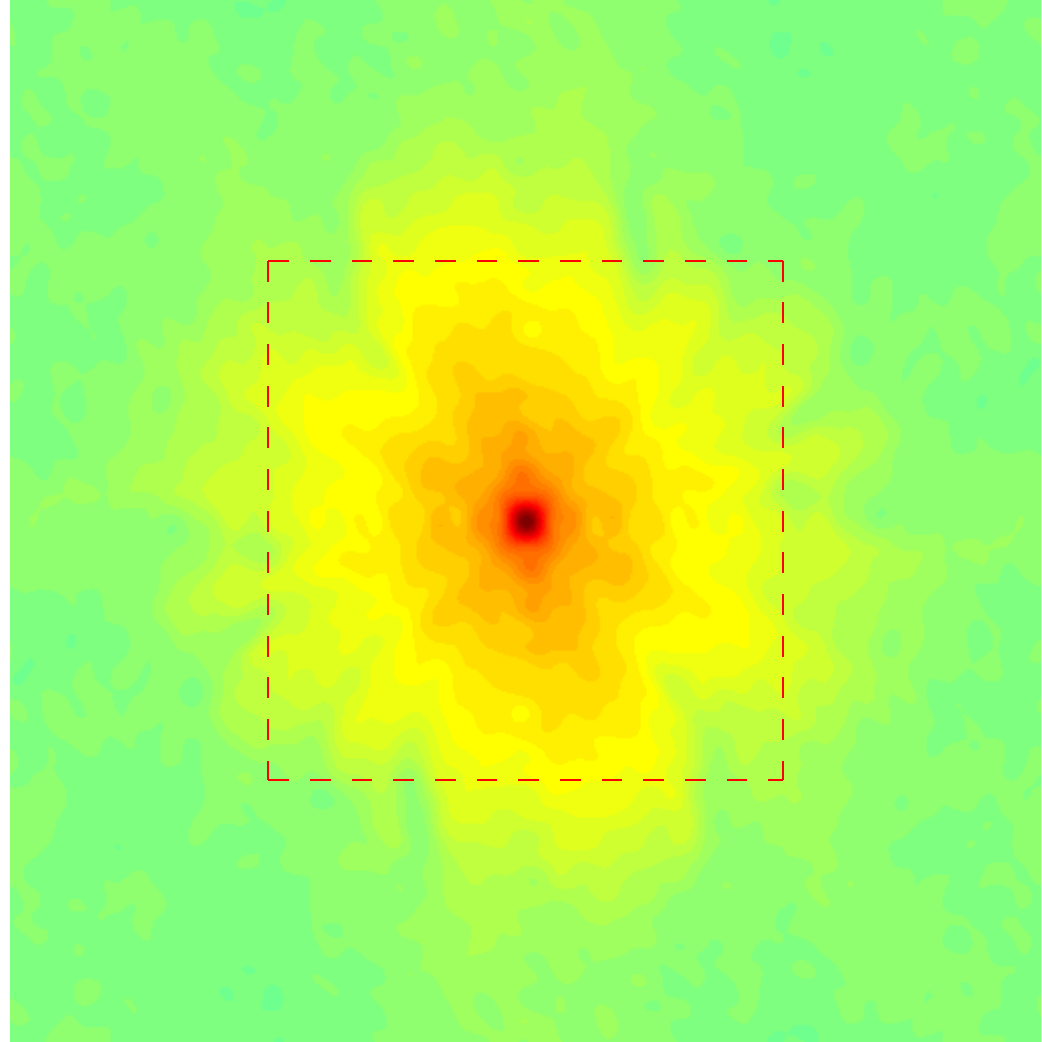}} \vspace{-1.05em}

\cfboxR{0.8pt}{blue1}{\includegraphics[width=.93\textwidth]{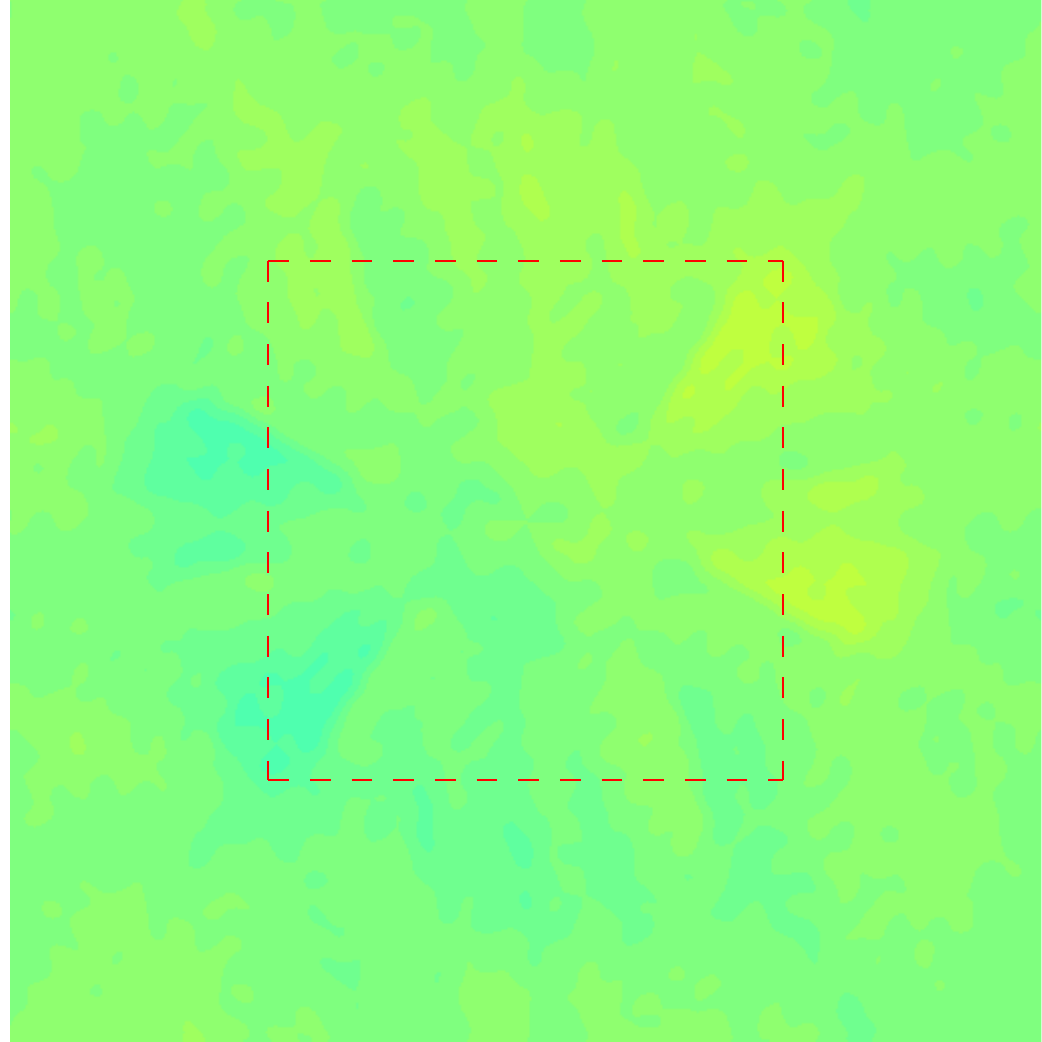}}
\end{minipage}
\hspace{-.5em}
\begin{minipage}{.011\textwidth}
\vspace{6.1em}

\includegraphics[width=1.45\textwidth]{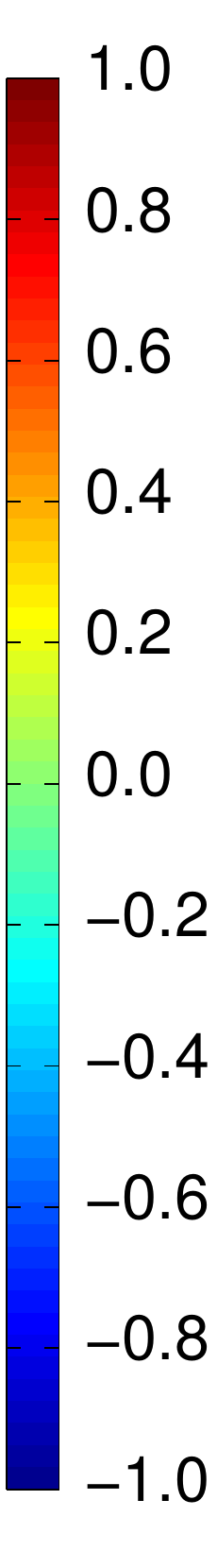}
\end{minipage}

\caption{Real camera shake kernels were computed  using a
sharp snapshot captured with a tripod as a reference. The first row shows a crop
of each input image (frames 1 to 14) and the proposed Fourier Burst Accumulation (from~\eqref{eq:fourierWeightsOrig}, no additional sharpening).
As noted before, the kernels are generally regular unidimensional trajectories (second row). The four last columns
in the second row show the resultant point spread functions (PSF) after the Fourier weighted average for different values of $p$. 
The kernel due to the Fourier average with $p>0$ is closer to a Delta function, showing the success of the method.
The two bottom rows show the Fourier real and imaginary parts of each blurring kernel (the red box indicates the $\pi/2$ frequency).
The real part is mostly positive and significantly larger than the imaginary part, implying that the blurring
kernels do not introduce significant phase distortions. This might not be the case for large motion kernels, uncommon in standard hand shakes.}
\label{fig:equivKernel}
\end{figure*}

Camera shake originated from hand tremor vibrations has undoubtedly a random 
nature~\cite{carignan2010quantifying,gavant2011physiological,xiao2006camera}. 
The independent movement of the photographer hand causes the camera to 
be pushed randomly and unpredictably, generating blurriness in the captured image.
Figure \ref{fig:realKernelsLaptop} shows several photographs taken with a digital single-lens reflex (\textsc{dslr}) handheld camera.
The photographed scene consists of a laptop displaying a black image with white dots. 
The captured picture of the white dots illustrates the trace of the camera movement in the image plane. 
If the dots are very small ---mimicking Dirac masses--- their photographs represent the blurring kernels themselves. 
As one can see, the kernels mostly consist of unidimensional regular random trajectories.  
This stochastic behavior will be the key ingredient in our proposed approach.

Let  $\mathcal F$ denote the Fourier Transform and $\hat{k}$ the Fourier Transform of the kernel $k$.  
Images are defined in a regular grid indexed by the $2D$ position $\bx$
and the Fourier domain is indexed by the $2D$ frequency $\zeta$.
Let us assume, without loss of generality, that the kernel $k$  due to camera shake is normalized such that
$\int k(\bx) d\bx = 1$. The blurring kernel is  nonnegative since the integration of incoherent light is always nonnegative. 
This implies that the motion blur does not amplify the Fourier spectrum:
\begin{claim}   
Let $k(\bx)\ge 0$ and $\int k(\bx) = 1$. Then, $|\hat{k}(\zeta)| \le 1, \forall \zeta$. 
(Blurring kernels do not amplify the spectrum.)
\end{claim}
\begin{proof}
$$
\left| \hat{k}(\zeta) \right| = \left| \int k(\bx) e^{i \bx \cdot \zeta} d\bx \right| \le  \int  \left| k(\bx) \right| d\bx =   \int k(\bx)  d\bx  = 1.
$$
\end{proof}
Most modern digital cameras have a burst mode where
the photographer is allowed to sequently take a series of images, one right after the other.
Let us assume that the photographer
takes a sequence of $M$ images of the same scene $u$,
\begin{align}
v_i = u \star k_i + n_i, \quad \text{for} \quad i=1,\ldots,M.
\label{eq:burst}
\end{align}
The movement of the camera during any two images of the burst will be essentially independent.
Thus, the blurring kernels $k_i$ will be mostly different for different images in the burst.
Hence, each Fourier frequency of $\hat{u}$ will be differently  affected on each  frame of the burst.  
The idea is to reconstruct an image whose Fourier spectrum takes for each frequency the value having the largest Fourier magnitude in the burst. 
Since a blurring kernel does not amplify the Fourier spectrum (Claim 1), the reconstructed image picks  what is less attenuated, in Fourier domain, 
from each image of the burst.
Choosing the least attenuated frequencies does not necessarily guarantee that those frequencies are the least affected by the blurring kernel, 
as the kernel may introduce changes in the Fourier image phase. 
However, for small motion kernels, the phase distortion introduced is small. This is illustrated in Figure~\ref{fig:equivKernel},
where several real motion kernels and their Fourier spectra are shown.

\subsection{Fourier Magnitude Weights}
Let $p$ be a non-negative integer,  we will call  {\it Fourier Burst Accumulation (\textsc{fba})} to the Fourier weighted averaged image,
\begin{align}
{u}_p(\bx) =  \mathcal{F}^{-1} \left( \sum_{i=1}^M w_i (\zeta) \cdot \hat{v}_i (\zeta) \right) (\bx),
\label{eq:fourierWeightsOrig}
\end{align}
$$
w_i(\zeta)   = \frac{ |\hat{v}_i (\zeta) |^p }{\sum_{j=1}^M | \hat{v}_j (\zeta)|^p},
$$
where $\hat{v}_i$ is the Fourier Transform of the individual burst image $v_i$. 
The weight $w_i  := w_i(\zeta)$ controls the contribution of the frequency $\zeta$ of image $v_i$ to the final reconstruction $u_p$. 
Given $\zeta$, for $p>0$, the larger the value of  
$|\hat{v}_i (\zeta)|$, the more $\hat{v}_i (\zeta)$ contributes to the average, reflecting what we discussed above that the strongest 
frequency values represent the least attenuated $u$ components. 

The integer $p$ controls the aggregation of the images in the Fourier domain. 
If $p=0$, the restored image  is just the arithmetic average of the burst (as standard for example in the case of noise only),
while if $p\to \infty$, each reconstructed frequency takes the maximum value of that frequency 
along the burst. This is stated in the following claim; the proof is straightforward and it is 
therefore omitted.
\begin{claim}Mean/Max aggregation.
Suppose that $\hat{v}_i(\zeta)$ for $i=1,\ldots,M$ are such that 
$|\hat{v}_{i_1} (\zeta)| = |\hat{v}_{i_2}(\zeta)| = \ldots = |\hat{v}_{i_q}(\zeta)| > |\hat{v}_{i_{q+1}}(\zeta)| \ge |\hat{v}_{i_{q+2}}(\zeta)|  \ge \ldots \ge |\hat{v}_{i_{M}}(\zeta)|$
and $w_i(\zeta)$ is given by \eqref{eq:fourierWeightsOrig}.  If  $p=0$, then $w_i (\zeta) = \frac{1}{M}, \forall i$ (arithmetic mean pooling), 
while if  
$p\to \infty$, then $w_i (\zeta) = \frac{1}{q}$  for $i = i_1,\ldots,i_q$ and  $w_i (\zeta)=0$ otherwise (maximum pooling).
\end{claim}

The Fourier weights only depend on the Fourier magnitude and hence they are not sensitive to potential image misalignment. 
However, when doing the average in \eqref{eq:fourierWeightsOrig},  the  images $v_i$ have to be correctly aligned to mitigate 
Fourier phase intermingling and get a sharp aggregation. The images in our experiments are aligned  using SIFT correspondences and then finding 
the dominant homography between each image in the burst and the first one (implementation details are given in Section~\ref{sec:algodetails}). This pre-alignment step can be done exploiting the camera gyroscope and accelerometer data. %

\subsubsection*{Dealing with noise}
The images in the burst are blurry but also contaminated with noise.
In the ideal case where the input images are not contaminated with noise,
\eqref{eq:fourierWeightsOrig} is reduced to
\begin{align}
w_i   =   \frac{ |\hat{v}_i  |^p }{\sum_{j=1}^M | \hat{v}_j |^p} =  \frac{ | \hat{k}_i  \cdot \hat{u} |^p }{\sum_{j=1}^M | \hat{k}_j  \cdot \hat{u} |^p} = \frac{ | \hat{k}_i  |^p }{\sum_{j=1}^M | \hat{k}_j |^p},
\label{eq:idealWeights}
\end{align}
as long as $|\hat{u}|>0$.  This is what we would like to have: a procedure for weighting more  the frequencies that are less attenuated
by the different camera shake kernels.
\mdG{Since camera shake kernels have typically a small support, of maximum only a few tenths of pixels,
their Fourier spectra magnitude vary smoothly.} Thus,  $|\hat{v}_i|$  can be smoothed out before computing the weights.
This helps to remove noise and also to stabilize the weights
(see Section~\ref{sec:algodetails}).

\subsection{Equivalent Point Spread Function}
The aggregation procedure can be seen as the convolution of the underlying sharp image $u$ with an average kernel $k_\textsc{fba}$ given by the Fourier weights in~\eqref{eq:idealWeights},
\begin{align}
{u}_p =  u \star k_\textsc{fba} + \bar{n},
\label{eq:equivalentKernel}
\end{align}
where
\begin{align}
k_\textsc{fba} (\bx)  = \mathcal{F}^{-1} \left( \sum_{i=1}^M w_i (\zeta) \cdot \hat{k}_i (\zeta)   \right)(\bx),
\end{align}
and $\bar{n}$ is the weighted average of the input noise. 

The \textsc{fba} kernel can be seen as the final point spread function (\textsc{psf}) obtained by the aggregation procedure (assuming a perfect alignment). 
The closer the \textsc{fba} kernel is to a Dirac function, the better the Fourier aggregation works. 
By construction, the average kernel is made from the least attenuated frequencies in the burst ---given by the 
Fourier weights.  However, since arbitrary convolution kernels may also introduce phase distortion, there 
is no guarantee in general that this aggregation procedure will lead to
an equivalent \textsc{psf} that is closer to a Dirac mass.

In Figure~\ref{fig:equivKernel} we show a series of input images and the respective motion kernels. 
\mdG{The motion kernels were estimated using a sharp snapshot, captured with a tripod. Using the sharp
reference $u_\text{ref}$ we solve for a blurring kernel $k_i$ minimizing the least squares distance to the blurred acquisition $v_i$, namely, 
$|| u_\text{ref} \star k_i - v_i||$  (see e.g., \cite{delbracio2012nonparametric} for a similar setup)}.
The figure also shows the real and imaginary parts of the Fourier kernels spectra. As one can see for most of the kernels 
the real part is mostly positive and  considerably larger than the imaginary part. This implies that the less attenuated frequencies 
will not present significant phase distortion. This assumption may not be accurate for large motion kernels, 
an unexpected scenario in ordinary camera shake. 
In this example, as $p$ increases the equivalent point spread function gets closer to a Dirac function. 
In particular, the \textsc{fba} kernel for $p>0$ attenuates significantly  
less the high frequencies than the regular arithmetic average ($p=0$), leading to a sharper aggregation.

\section{Fourier Burst Accumulation Analysis}
\label{sec:fbaAnalysis}

\mdG{
\subsection{Anatomy of the Fourier Accumulation}}

The value of $p$ sets a tradeoff between  sharpness and noise reduction. Although one would always prefer to get a sharp image, 
due to noise and the unknown Fourier phase shift introduced by the aggregation procedure, the resultant image would
not necessary be better as $p\to \infty$. 
Figure \ref{fig:weightsTorres} shows an example of the proposed \textsc{fba} for a burst of 7 images,
and the amount of contribution of each frame to the final aggregation.
As $p$ grows, the weights are concentrated in fewer images (Figure~\ref{fig:weightsTorres} c) and d)). Also, the weights maps clearly show
that different Fourier frequencies are recovered from different frames (Figure~\ref{fig:weightsTorres} a)). In this example, the high frequency
content is uniformly taken from all the frames in the burst. This produces a strong noise reduction behavior,
in addition to the sharpening effect.

\begin{figure}[htpb]
\centering
\cfboxR{0.8pt}{black}{\includegraphics[width=0.1345\columnwidth]{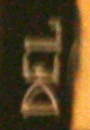}}\hspace{-0.25em}
\cfboxR{0.8pt}{black}{\includegraphics[width=0.1345\columnwidth]{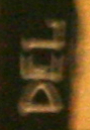}}\hspace{-0.25em}
\cfboxR{0.8pt}{black}{\includegraphics[width=0.1345\columnwidth]{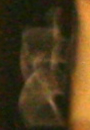}}\hspace{-0.25em}
\cfboxR{0.8pt}{black}{\includegraphics[width=0.1345\columnwidth]{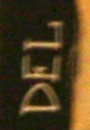}}\hspace{-0.25em}
\cfboxR{0.8pt}{black}{\includegraphics[width=0.1345\columnwidth]{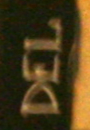}}\hspace{-0.25em}
\cfboxR{0.8pt}{black}{\includegraphics[width=0.1345\columnwidth]{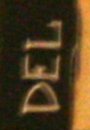}}\hspace{-0.25em}
\cfboxR{0.8pt}{black}{\includegraphics[width=0.1345\columnwidth]{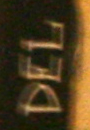}}

\vspace{.1em}

\cfboxR{0.8pt}{black}{\includegraphics[width=0.1345\columnwidth, height=0.1345\columnwidth]{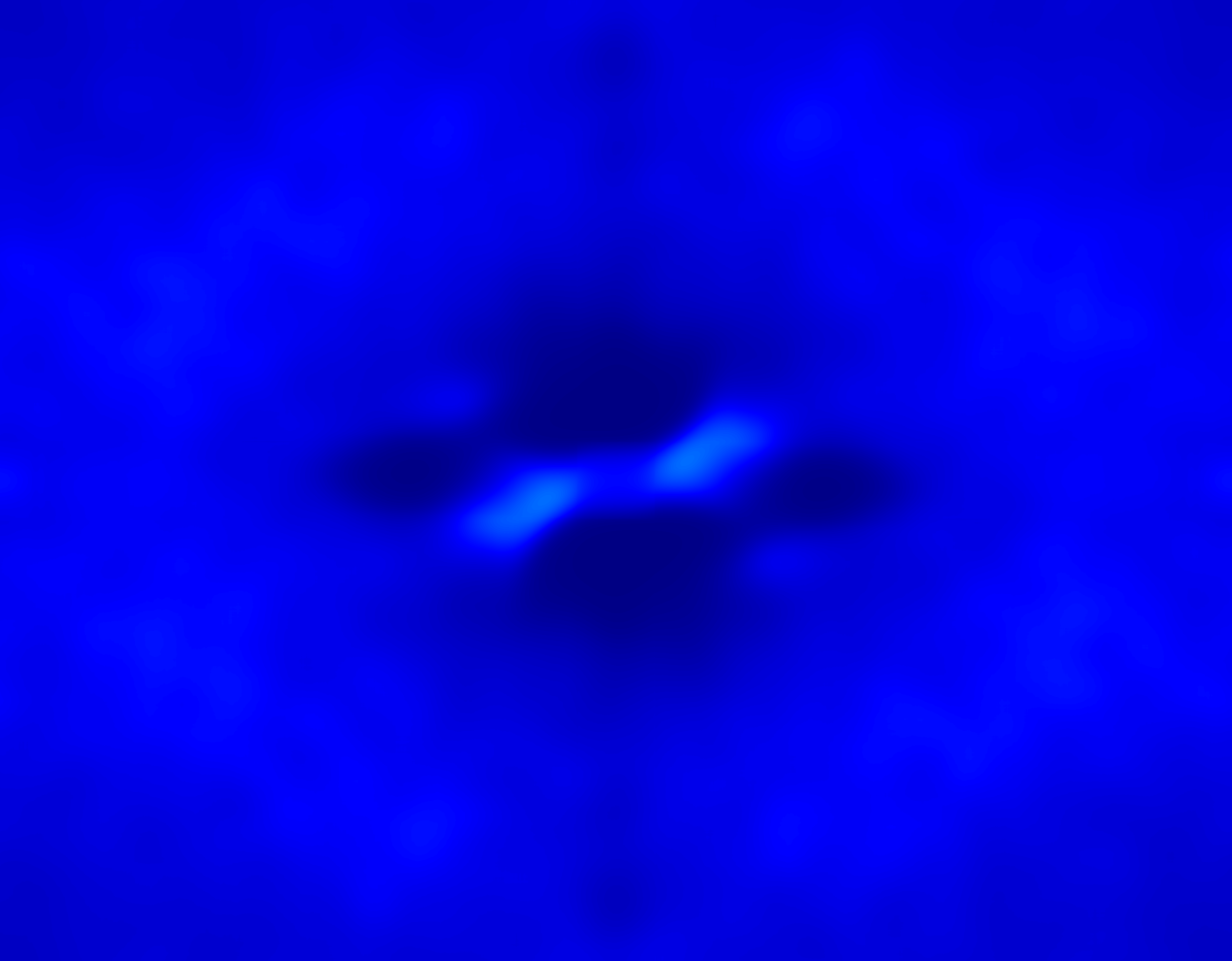}}\hspace{-0.25em}
\cfboxR{0.8pt}{black}{\includegraphics[width=0.1345\columnwidth, height=0.1345\columnwidth]{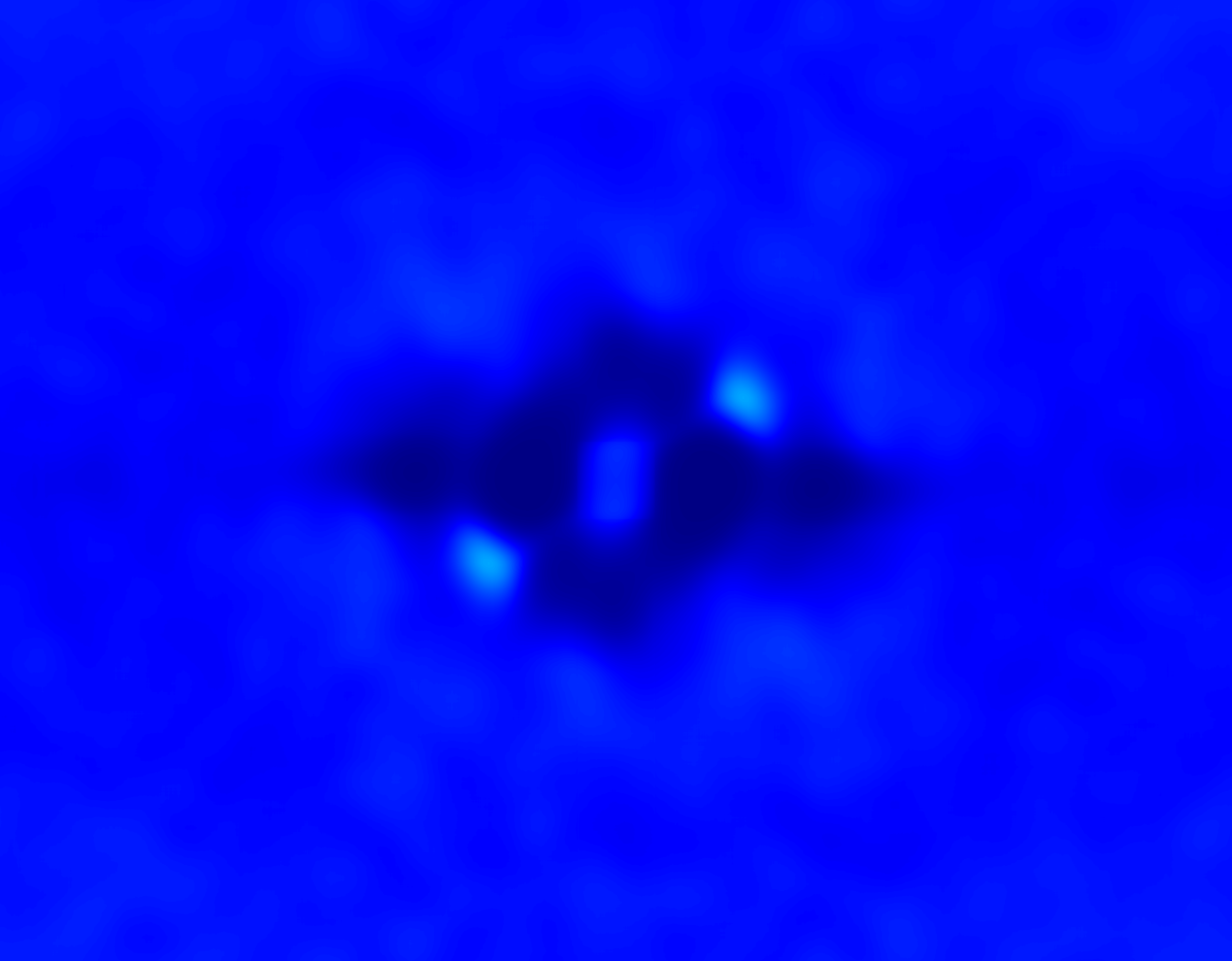}}\hspace{-0.25em}
\cfboxR{0.8pt}{black}{\includegraphics[width=0.1345\columnwidth, height=0.1345\columnwidth]{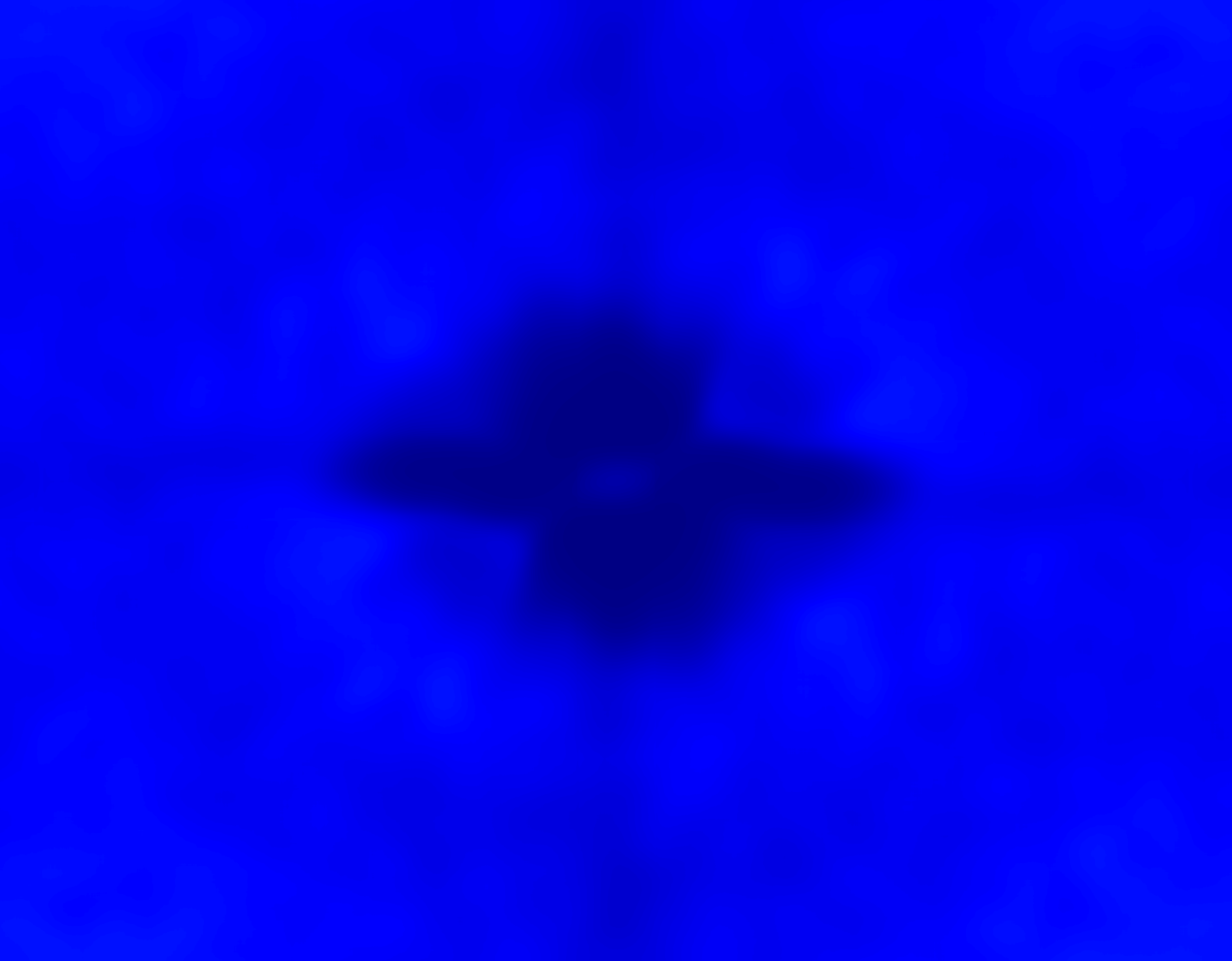}}\hspace{-0.25em}
\cfboxR{0.8pt}{black}{\includegraphics[width=0.1345\columnwidth, height=0.1345\columnwidth]{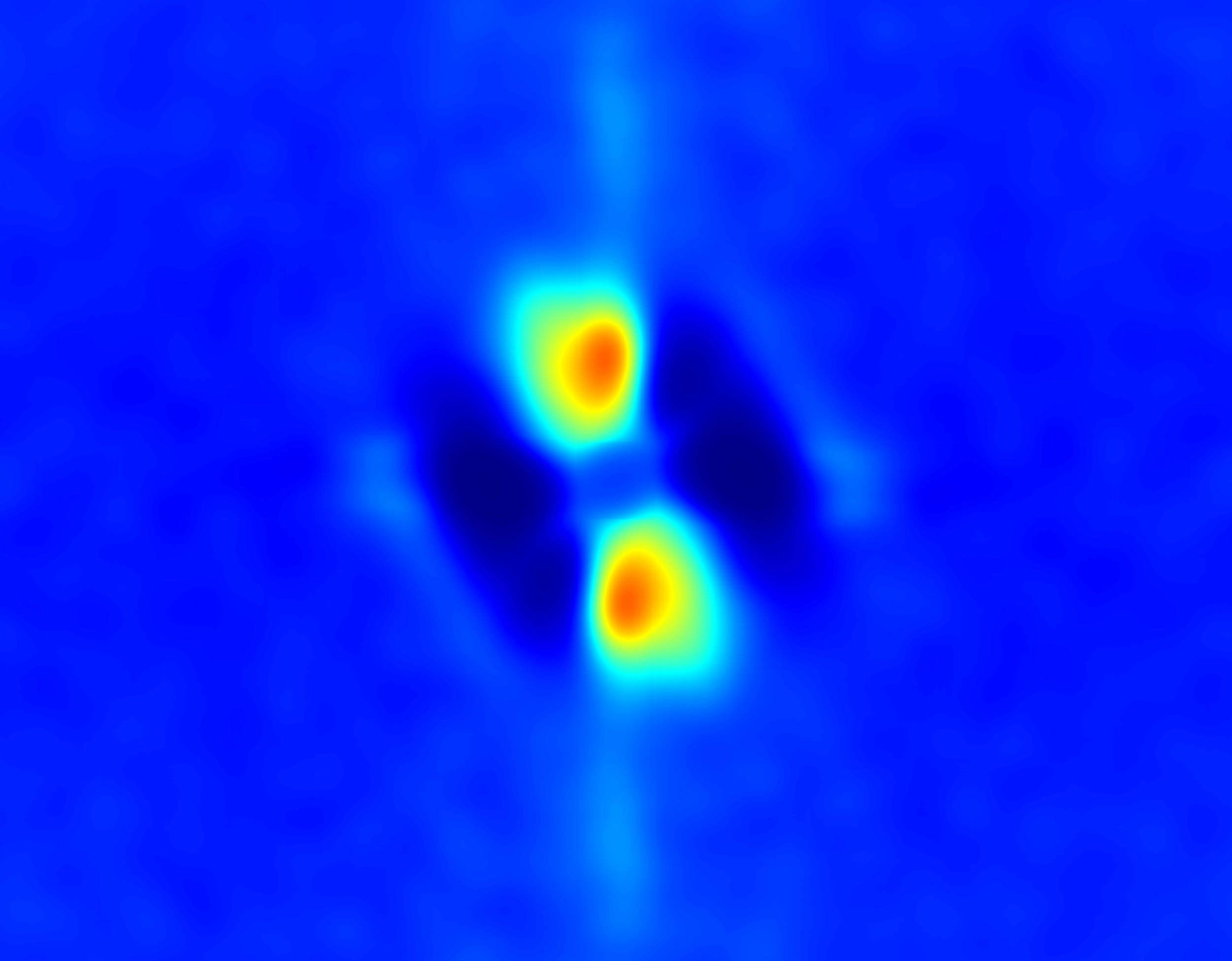}}\hspace{-0.25em}
\cfboxR{0.8pt}{black}{\includegraphics[width=0.1345\columnwidth, height=0.1345\columnwidth]{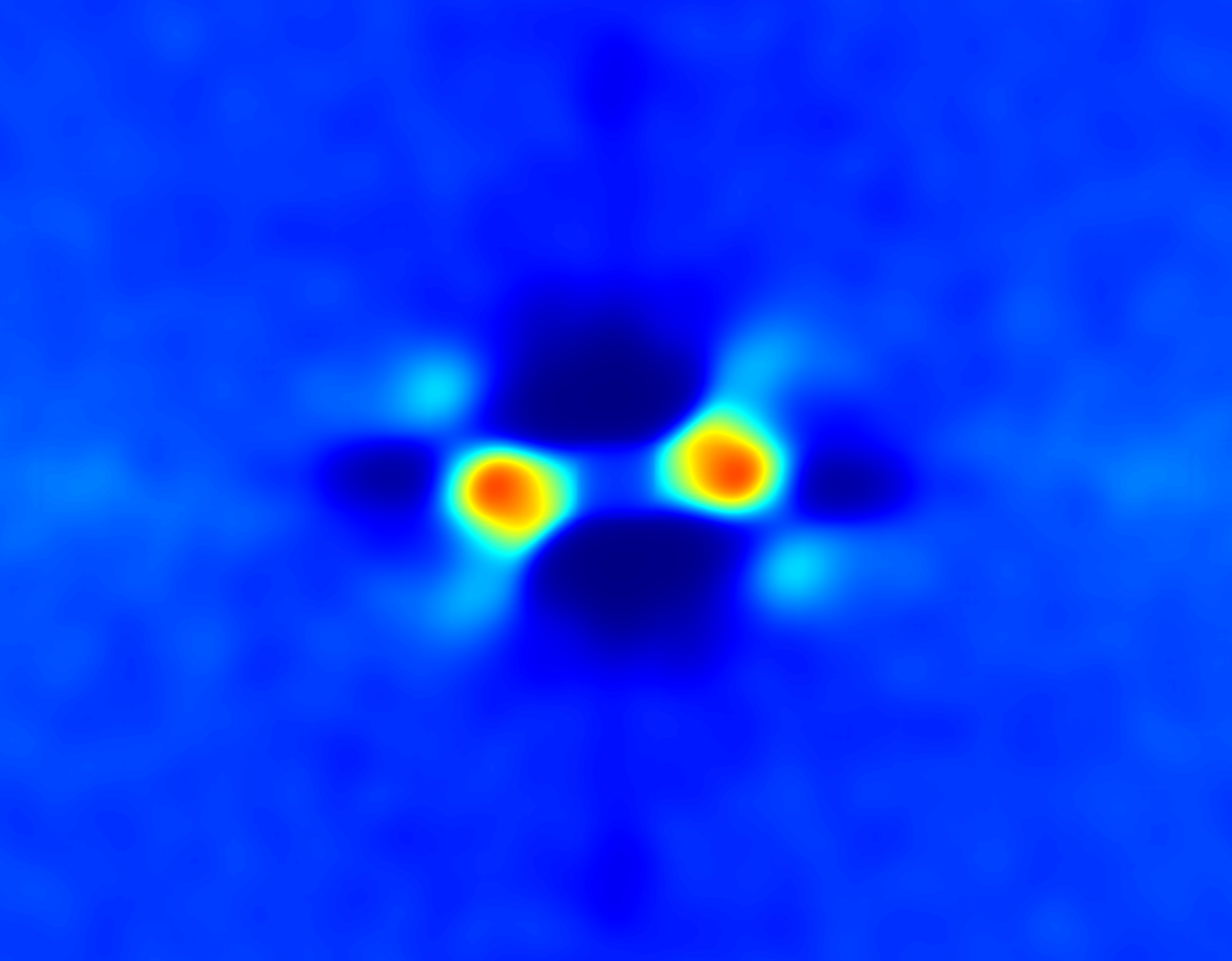}}\hspace{-0.25em}
\cfboxR{0.8pt}{black}{\includegraphics[width=0.1345\columnwidth, height=0.1345\columnwidth]{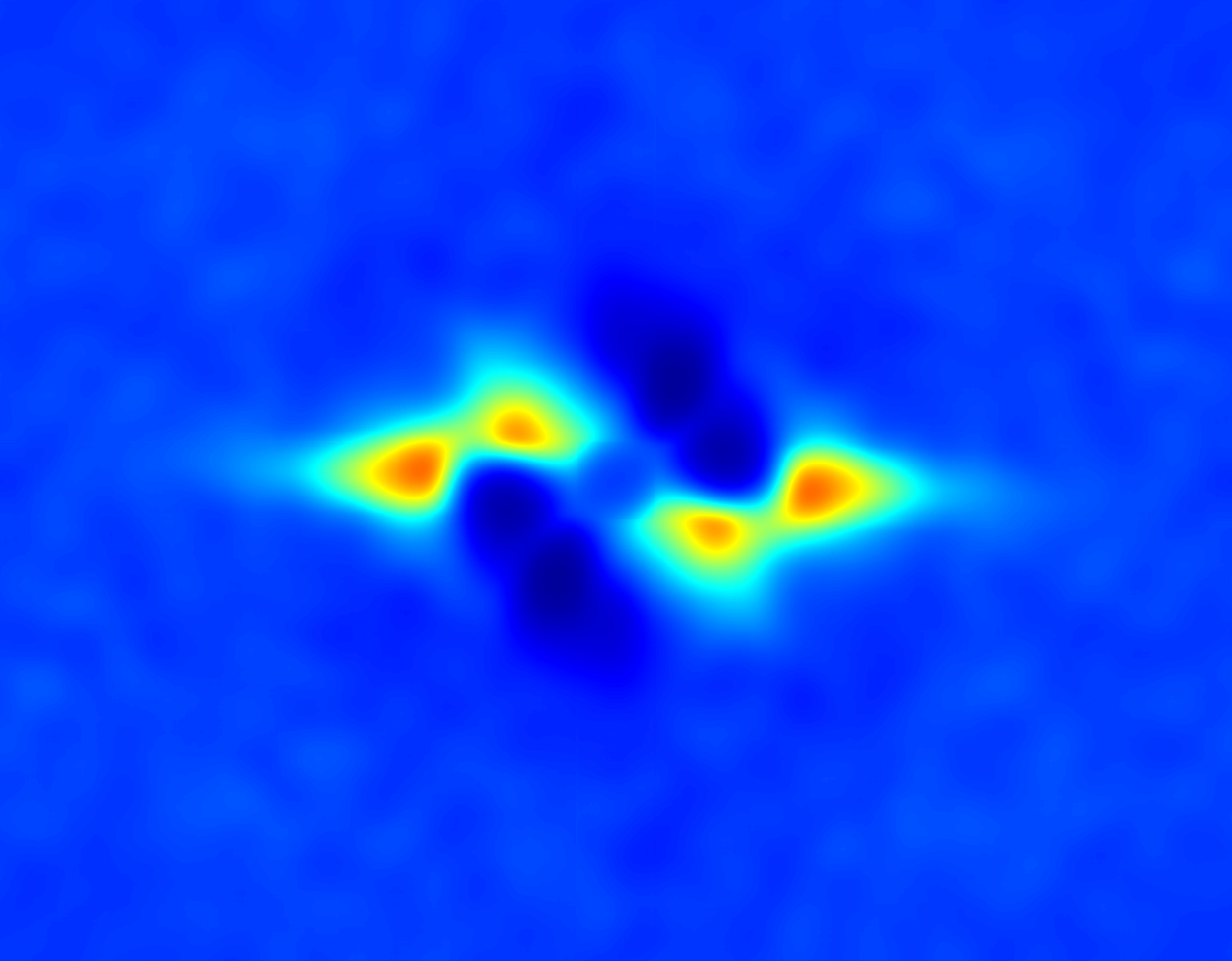}}\hspace{-0.25em}
\cfboxR{0.8pt}{black}{\includegraphics[width=0.1345\columnwidth, height=0.1345\columnwidth]{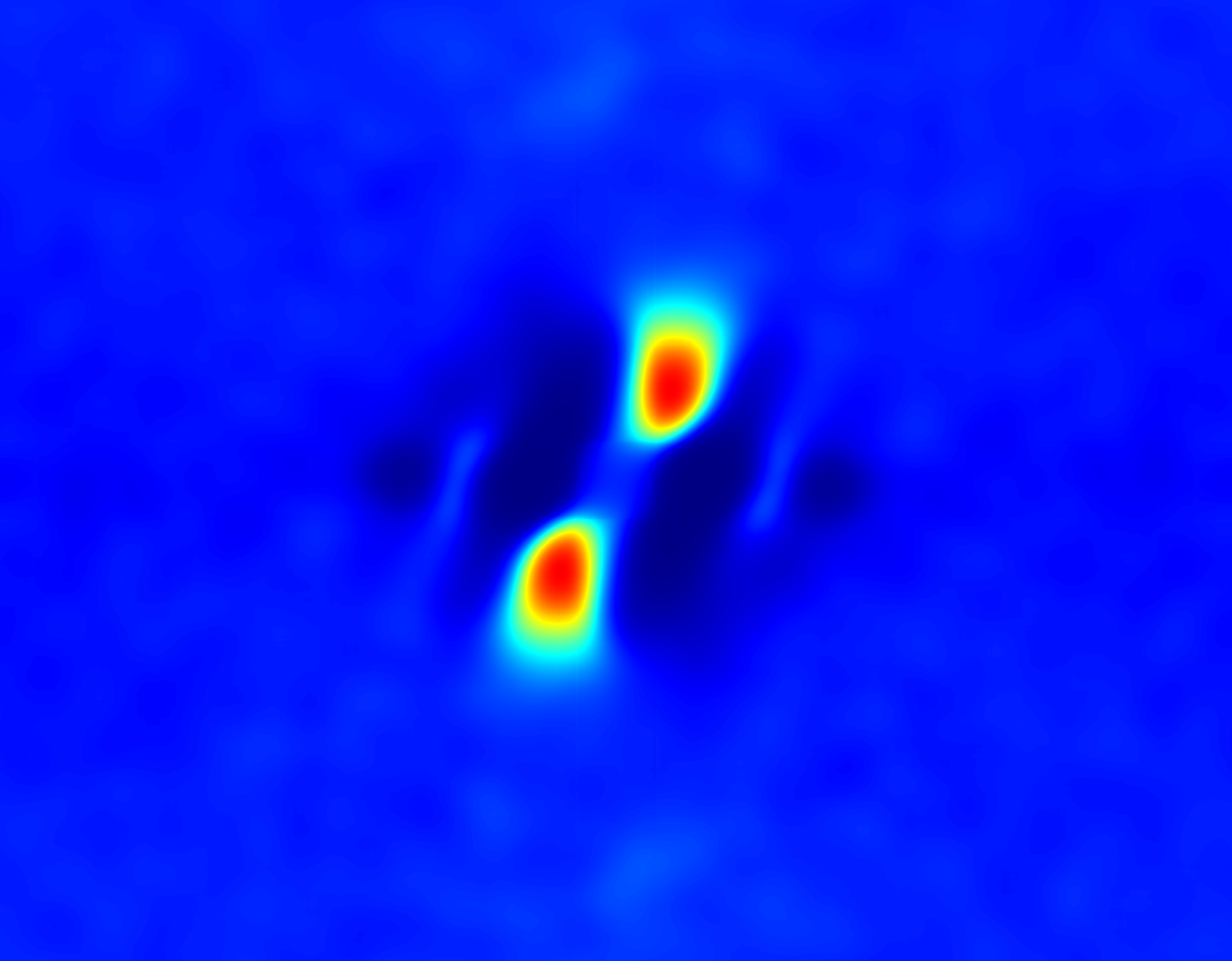}}

\vspace{.5em}

\scriptsize (a) Frames crop 1-7  and the Fourier weights $w_i$ for $p=11$.

\vspace{1.1em}

\begin{minipage}[c]{0.161\columnwidth}
\centering
\cfboxR{0.8pt}{black}{\includegraphics[width=.98\columnwidth]{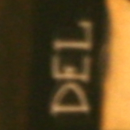}}

\ssmall $p=0$
\end{minipage}\hspace{-0.1em}
\begin{minipage}[c]{0.161\columnwidth}
\centering
\cfboxR{0.8pt}{black}{\includegraphics[width=.98\columnwidth]{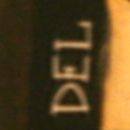}}

\ssmall $p=3$
\end{minipage}\hspace{-0.1em}
\begin{minipage}[c]{0.161\columnwidth}
\centering
\cfboxR{0.8pt}{black}{\includegraphics[width=.98\columnwidth]{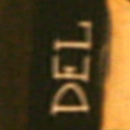}}

\ssmall $p=7$
\end{minipage}\hspace{-0.1em}
\begin{minipage}[c]{0.161\columnwidth}
\centering
\cfboxR{0.8pt}{black}{\includegraphics[width=.98\columnwidth]{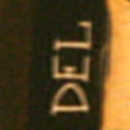}}

\ssmall $p=11$
\end{minipage}\hspace{-0.1em}
\begin{minipage}[c]{0.161\columnwidth}
\centering
\cfboxR{0.8pt}{black}{\includegraphics[width=.98\columnwidth]{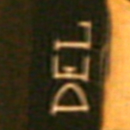}}

\ssmall $p=20$
\end{minipage}\hspace{-0.1em}
\begin{minipage}[c]{0.161\columnwidth}
\centering
\cfboxR{0.8pt}{black}{\includegraphics[width=.98\columnwidth]{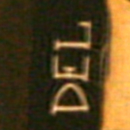}}

\ssmall $p=50$
\end{minipage} \vspace{.5em}

\scriptsize (b) Fourier Aggregation results for different $p$ values.

\vspace{1.2em}

\begin{minipage}[c]{0.49\columnwidth}
\centering
\includegraphics[width=0.98\columnwidth]{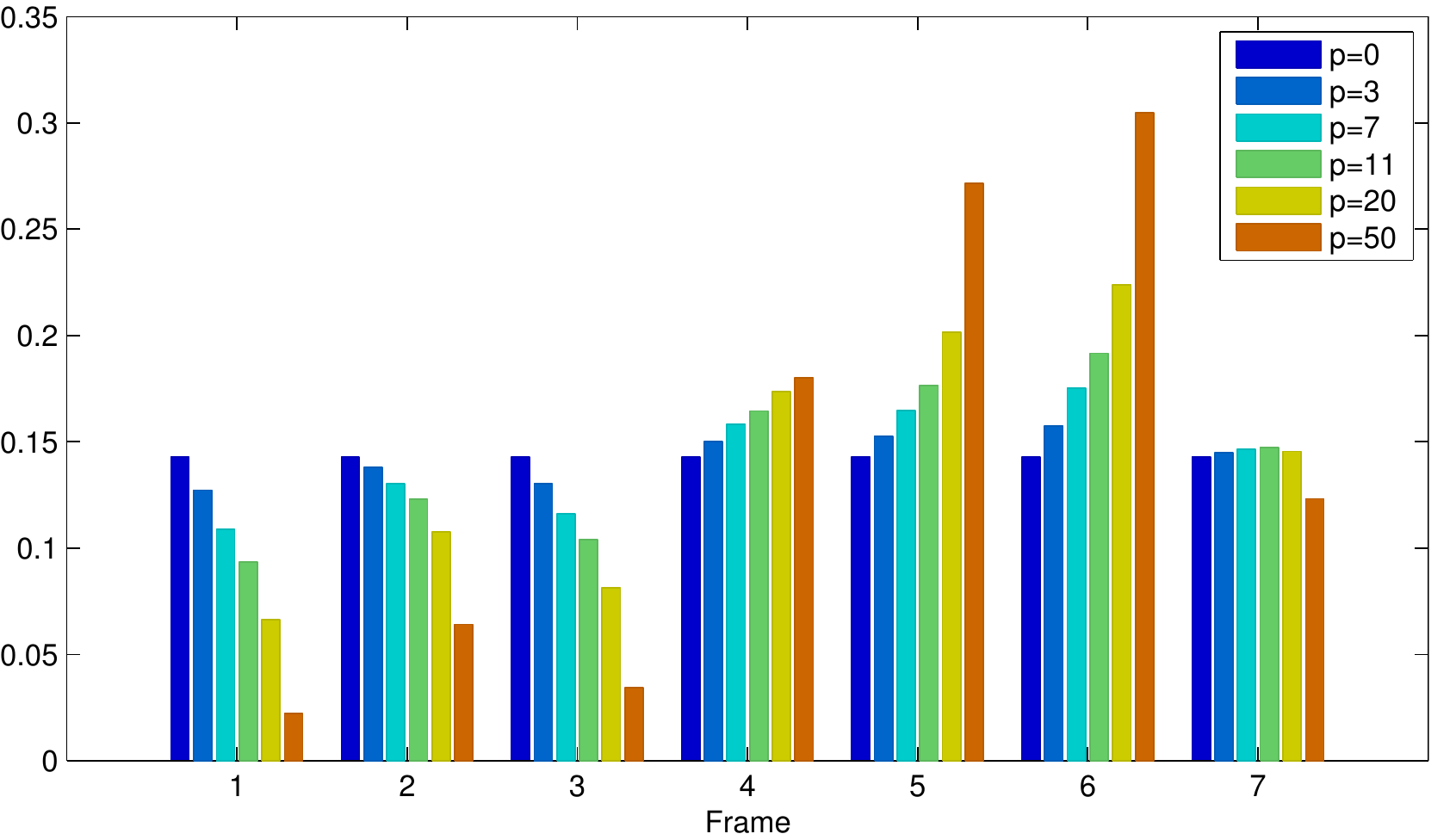}
\scriptsize (c) Weights energy distribution: $\text{we}_i = \int | w_i(\zeta) |^2 d \zeta$ 
\end{minipage}
\begin{minipage}[c]{0.49\columnwidth}
\centering
\includegraphics[width=0.98\columnwidth]{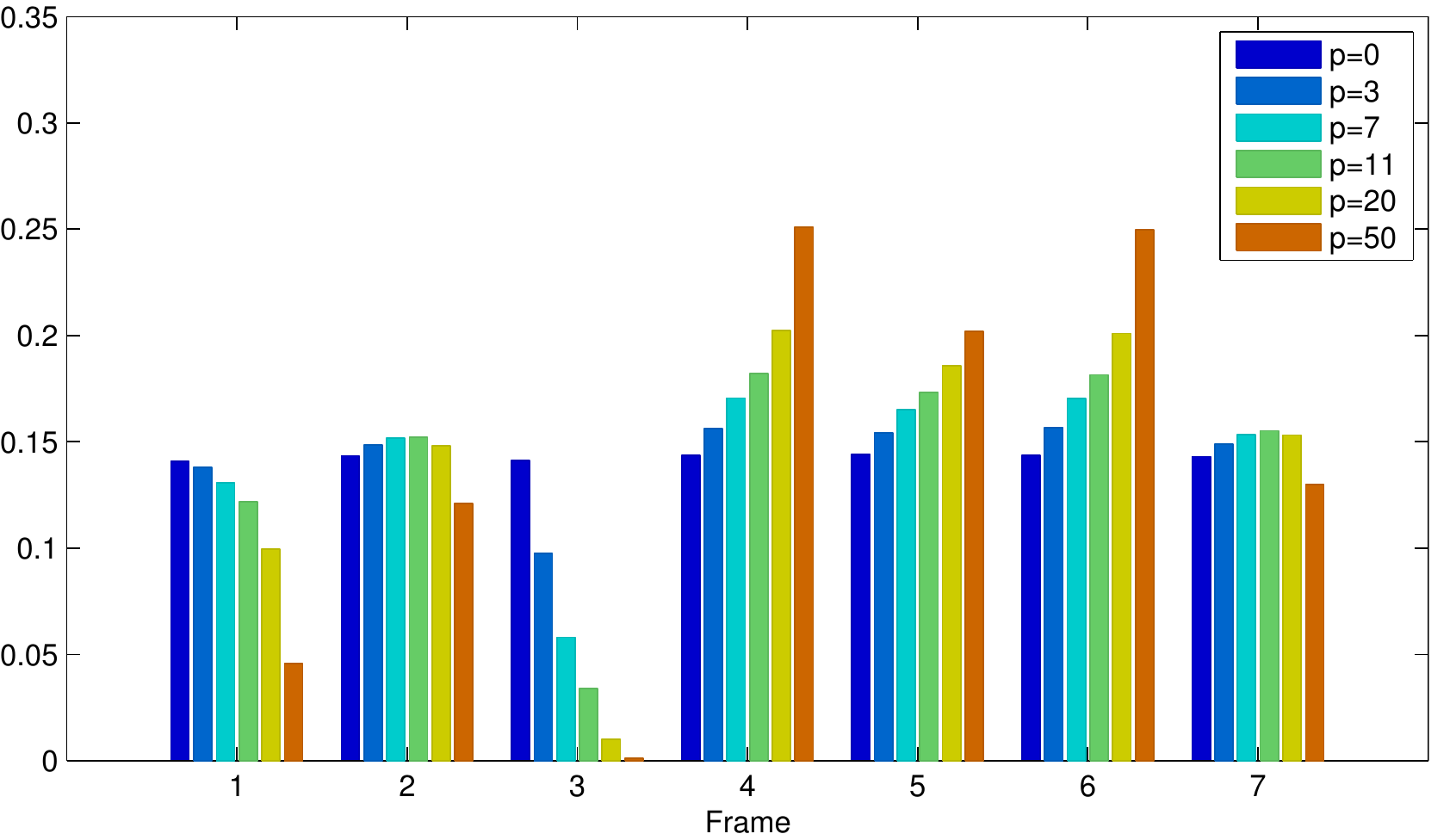}
\scriptsize (d) Weighted frames energy distribution: $\text{wfe}_i = \int | w_i(\zeta) \cdot \hat{v}_i (\zeta) |^2 d \zeta$
\end{minipage}

\caption{Weights distribution of the Fourier Burst Aggregation when changing the value of $p$. As $p$ increases,
the weights are concentrated in fewer images and the aggregated image becomes sharper but also noisier.
The difference between (c) and (d) lies in the fact that a large weight may not necessarily reflect a large contribution
to the final image, although this is generally the case. 
}
\label{fig:weightsTorres}

\end{figure}

\mdG{
To make explicit the contribution of each frame to the final image,
the FBA weighted average \eqref{eq:fourierWeightsOrig} can be decomposed into its contributions:
$$
{u}_p(\bx) =   \sum_{i=1}^M  \mathcal{F}^{-1} \Big(w_i   (\zeta) \cdot \hat{v}_i  (\zeta) \Big) (\bx)  := \sum_{i=1}^M \bar{v}_i  (\bx). 
$$
Each term $\bar{v}_i$ is the result of applying the corresponding Fourier weights $w_i$ (a filter) to
the respective input frame $v_i$. Figure~\ref{fig:anatomy} shows each of these terms
for a crop of a real burst. As the reader may notice, each frame contributes differently. None of the frames capture all the structure present in the final aggregated image.
}

\begin{figure*}
\centering

\begin{minipage}{.5em}
\begin{sideways}

{\ssmall \hspace{.7em}
Weights $w_i$\hspace{3.3em}
Filtered $\bar{v}_i$\hspace{3.6em}
Image $v_i$\hspace{2.4em}
}

\end{sideways}
\end{minipage}
\begin{minipage}{.0925\textwidth}
\centering

{\ssmall input 1} \vspace{.12em}

\cfboxR{0.4pt}{black}{\includegraphics[width=\textwidth]{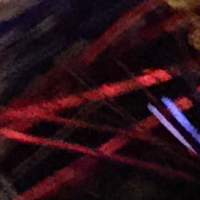}} \vspace{-1em}

\cfboxR{0.4pt}{black}{\includegraphics[width=\textwidth]{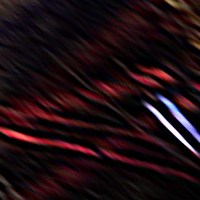}}\vspace{.15em}

\cfboxR{0.4pt}{black}{\includegraphics[width=\textwidth, height=\textwidth]{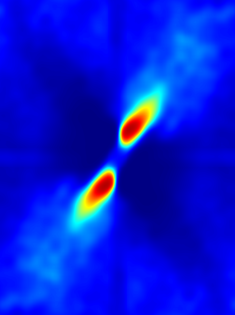}}

\end{minipage}\hspace{-.15em}
\begin{minipage}{.0925\textwidth}
\centering

{\ssmall input 2} \vspace{.12em}

\cfboxR{0.4pt}{black}{\includegraphics[width=\textwidth]{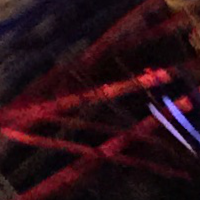}} \vspace{-1em}

\cfboxR{0.4pt}{black}{\includegraphics[width=\textwidth]{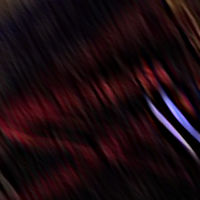}}\vspace{.15em}

\cfboxR{0.4pt}{black}{\includegraphics[width=\textwidth,height=\textwidth]{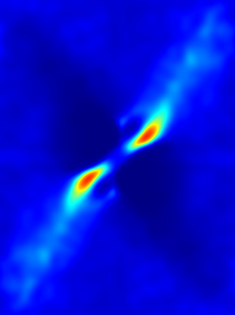}}

\end{minipage}\hspace{-.15em}
\begin{minipage}{.0925\textwidth}
\centering

{\ssmall input 3} \vspace{.12em}

\cfboxR{0.4pt}{black}{\includegraphics[width=\textwidth]{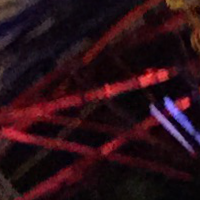}} \vspace{-1em}

\cfboxR{0.4pt}{black}{\includegraphics[width=\textwidth]{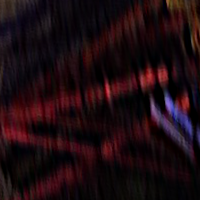}}\vspace{.15em}

\cfboxR{0.4pt}{black}{\includegraphics[width=\textwidth,height=\textwidth]{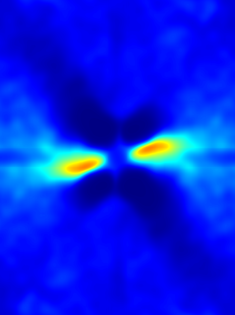}}

\end{minipage}\hspace{-.15em}
\begin{minipage}{.0925\textwidth}
\centering

{\ssmall input 4} \vspace{.12em}

\cfboxR{0.4pt}{black}{\includegraphics[width=\textwidth]{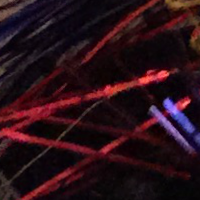}} \vspace{-1em}

\cfboxR{0.4pt}{black}{\includegraphics[width=\textwidth]{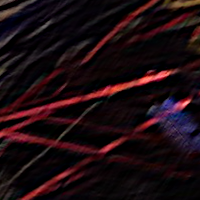}}\vspace{.15em}

\cfboxR{0.4pt}{black}{\includegraphics[width=\textwidth,height=\textwidth]{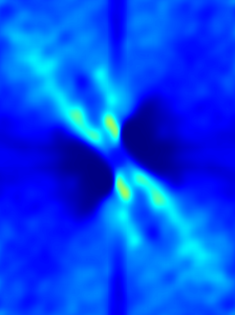}}

\end{minipage}\hspace{-.15em}
\begin{minipage}{.0925\textwidth}
\centering

{\ssmall input 5} \vspace{.12em}

\cfboxR{0.4pt}{black}{\includegraphics[width=\textwidth]{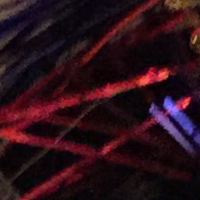}} \vspace{-1em}

\cfboxR{0.4pt}{black}{\includegraphics[width=\textwidth]{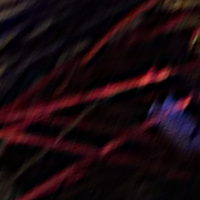}}\vspace{.15em}

\cfboxR{0.4pt}{black}{\includegraphics[width=\textwidth,height=\textwidth]{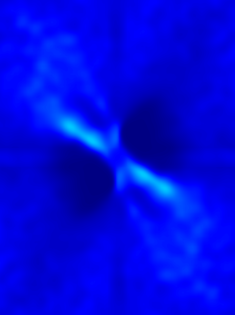}}

\end{minipage}\hspace{-.15em}
\begin{minipage}{.0925\textwidth}
\centering

{\ssmall input 6} \vspace{.12em}

\cfboxR{0.4pt}{black}{\includegraphics[width=\textwidth]{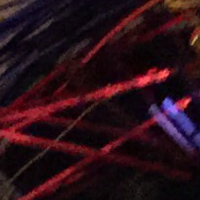}} \vspace{-1em}

\cfboxR{0.4pt}{black}{\includegraphics[width=\textwidth]{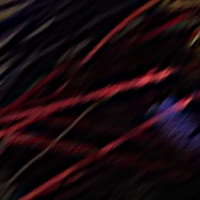}}\vspace{.15em}

\cfboxR{0.4pt}{black}{\includegraphics[width=\textwidth,height=\textwidth]{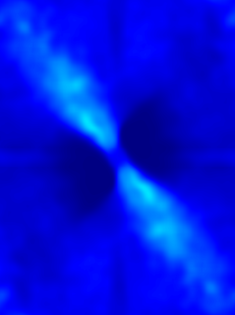}}

\end{minipage}\hspace{-.15em}
\begin{minipage}{.0925\textwidth}
\centering

{\ssmall input 7} \vspace{.12em}

\cfboxR{0.4pt}{black}{\includegraphics[width=\textwidth]{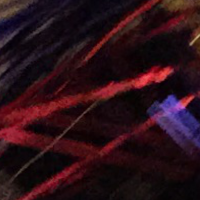}} \vspace{-1em}

\cfboxR{0.4pt}{black}{\includegraphics[width=\textwidth]{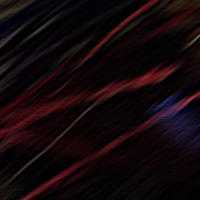}}\vspace{.15em}

\cfboxR{0.4pt}{black}{\includegraphics[width=\textwidth,height=\textwidth]{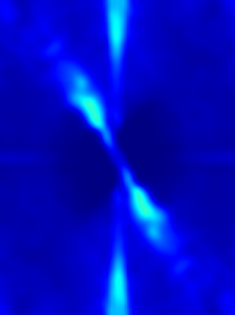}}

\end{minipage}\hspace{-.15em}
\begin{minipage}{.0925\textwidth}
\centering

{\ssmall input 8} \vspace{.12em}

\cfboxR{0.4pt}{black}{\includegraphics[width=\textwidth]{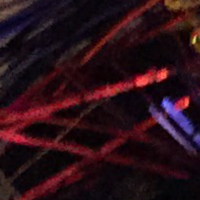}} \vspace{-1em}

\cfboxR{0.4pt}{black}{\includegraphics[width=\textwidth]{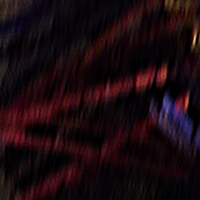}}\vspace{.15em}

\cfboxR{0.4pt}{black}{\includegraphics[width=\textwidth, height=\textwidth]{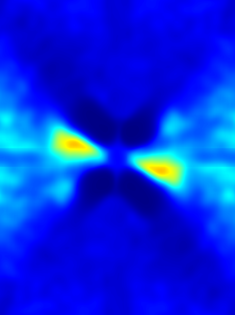}}

\end{minipage}\hspace{-.15em}
\begin{minipage}{.0925\textwidth}
\centering

{\ssmall input 9} \vspace{.12em}

\cfboxR{0.4pt}{black}{\includegraphics[width=\textwidth]{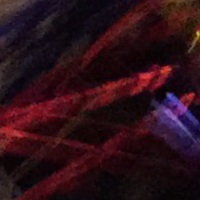}} \vspace{-1em}

\cfboxR{0.4pt}{black}{\includegraphics[width=\textwidth]{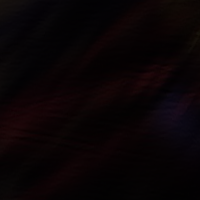}}\vspace{.15em}

\cfboxR{0.4pt}{black}{\includegraphics[width=\textwidth, height=\textwidth]{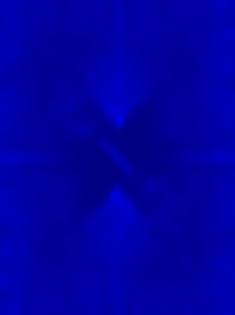}}

\end{minipage}\hspace{-.15em}
\begin{minipage}{.0925\textwidth}
\centering

{\ssmall \textbf{FBA}} \vspace{.12em}

\cfboxR{0.4pt}{black}{\includegraphics[width=\textwidth]{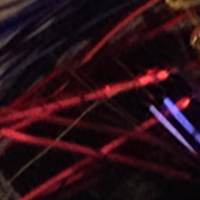}} 

\vspace{5.9em}

{\includegraphics[width=1.1\textwidth]{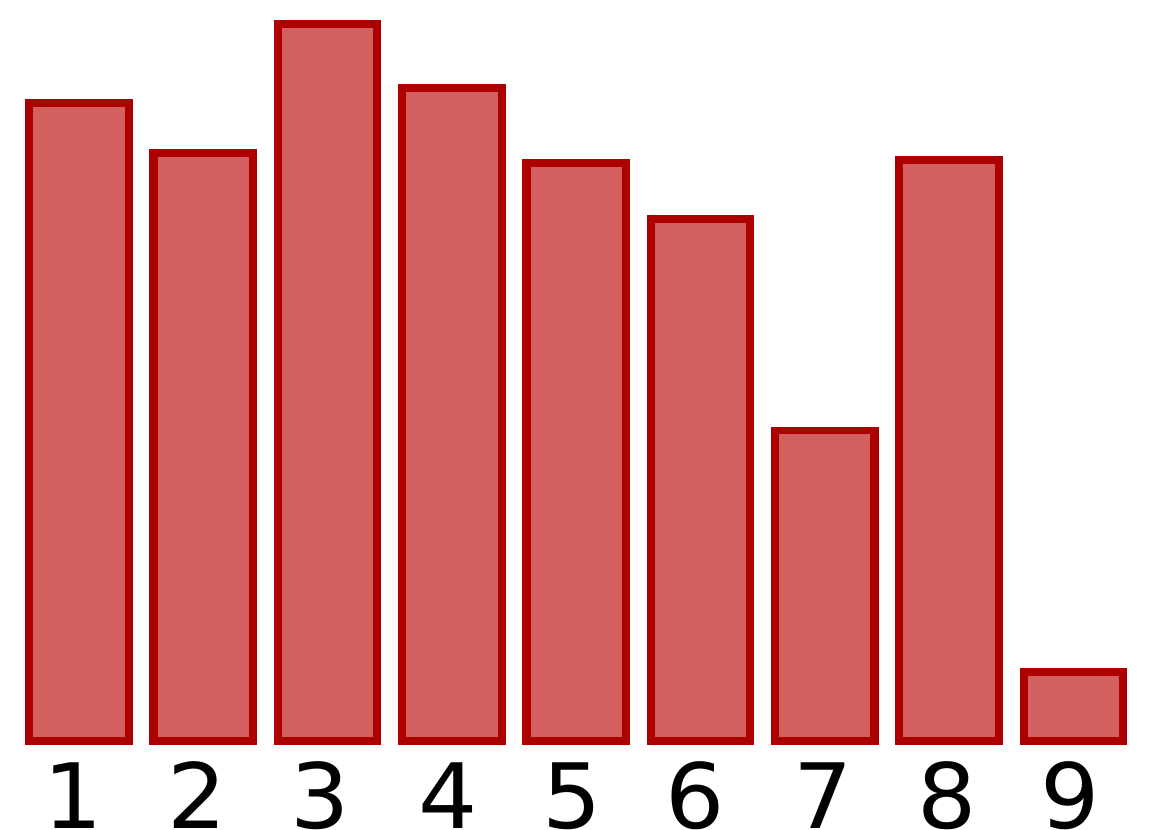} }

\vspace{-1em}

\end{minipage}

\mdG{\caption{Anatomy of the Fourier aggregation.  The first row shows a crop
of each input image $v_i$ (frames 1 to 9) and the proposed Fourier Burst Accumulation for $p=11$ (from~\eqref{eq:fourierWeightsOrig}, no additional sharpening).
The second row shows the contribution of each frame to the final aggregation $\bar{v}_i$ (rescaled for better visualization).
The reader can check that the \textsc{fba} results from the aggregation of different components presented in different frames. This is also confirmed by the Fourier weights distribution shown in the bottom row. The bar plot shown on the right indicates the total contribution of each frame to the \textsc{fba} image. None of the input images contain all the structure presented in the final aggregation.}
\label{fig:anatomy}

}

\end{figure*}

\subsection{Statistical Performance Analysis}
To show the statistical performance of the Fourier weighted accumulation, we carried out an empirical 
analysis applying the proposed aggregation with different values of $p$. We simulated motion kernels following~\cite{gavant2011physiological},
where the (expected value) amount of blur is controlled by a parameter related to the exposure time $T_\text{exp}$. 
The simulated motion kernels were applied to a sharp clean image (ground truth).
We also controlled
the number of images in the burst $M$ and the  noise level in each frame $s$.  The kernels were generated by simulating 
the random shake of the camera from the power spectral density of measured physiological hand tremor~\cite{carignan2010quantifying}. 
All the images were  aligned by pre-centering the motion kernels before blurring the underlying sharp image. 
Figure~\ref{fig:biasVariance} shows several different realizations for different exposure values $T_\text{exp}$. 
Actually, the amount of blur not only depends on the exposure time but also on the focal distance, user expertise, 
and camera dimensions and mass~\cite{boracchi2012modeling}. However, for simplicity, all these variables were
controlled by the single parameter $T_\text{exp}$.

We randomly sampled thousands of different motion kernels and Gaussian noise realizations, and then applied the Fourier aggregation 
procedure for each different configuration ($T_\text{exp}, M, s$) several times.
The empirical mean square error (\textsc{mse}) between the ground truth reference and each \textsc{fba} restoration was computed and averaged
over hundreds of independent realizations.
We decomposed the mean square error into the bias and variance 
terms, namely $\textsc{mse}(u_p) = \text{bias}(u_p)^2 + \text{var}(u_p)$, to help visualize the behavior of the algorithm.

Figure~\ref{fig:biasVariance} shows the average algorithm performance when changing the acquisition conditions. In general, the larger the
value of $p$ the smaller the bias and the larger the variance. There is a minimum of the mean square error for $p \in [7,30]$. 
This is reasonable since there exists a tradeoff between variance reduction and bias. Although both the bias and the variance are affected by the noise level, 
the qualitative performance of the algorithm remains the same. The bias is not altered 
by the number of frames in the burst but the variance is reduced as more images are used, implying a gain in the expected \textsc{mse} as more
images are used. On the other hand, the exposure time
mostly affects the bias, being much more significant for larger exposures as expected.

\begin{figure*}[hptb]
\begin{minipage}[c]{0.3\textwidth}

\begin{center}

\begin{minipage}[c]{0.03\textwidth}
\begin{sideways}\tiny $\,\,\nicefrac{1}{10}$\end{sideways}
\end{minipage}\hspace{-.1em}
\begin{minipage}[c]{0.9\textwidth}
\cfboxR{0.5pt}{black}{\includegraphics[width=.132\columnwidth]{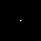}}\hspace{-.23em}
\cfboxR{0.5pt}{black}{\includegraphics[width=.132\columnwidth]{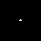}}\hspace{-.23em}
\cfboxR{0.5pt}{black}{\includegraphics[width=.132\columnwidth]{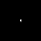}}\hspace{-.23em}
\cfboxR{0.5pt}{black}{\includegraphics[width=.132\columnwidth]{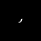}}\hspace{-.23em}
\cfboxR{0.5pt}{black}{\includegraphics[width=.132\columnwidth]{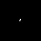}}\hspace{-.23em}
\cfboxR{0.5pt}{black}{\includegraphics[width=.132\columnwidth]{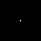}}\hspace{-.23em}
\cfboxR{0.5pt}{black}{\includegraphics[width=.132\columnwidth]{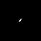}}

\vspace{.1em}

\cfboxR{0.5pt}{black}{\includegraphics[width=.132\columnwidth]{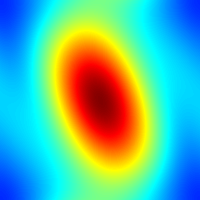}}\hspace{-.23em}
\cfboxR{0.5pt}{black}{\includegraphics[width=.132\columnwidth]{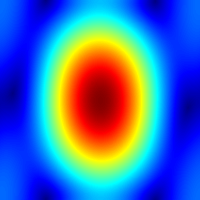}}\hspace{-.23em}
\cfboxR{0.5pt}{black}{\includegraphics[width=.132\columnwidth]{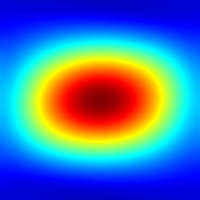}}\hspace{-.23em}
\cfboxR{0.5pt}{black}{\includegraphics[width=.132\columnwidth]{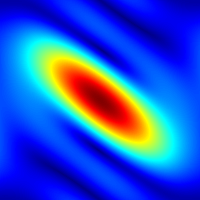}}\hspace{-.23em}
\cfboxR{0.5pt}{black}{\includegraphics[width=.132\columnwidth]{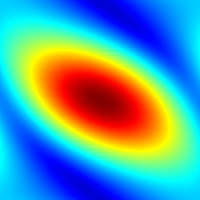}}\hspace{-.23em}
\cfboxR{0.5pt}{black}{\includegraphics[width=.132\columnwidth]{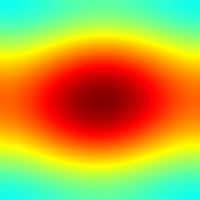}}\hspace{-.23em}
\cfboxR{0.5pt}{black}{\includegraphics[width=.132\columnwidth]{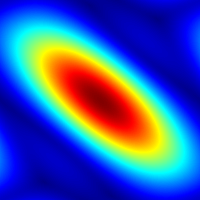}}
\end{minipage}

\vspace{.4em}

\begin{minipage}[c]{0.03\textwidth}
\begin{sideways}\tiny $\,\,\nicefrac{1}{5}$\end{sideways}
\end{minipage}\hspace{-.1em}
\begin{minipage}[c]{0.9\textwidth}
\cfboxR{0.5pt}{black}{\includegraphics[width=.132\columnwidth]{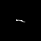}}\hspace{-.23em}
\cfboxR{0.5pt}{black}{\includegraphics[width=.132\columnwidth]{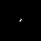}}\hspace{-.23em}
\cfboxR{0.5pt}{black}{\includegraphics[width=.132\columnwidth]{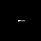}}\hspace{-.23em}
\cfboxR{0.5pt}{black}{\includegraphics[width=.132\columnwidth]{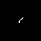}}\hspace{-.23em}
\cfboxR{0.5pt}{black}{\includegraphics[width=.132\columnwidth]{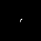}}\hspace{-.23em}
\cfboxR{0.5pt}{black}{\includegraphics[width=.132\columnwidth]{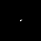}}\hspace{-.23em}
\cfboxR{0.5pt}{black}{\includegraphics[width=.132\columnwidth]{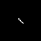}}

\vspace{.1em}

\cfboxR{0.5pt}{black}{\includegraphics[width=.132\columnwidth]{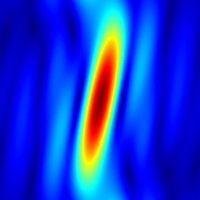}}\hspace{-.23em}
\cfboxR{0.5pt}{black}{\includegraphics[width=.132\columnwidth]{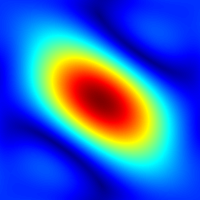}}\hspace{-.23em}
\cfboxR{0.5pt}{black}{\includegraphics[width=.132\columnwidth]{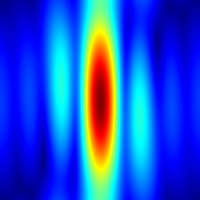}}\hspace{-.23em}
\cfboxR{0.5pt}{black}{\includegraphics[width=.132\columnwidth]{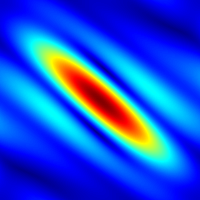}}\hspace{-.23em}
\cfboxR{0.5pt}{black}{\includegraphics[width=.132\columnwidth]{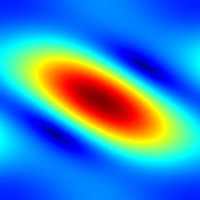}}\hspace{-.23em}
\cfboxR{0.5pt}{black}{\includegraphics[width=.132\columnwidth]{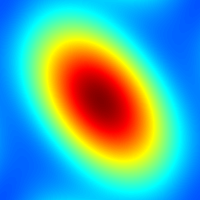}}\hspace{-.23em}
\cfboxR{0.5pt}{black}{\includegraphics[width=.132\columnwidth]{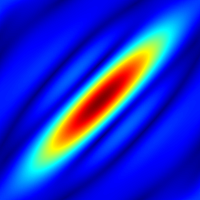}}
\end{minipage}

\vspace{.4em}

\begin{minipage}[c]{0.03\textwidth}
\begin{sideways}\tiny $\,\,\nicefrac{1}{3}$\end{sideways}
\end{minipage}\hspace{-.1em}
\begin{minipage}[c]{0.9\textwidth}
\cfboxR{0.5pt}{black}{\includegraphics[width=.132\columnwidth]{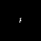}}\hspace{-.23em}
\cfboxR{0.5pt}{black}{\includegraphics[width=.132\columnwidth]{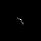}}\hspace{-.23em}
\cfboxR{0.5pt}{black}{\includegraphics[width=.132\columnwidth]{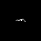}}\hspace{-.23em}
\cfboxR{0.5pt}{black}{\includegraphics[width=.132\columnwidth]{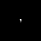}}\hspace{-.23em}
\cfboxR{0.5pt}{black}{\includegraphics[width=.132\columnwidth]{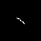}}\hspace{-.23em}
\cfboxR{0.5pt}{black}{\includegraphics[width=.132\columnwidth]{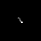}}\hspace{-.23em}
\cfboxR{0.5pt}{black}{\includegraphics[width=.132\columnwidth]{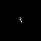}}

\vspace{0.1em}

\cfboxR{0.5pt}{black}{\includegraphics[width=.132\columnwidth]{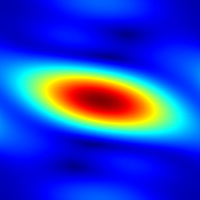}}\hspace{-.23em}
\cfboxR{0.5pt}{black}{\includegraphics[width=.132\columnwidth]{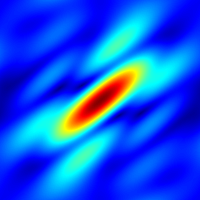}}\hspace{-.23em}
\cfboxR{0.5pt}{black}{\includegraphics[width=.132\columnwidth]{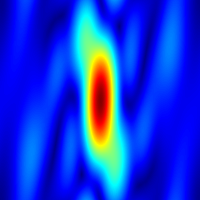}}\hspace{-.23em}
\cfboxR{0.5pt}{black}{\includegraphics[width=.132\columnwidth]{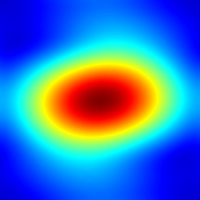}}\hspace{-.23em}
\cfboxR{0.5pt}{black}{\includegraphics[width=.132\columnwidth]{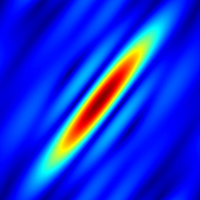}}\hspace{-.23em}
\cfboxR{0.5pt}{black}{\includegraphics[width=.132\columnwidth]{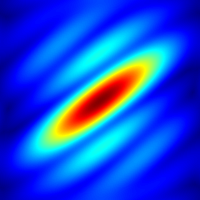}}\hspace{-.23em}
\cfboxR{0.5pt}{black}{\includegraphics[width=.132\columnwidth]{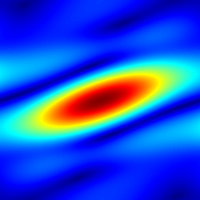}}
\end{minipage}

\vspace{.4em}

\begin{minipage}[c]{0.03\textwidth}
\begin{sideways}\tiny $\,\,\nicefrac{1}{2}$\end{sideways}
\end{minipage}\hspace{-.1em}
\begin{minipage}[c]{0.9\textwidth}
\cfboxR{0.5pt}{black}{\includegraphics[width=.132\columnwidth]{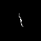}}\hspace{-.23em}
\cfboxR{0.5pt}{black}{\includegraphics[width=.132\columnwidth]{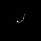}}\hspace{-.23em}
\cfboxR{0.5pt}{black}{\includegraphics[width=.132\columnwidth]{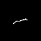}}\hspace{-.23em}
\cfboxR{0.5pt}{black}{\includegraphics[width=.132\columnwidth]{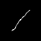}}\hspace{-.23em}
\cfboxR{0.5pt}{black}{\includegraphics[width=.132\columnwidth]{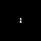}}\hspace{-.23em}
\cfboxR{0.5pt}{black}{\includegraphics[width=.132\columnwidth]{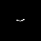}}\hspace{-.23em}
\cfboxR{0.5pt}{black}{\includegraphics[width=.132\columnwidth]{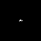}}

\vspace{0.1em}

\cfboxR{0.5pt}{black}{\includegraphics[width=.132\columnwidth]{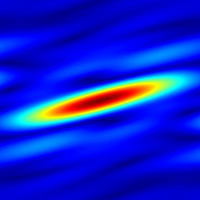}}\hspace{-.23em}
\cfboxR{0.5pt}{black}{\includegraphics[width=.132\columnwidth]{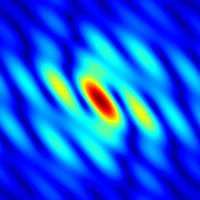}}\hspace{-.23em}
\cfboxR{0.5pt}{black}{\includegraphics[width=.132\columnwidth]{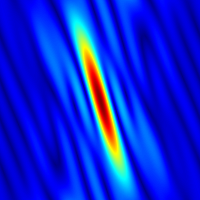}}\hspace{-.23em}
\cfboxR{0.5pt}{black}{\includegraphics[width=.132\columnwidth]{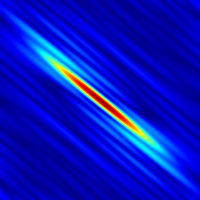}}\hspace{-.23em}
\cfboxR{0.5pt}{black}{\includegraphics[width=.132\columnwidth]{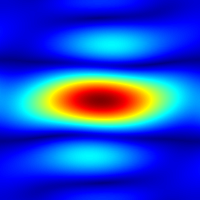}}\hspace{-.23em}
\cfboxR{0.5pt}{black}{\includegraphics[width=.132\columnwidth]{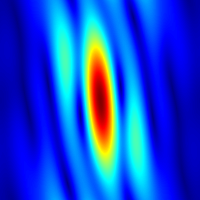}}\hspace{-.23em}
\cfboxR{0.5pt}{black}{\includegraphics[width=.132\columnwidth]{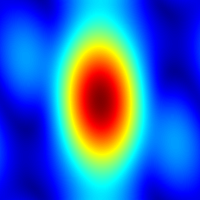}}
\end{minipage}

\vspace{.4em}

\begin{minipage}[c]{0.03\textwidth}
\begin{sideways}\tiny $\,\,\nicefrac{2}{3}$\end{sideways}
\end{minipage}\hspace{-.1em}
\begin{minipage}[c]{0.9\textwidth}
\cfboxR{0.5pt}{black}{\includegraphics[width=.132\columnwidth]{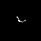}}\hspace{-.23em}
\cfboxR{0.5pt}{black}{\includegraphics[width=.132\columnwidth]{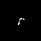}}\hspace{-.23em}
\cfboxR{0.5pt}{black}{\includegraphics[width=.132\columnwidth]{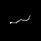}}\hspace{-.23em}
\cfboxR{0.5pt}{black}{\includegraphics[width=.132\columnwidth]{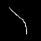}}\hspace{-.23em}
\cfboxR{0.5pt}{black}{\includegraphics[width=.132\columnwidth]{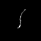}}\hspace{-.23em}
\cfboxR{0.5pt}{black}{\includegraphics[width=.132\columnwidth]{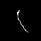}}\hspace{-.23em}
\cfboxR{0.5pt}{black}{\includegraphics[width=.132\columnwidth]{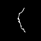}}

\vspace{0.1em}

\cfboxR{0.5pt}{black}{\includegraphics[width=.132\columnwidth]{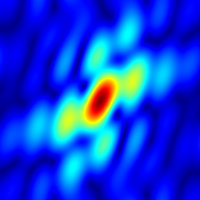}}\hspace{-.23em}
\cfboxR{0.5pt}{black}{\includegraphics[width=.132\columnwidth]{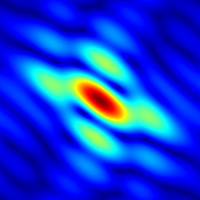}}\hspace{-.23em}
\cfboxR{0.5pt}{black}{\includegraphics[width=.132\columnwidth]{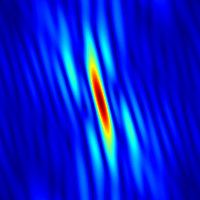}}\hspace{-.23em}
\cfboxR{0.5pt}{black}{\includegraphics[width=.132\columnwidth]{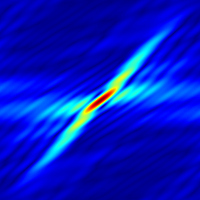}}\hspace{-.23em}
\cfboxR{0.5pt}{black}{\includegraphics[width=.132\columnwidth]{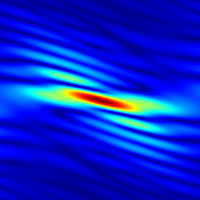}}\hspace{-.23em}
\cfboxR{0.5pt}{black}{\includegraphics[width=.132\columnwidth]{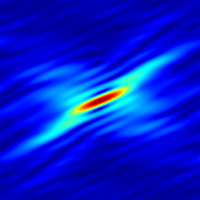}}\hspace{-.23em}
\cfboxR{0.5pt}{black}{\includegraphics[width=.132\columnwidth]{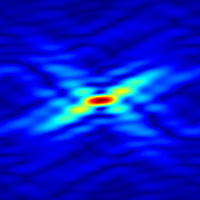}}
\end{minipage}

\vspace{.6em}

\scriptsize (a) Examples of simulated kernels  (exposure $T_\text{exp}$)

\vspace{1.5em}

\begin{minipage}[c]{0.18\textwidth}
\cfboxR{0.5pt}{black}{\includegraphics[width=\textwidth]{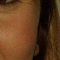}}

\ssmall $s=0.01$ 

\end{minipage}
\begin{minipage}[c]{0.18\textwidth}
\cfboxR{0.5pt}{black}{\includegraphics[width=\textwidth]{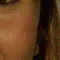}}

\ssmall $s=0.02$ 

\end{minipage}
\begin{minipage}[c]{0.18\textwidth}
\cfboxR{0.5pt}{black}{\includegraphics[width=\textwidth]{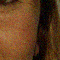}}

\ssmall $s=0.04$ 

\end{minipage}
\begin{minipage}[c]{0.18\textwidth}
\cfboxR{0.5pt}{black}{\includegraphics[width=\textwidth]{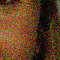}}

\ssmall $s=0.08$ 

\end{minipage}
\begin{minipage}[c]{0.18\textwidth}
\cfboxR{0.5pt}{black}{\includegraphics[width=\textwidth]{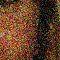}}

\ssmall $s=0.16$ 

\end{minipage}
\vspace{1em}

\scriptsize (b) Examples of noisy simulations  (noise level $s$)

\end{center}

\end{minipage}
\begin{minipage}[c]{0.7\textwidth}
\begin{minipage}[c]{\textwidth}
\centering

\includegraphics[width=0.32\columnwidth]{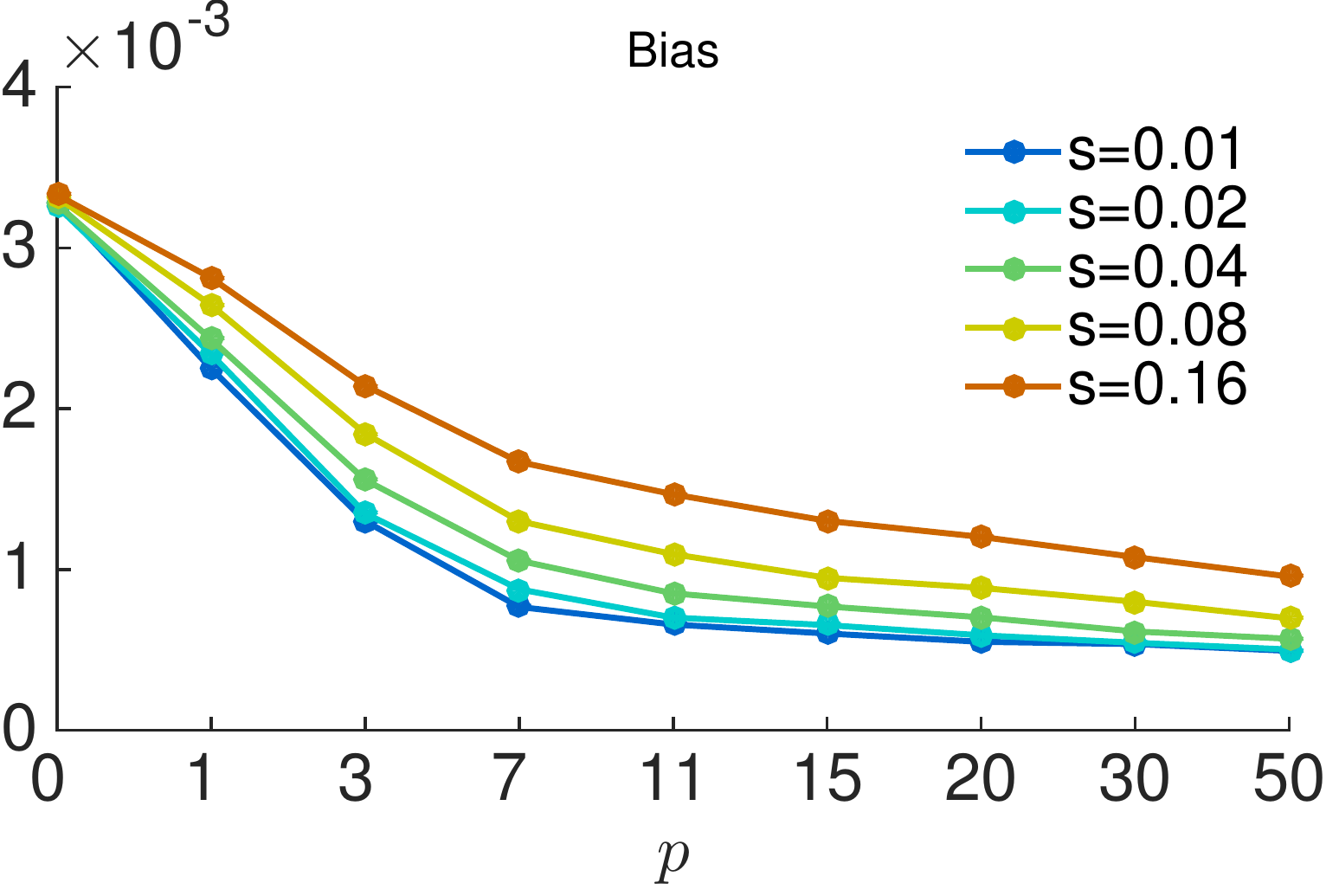}\hspace{.3em} %
\includegraphics[width=0.32\columnwidth]{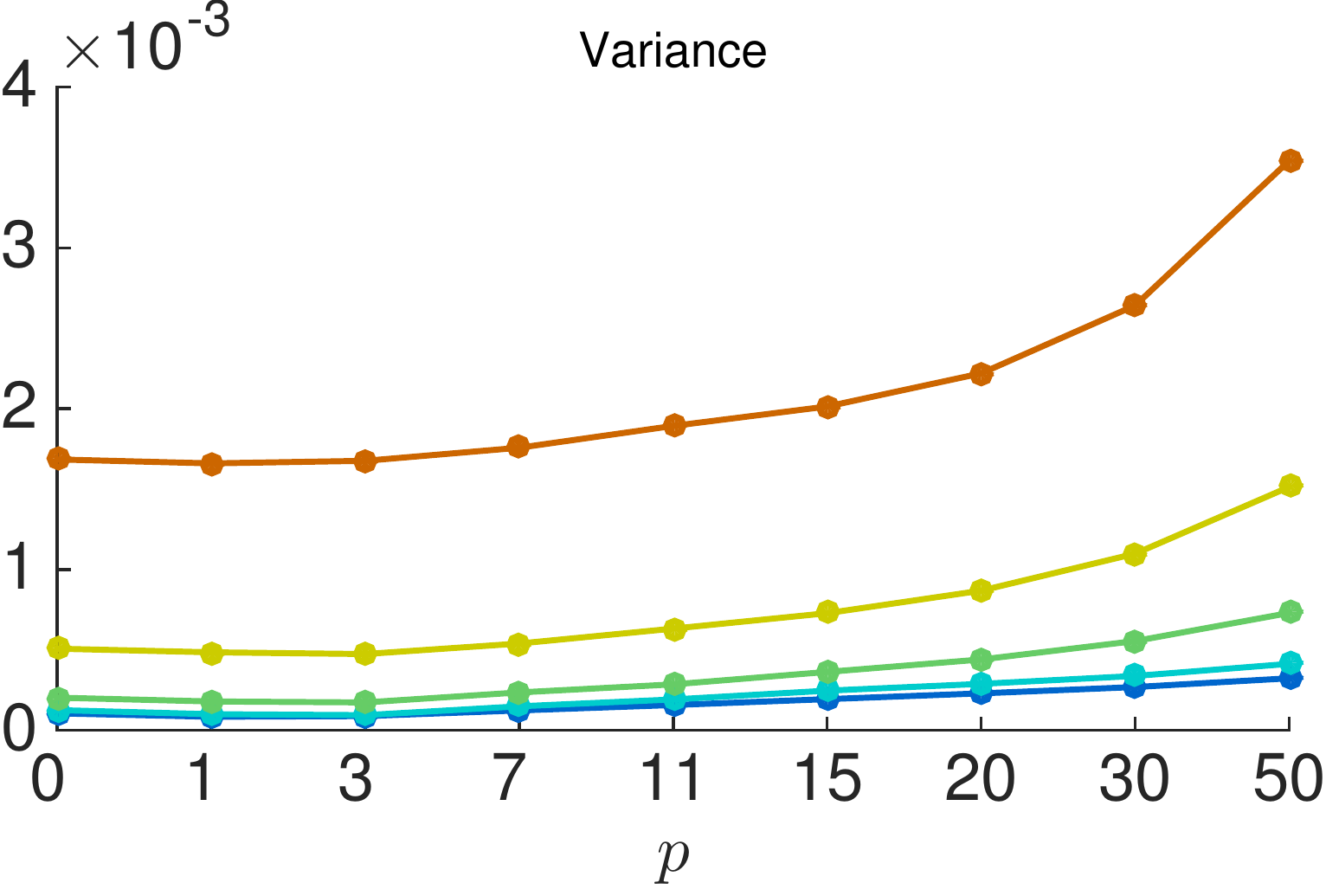}\hspace{.3em}%
\includegraphics[width=0.32\columnwidth]{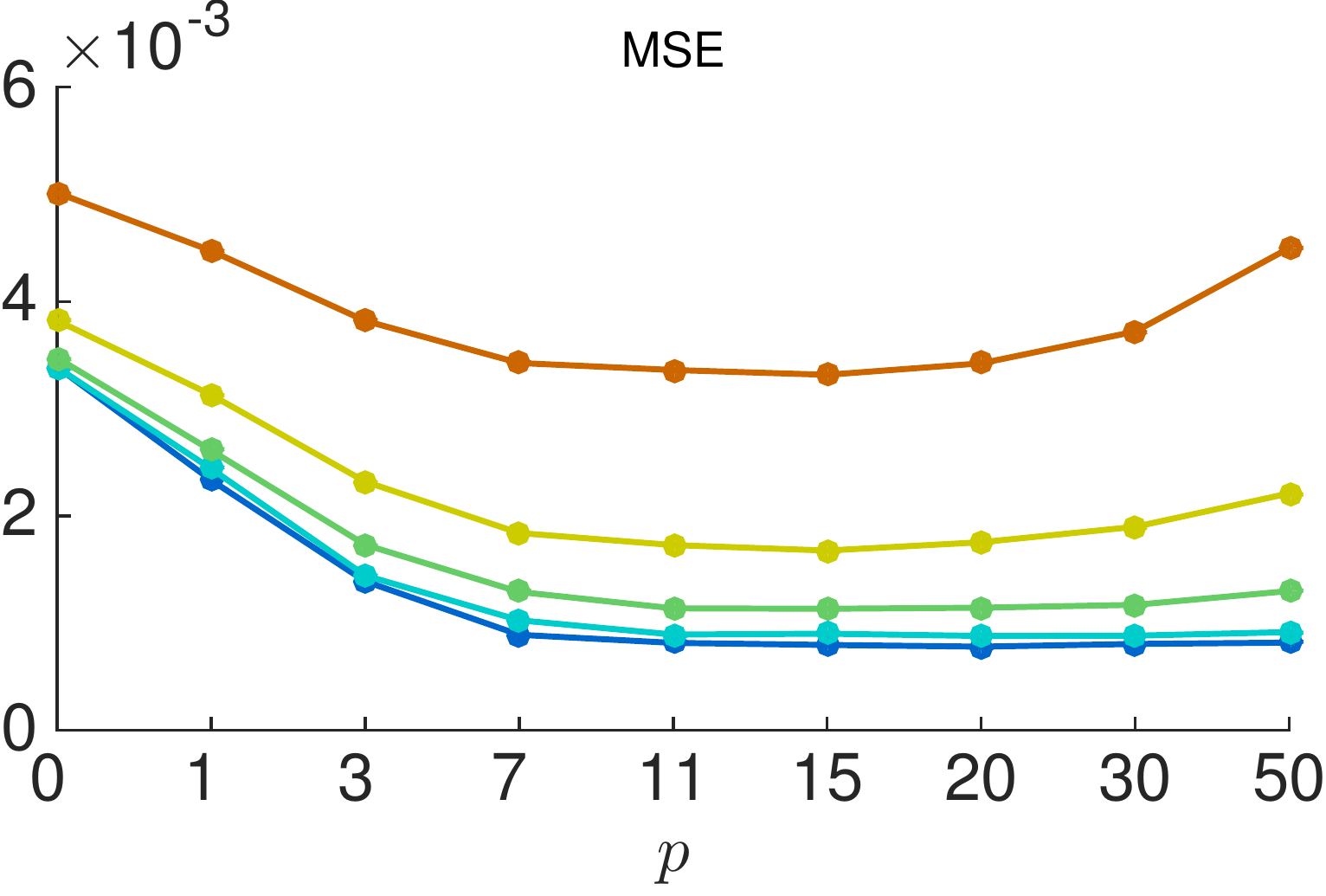}%

\scriptsize (c) Noise level $s$

\end{minipage} \vspace{1.3em}

\begin{minipage}[c]{\textwidth}
\centering

\includegraphics[width=0.32\columnwidth]{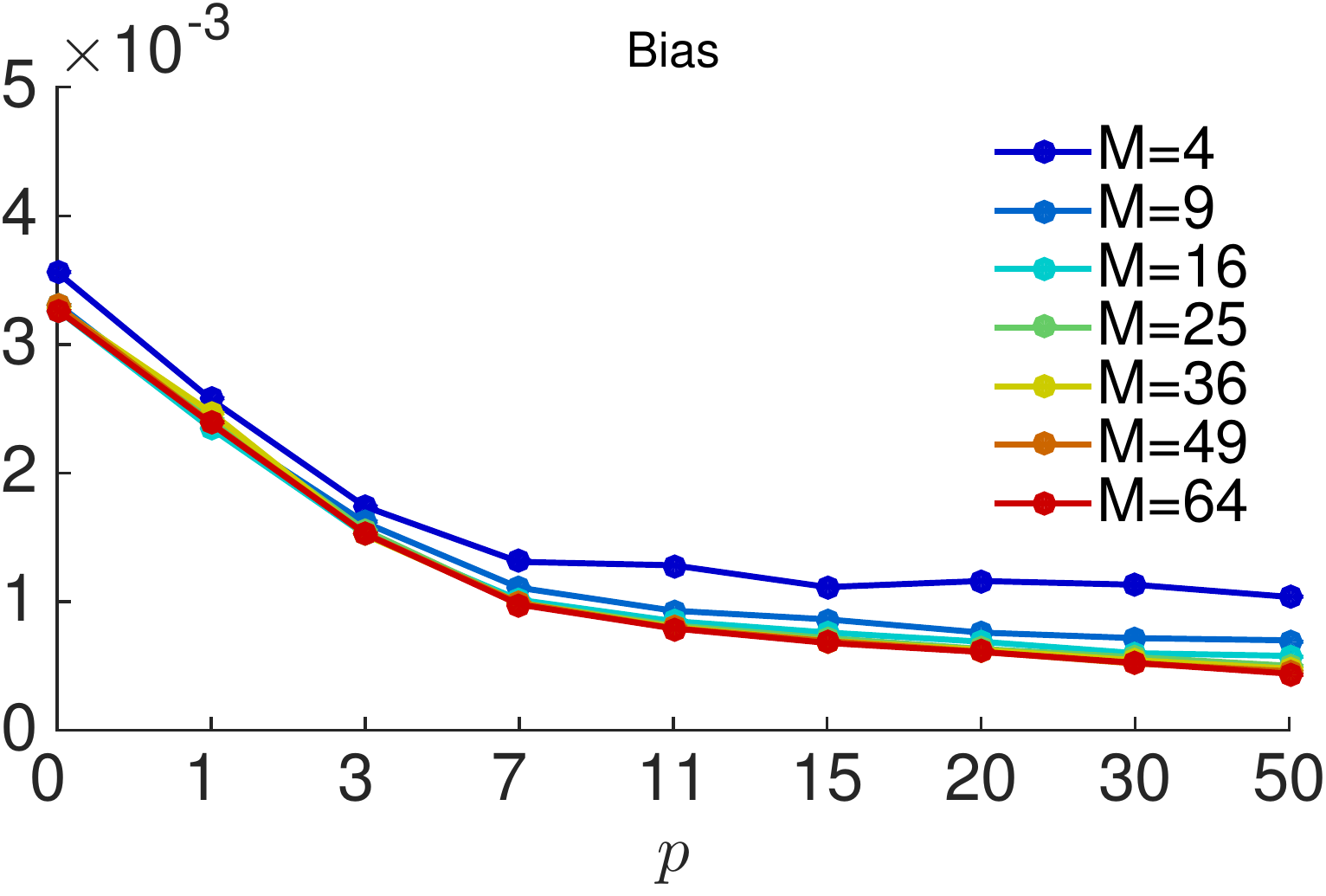}\hspace{.3em}%
\includegraphics[width=0.32\columnwidth]{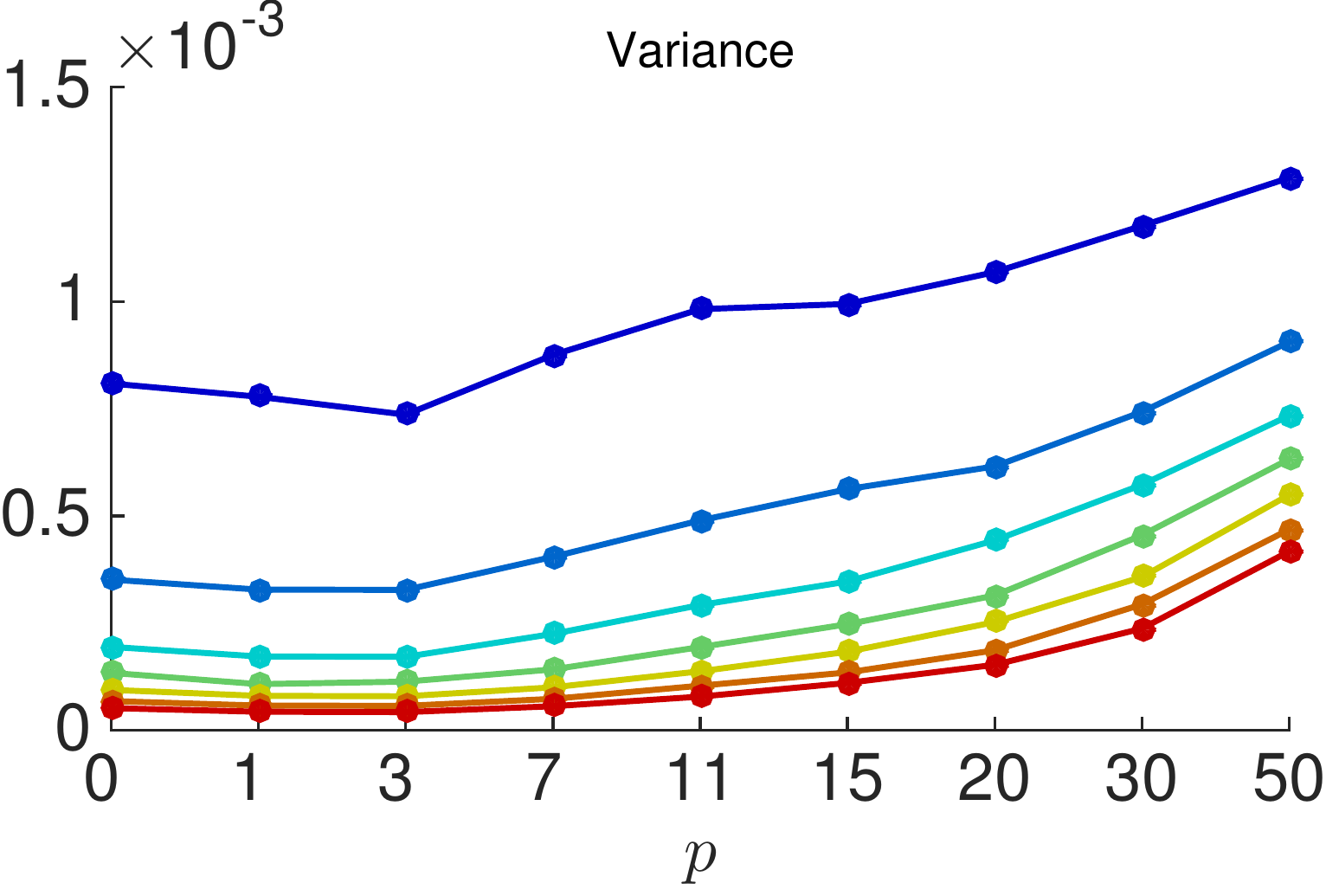}\hspace{.3em}%
\includegraphics[width=0.32\columnwidth]{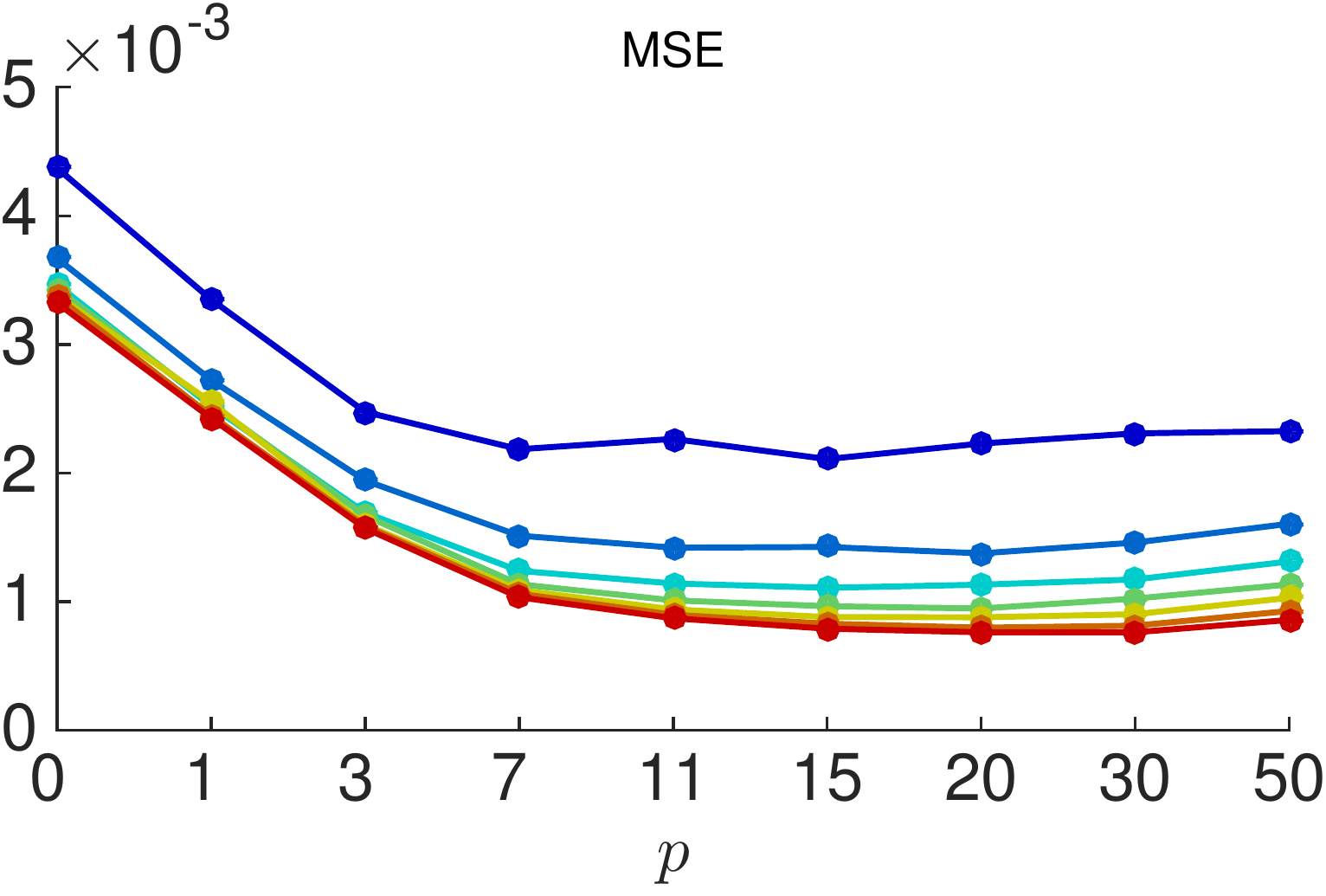}%

\scriptsize (d) Number of images $M$

\end{minipage} \vspace{1.3em}

\begin{minipage}[c]{\textwidth}
\centering

\includegraphics[width=0.32\columnwidth]{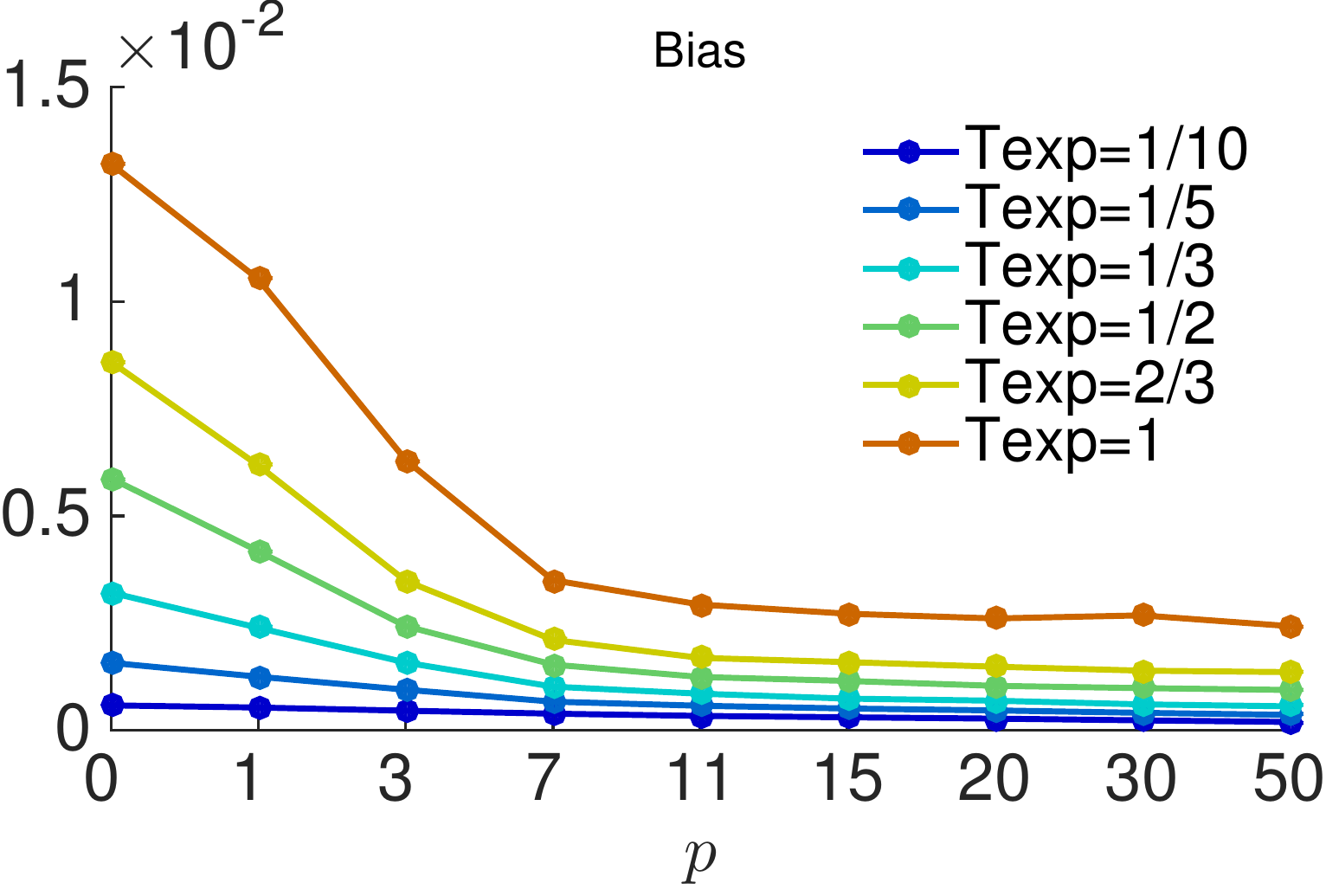}\hspace{.3em}%
\includegraphics[width=0.32\columnwidth]{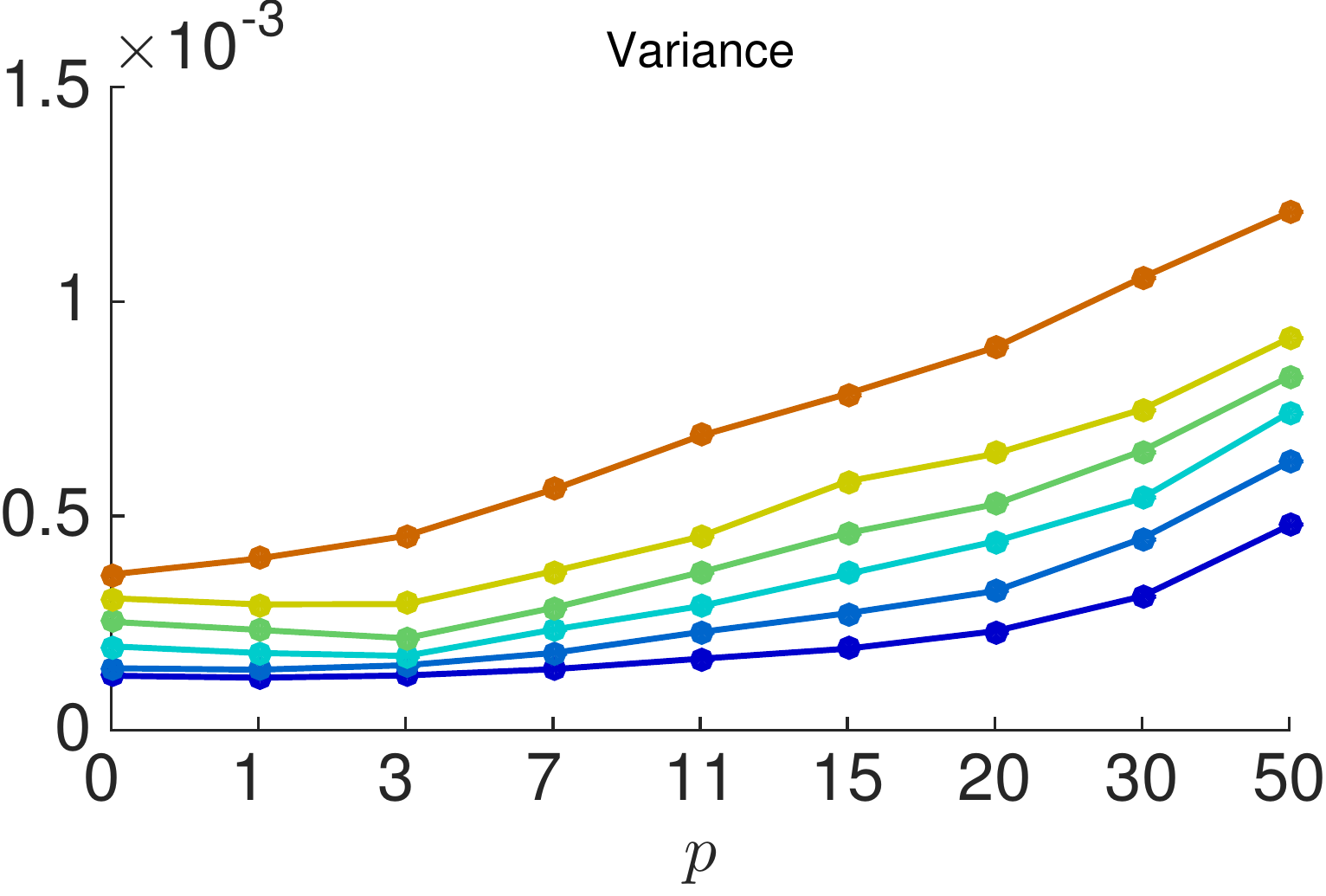}\hspace{.3em}%
\includegraphics[width=0.32\columnwidth]{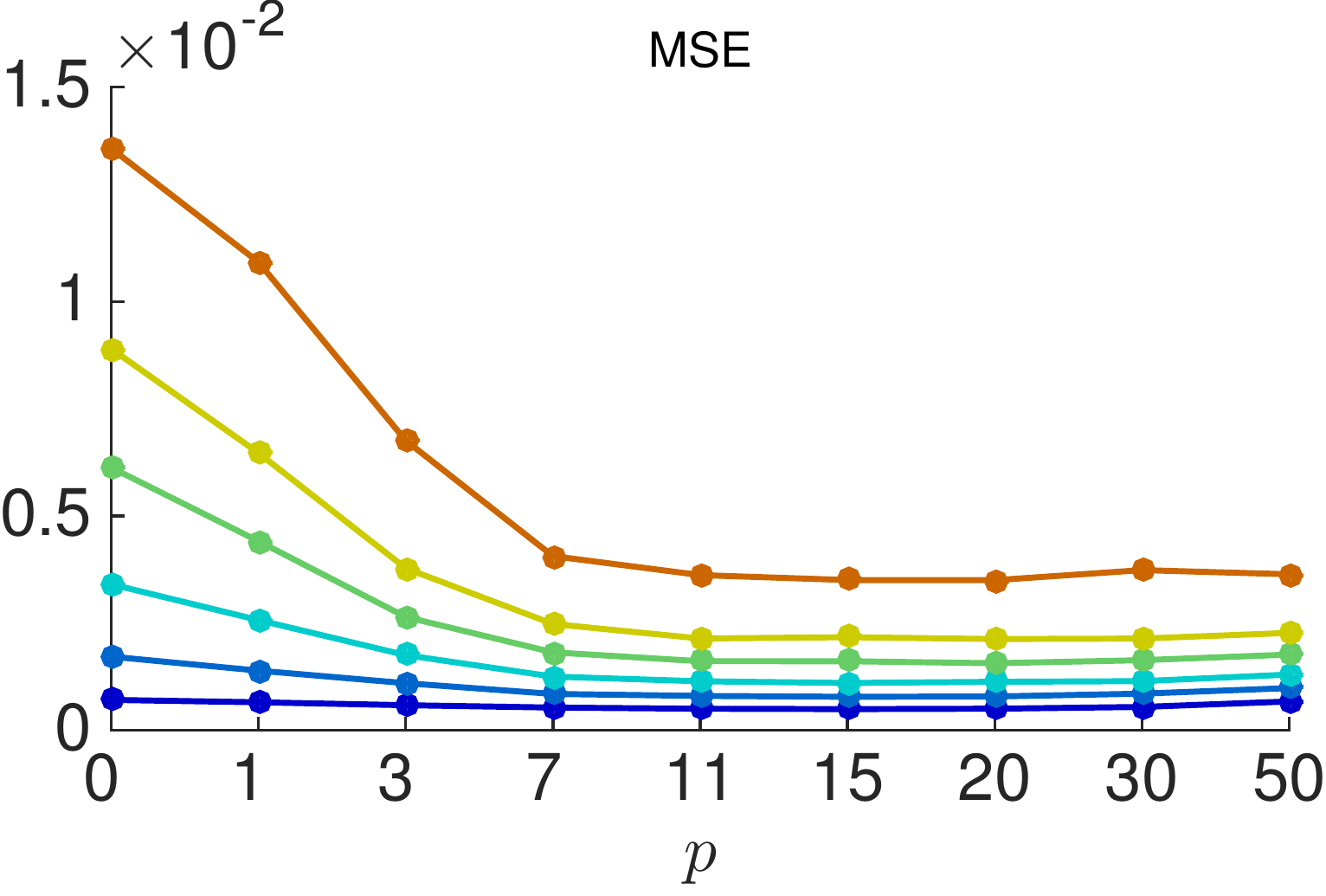}%

\scriptsize (e) Exposure time $T_\text{exp}$

\end{minipage} \vspace{1.3em}

\end{minipage}

\vspace{.5em}

\caption{Bias-Variance tradeoff.
(a) Simulated kernels due to hand tremor following~\cite{gavant2011physiological}.
 Each row shows a set of simulated kernels (left panel) for different exposures $T_\text{exp}=\nicefrac{1}{10},\nicefrac{1}{5},
\nicefrac{1}{3},\nicefrac{1}{2},\nicefrac{2}{3},$ and the respective Fourier spectrum magnitude (bottom panel).
The parameter $T_\text{exp}$ controls the amount of expected blur. (b) Simulated noise levels. 
Average algorithm performance with respect to  $p$  when
 changing (c) the amount of noise $s$ in the input images, (d) the number of images in the burst $M$, and
(e) the exposure time $T_\text{exp}$. The rest of the parameters are set to $M=16$, $s = 0.04$ and 
$T_\text{exp}=\nicefrac{1}{3},$ unless other specified.  With short exposures, the arithmetic 
average ($p=0$) produces the best \textsc{mse} since the images are not blurred. 
The bias does not depend on $M$, but the variance can be significantly reduced by taking
more images (light accumulation procedure). Noise affects the bias and the
 variance terms (with the exception of $p=0$ where the bias is unaffected). The 
\textsc{mse} plots show the existence of a minimum for $p \in [7,30]$,  indicating that the best is to
balance a perfect average and a max pooling. }
\label{fig:biasVariance}
\end{figure*}

\mdG{
\subsection{Impact of Burst Misalignment}
Misalignment of the burst will certainly have an impact on the quality of the  aggregated image.
For the general case where images are noisy but also degraded by anisotropic blur the problem of defining a correct alignment 
is not well defined. 

In what follows we consider that the burst is correctly aligned if each $v_i$ satisfies 
$v_i = u \star k_i + n_i$,
with the blurring kernel $k_i$ having vanishing first moment. That is,
$\int k_i(\bx) \bx d\bx = 0$. 
This constraint on the kernel implies that the kernel does not drift the image $u$, so each $v_i$ is aligned to~$u$ (see Appendix).

To evaluate the impact of misalignment, we considered the particular setting in which the error due to registration is a pure shift.
Although being a simplified case, this helps to understand the general algorithm performance
as a function of the parameter $p$ and the level of misalignment.
In this particular case the translation error can be absorbed in a phase shift of the kernel.
Although the weights will not change, since they only depend on the Fourier magnitude, the average in \eqref{eq:fourierWeightsOrig} 
will be out-of-phase due to the misalignment of the images $v_i$. 
This will  result in blur but also on image artifacts. 

We carried out a similar empirical analysis as before, where several
thousands kernels were simulated and centered by forcing to have vanishing first moment. 
We introduce a Gaussian random \textsc{2d} shift to each blurring kernel (i.e., the first moment of the kernel is shifted)
with zero mean and standard deviation controlled by a parameter $\epsilon$ to simulate the misalignment.  Figure~\ref{fig:missRegistration} shows the algorithm performance when changing the amountof error in the registration (from 0--5 pixels in average) and the value of $p$ controlling the \textsc{fba} aggregaion.
As the alignment error is more significant, the mean square error increases as expected. When the misalignment error is large, the best is to use low $p$ values (close to arithmetic mean).
On the other hand, when the misregistration error is low, large $p$ values will produce the best performance in terms of reconstruction error (\textsc{mse}).
For shift errors in the order of 1~pixel, the $p$ value giving the minimum \textsc{mse} is in $p \in [7,20]$. In general, the bias is always reduced with $p$ while the variance is increased. This is a direct consequence of the fact that smoothness is reduced when increasing the $p$ value (see in Fig.~\ref{fig:weightsTorres} how the energy is concentrated in fewer images as $p\to \infty$, indicating less smoothing). 
}

\begin{figure}[hptb]
\begin{minipage}[c]{0.49\columnwidth}
\centering
{\scriptsize

{\setlength{\tabcolsep}{0.5em}
\begin{tabular}{ c c }
  $\epsilon$ & Avg. error (pixels) \\\hline
  0                     & 0 \\
  \nicefrac{1}{2} & 0.6  \\
  1   & 1.3  \\
  2   & 2.5  \\
  3   & 3.8 \\
  4   & 5.0 
\end{tabular}
}}

\vspace{2em}

\includegraphics[width=.9\textwidth]{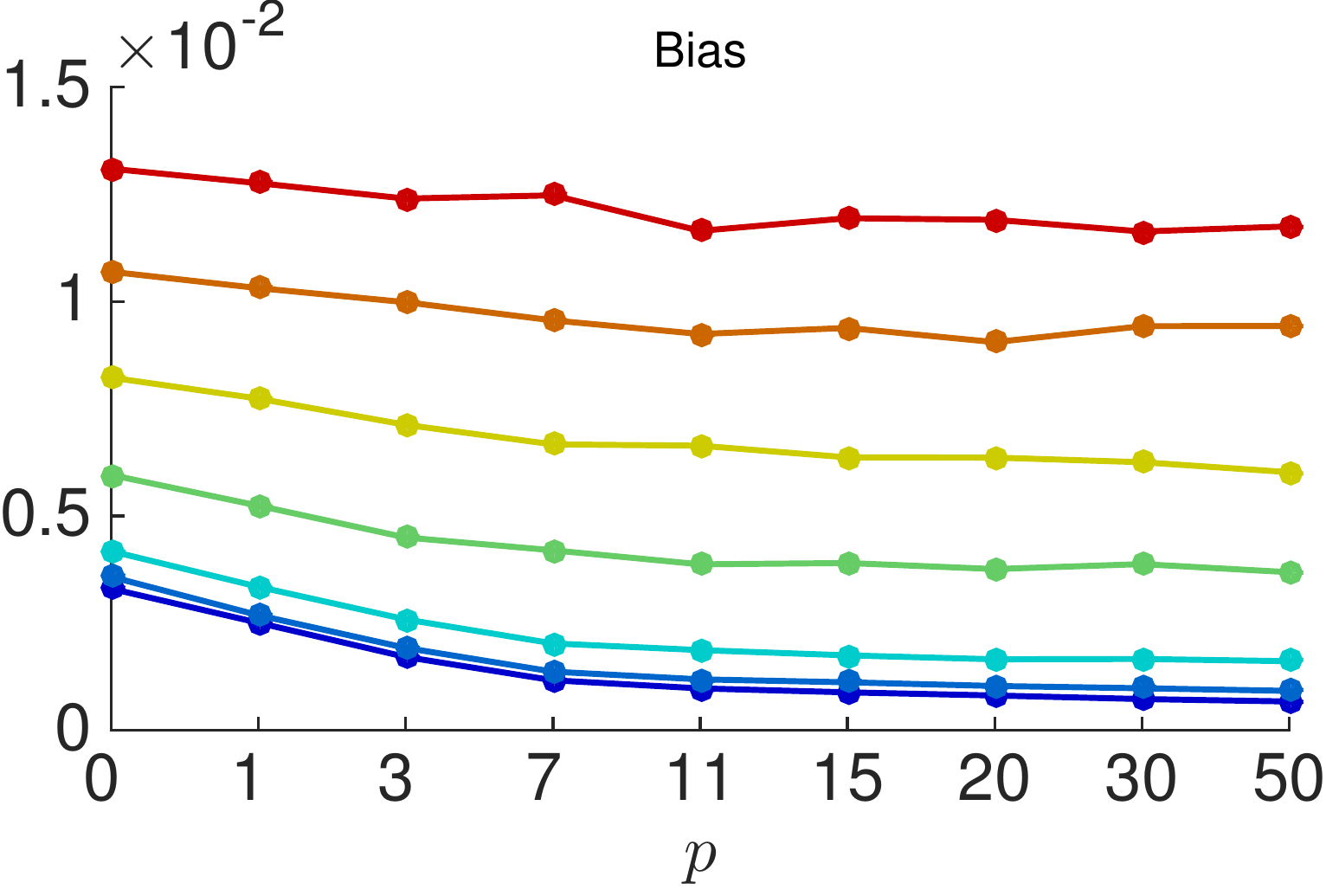}

\end{minipage}
\begin{minipage}[c]{0.49\columnwidth}

\centering
\includegraphics[width=.9\textwidth]{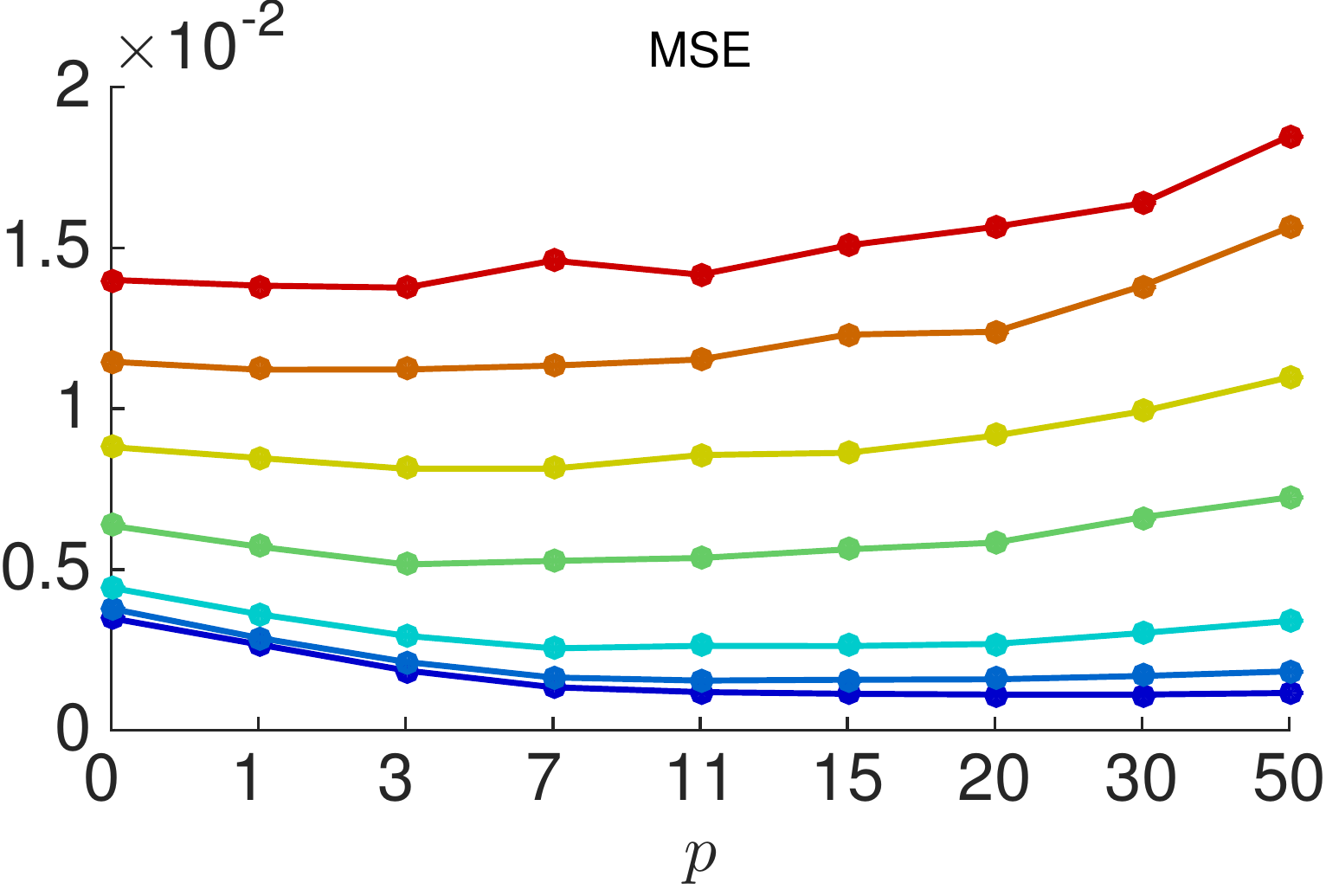}\vspace{.7em}

\includegraphics[width=.9 \textwidth]{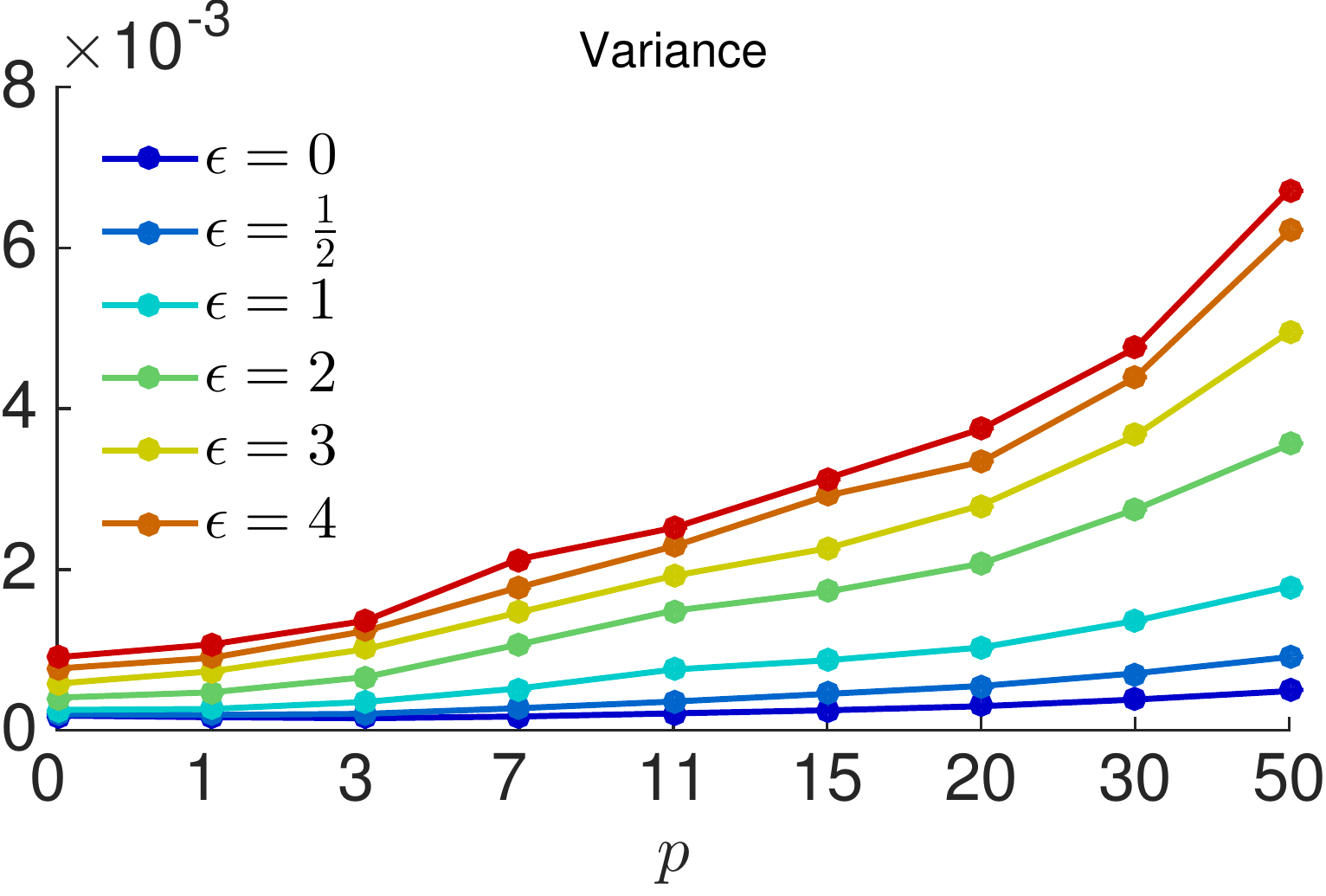}

\scriptsize (e) Registration error $\epsilon$

\end{minipage}

\mdG{\caption{Impact of misalignment.
We simulated shifts following a \textsc{2d} Gaussian distribution of zero mean and standard deviation $\epsilon$. The average displacement is approximately $1.25\epsilon$  (shown in the table).  The top-right plot shows the average algorithm performance (in \textsc{mse}) with respect to  $p$  when
 changing the amount of registration error $\epsilon$.  
 Misalignment deteriorates the algorithm  performance. 
 When the error is large ($\epsilon\ge2$), low $p$ values (aggregation is close to 
 arithmetic  average) produces the best \textsc{mse}, while when the registration
 error is low ($\epsilon\le1$), large $p$ values produce lower reconstruction errors. 
 Using a large $p$ produces results that are less smoothed, so in general, variance is increased while bias is reduced.
}

\label{fig:missRegistration}
}
\end{figure}

\subsection{Comparison to Classical  Lucky Imaging Techniques}

Lucky imaging techniques, very popular when imaging through atmospheric turbulence,
seek to select and average  the sharpest images in a video.
The Fourier weighting scheme can be seen as a generalization
of the lucky imaging family. Lucky imaging selects (or weights more)
frames/regions that are sharper but without paying attention to the 
characteristics of the blur. Thus, when dealing with camera shake, where 
frames are randomly blurred in different directions, classical lucky imaging 
techniques will have a suboptimal performance. In contrast, trying to detect 
the Fourier frequencies that were less affected by the blur and then build an image
with them makes much more sense. This is what the Fourier Burst Accumulation seeks.

\mdG{Recently, Garrel et al.~\cite{garrel2012highly} introduced a selection scheme for astronomic images, based on the relative 
strength of signal for each frequency in the Fourier domain. 
Given a spatial frequency, only the Fourier values having largest magnitude are averaged. The user parameter is the percentage
of largest frames to be averaged (typically ranging from 1\%--10\%).  This method was developed for the particular case of of astronomical
images, where astronomers capture videos having thousands of frames (for example 9000 frames in \cite{garrel2012highly}).
Our algorithm is built on similar ideas, but we do not specify a constant number of frames to be averaged. We let
the Fourier weighting scheme select the total contribution of each frame depending on the relative strength of the Fourier magnitudes.
The authors showed that this generalized lucky imaging procedure produces astronomical images of higher resolution and better signal to noise ratio
than traditional lucky image fusion schemes, further stressing our findings.
}

\begin{figure*}[htpb]
\centering

\begin{minipage}[c]{.0770\textwidth}
\centering
\cfboxR{1.3pt}{red1}{\includegraphics[width=1.00\textwidth]{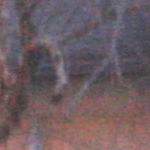}}
\end{minipage}
\begin{minipage}[c]{.0770\textwidth}
\centering
\cfboxR{1.0pt}{red1}{\includegraphics[width=1.00\textwidth]{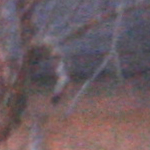}} 
\end{minipage}
\begin{minipage}[c]{.0770\textwidth}
\centering
\cfboxR{1.3pt}{red1}{\includegraphics[width=1.00\textwidth]{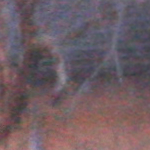}}
\end{minipage}
\begin{minipage}[c]{.0770\textwidth}
\centering
\cfboxR{1.3pt}{red1}{\includegraphics[width=1.00\textwidth]{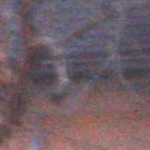}}
\end{minipage}
\begin{minipage}[c]{.0770\textwidth}
\centering
\cfboxR{1.3pt}{red1}{\includegraphics[width=1.00\textwidth]{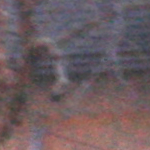}}
\end{minipage}
\begin{minipage}[c]{.0770\textwidth}
\centering
\cfboxR{1.3pt}{red1}{\includegraphics[width=1.00\textwidth]{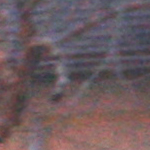}}
\end{minipage}
\begin{minipage}[c]{.0770\textwidth}
\centering
\cfboxR{1.3pt}{red1}{\includegraphics[width=1.00\textwidth]{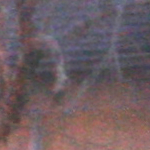}}
\end{minipage}
\begin{minipage}[c]{.0770\textwidth}
\centering
\cfboxR{1.3pt}{red1}{\includegraphics[width=1.00\textwidth]{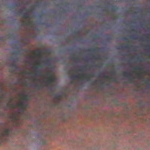}}
\end{minipage}
\begin{minipage}[c]{.0770\textwidth}
\centering
\cfboxR{1.3pt}{red1}{\includegraphics[width=1.00\textwidth]{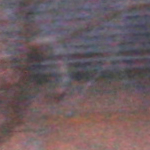}}
\end{minipage}
\begin{minipage}[c]{.0770\textwidth}
\centering
\cfboxR{1.3pt}{red1}{\includegraphics[width=1.00\textwidth]{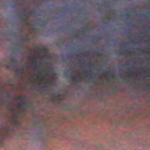}}
\end{minipage}
\begin{minipage}[c]{.0770\textwidth}
\centering
\cfboxR{1.3pt}{red1}{\includegraphics[width=1.00\textwidth]{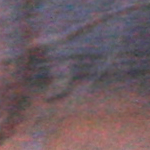}}
\end{minipage}
\begin{minipage}[c]{.0770\textwidth}
\centering
\cfboxR{1.3pt}{red1}{\includegraphics[width=1.00\textwidth]{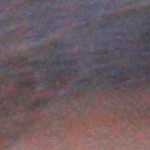}}
\end{minipage}

\vspace{0.2em}

{\ssmall Input frames crops 1--12 ordered from highest (left) to lowest (right) gradient energy}

\vspace{0.4em}

\begin{minipage}[c]{.155\textwidth}
\centering

\begin{tikzpicture}
    \node[anchor=north west,inner sep=0] (image) at (0,0) {\cfboxR{1.0pt}{black}{
               \includegraphics[width=0.99\textwidth]{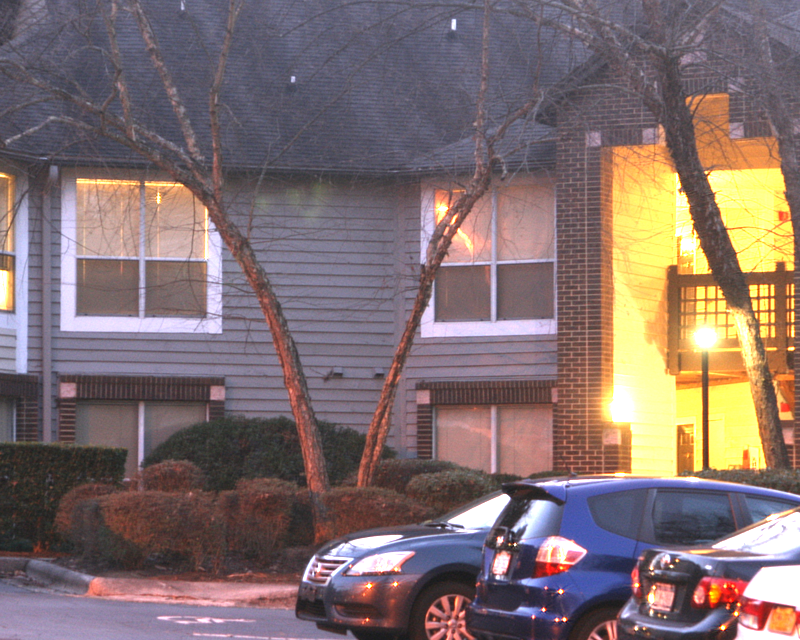}}};
    \begin{scope}[x={(image.north east)},y={(image.south west)}]
        \draw[red1, thick] (0.600,0.2188) rectangle (0.675,0.3125); %
    \end{scope}
\end{tikzpicture}

\vspace{1em}
\end{minipage}
\begin{minipage}[c]{.134\textwidth}
\centering
\cfboxR{1.2pt}{black}{\includegraphics[width=0.995\textwidth]{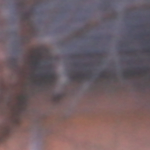}} \vspace{-1.5em}

{\ssmall Arithmetic average}
\end{minipage}
\begin{minipage}[c]{.134\textwidth}
\centering
\cfboxR{1.2pt}{black}{\includegraphics[width=0.995\textwidth]{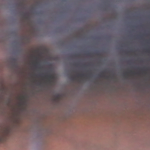}} \vspace{-1.5em}

{\ssmall Sharpness--selectivity~\cite{joshi2010seeing}}
\end{minipage}
\begin{minipage}[c]{.134\textwidth}
\centering
\cfboxR{1.2pt}{black}{\includegraphics[width=0.995\textwidth]{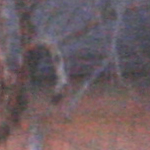}} \vspace{-1.5em}

{\ssmall \textsc{lfa} $1$}\vspace{.1em}
\end{minipage}
\begin{minipage}[c]{.134\textwidth}
\centering
\cfboxR{1.2pt}{black}{\includegraphics[width=0.995\textwidth]{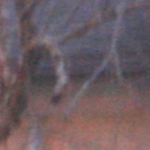}} \vspace{-1.5em}

{\ssmall \textsc{lfa} $2$}\vspace{.1em}
\end{minipage}
\begin{minipage}[c]{.134\textwidth}
\centering
\cfboxR{1.2pt}{black}{\includegraphics[width=0.995\textwidth]{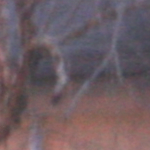}} \vspace{-1.5em}

{\ssmall \textsc{lfa} $3$}\vspace{.1em}
\end{minipage}
\begin{minipage}[c]{.134\textwidth}
\centering
\cfboxR{1.3pt}{blue1}{\includegraphics[width=0.995\textwidth]{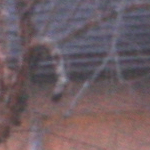}} \vspace{-1.5em}

{\ssmall \textbf{\textsc{fba} $p=11$}}
\end{minipage}

\caption{Comparison to lucky imaging techniques on a  real data burst (\textsf{building}). The arithmetic average produces the best noise reduction but completely removes the details.
The lucky imaging algorithm proposed in~\cite{joshi2010seeing} (sharpness--selectivity) slightly improves the arithmetic average. In this particular experiment
there is blur in different directions and no single frame sharp in all directions. The sharpest frame detected (\textsc{lfa} 1) is still blurred and significantly noisier than the
result given by the Fourier Burst Accumulation with $p=11$ (final sharpening step has been disabled to do a fair comparison). As more sharp frames are averaged
(\textsc{lfa} 2 and \textsc{lfa} 3) noise is reduced at the expense of blurring the final image.}
\label{fig:LuckyParking}

\end{figure*}

Traditional Lucky imaging techniques, are based on evaluating the quality of a given frame.
In astronomy, the most common sharpness measure is the 
intensity of the brightest spot, being a direct measure 
of the concentration of the system's point spread function. 
However, this is not applicable in the context of classical photography. 

Haro et al.~\cite{haro2012photographing} propose to locally estimate the 
level of sharpness using the integral of the energy of 
the image gradient in a surrounding region (i.e., the Dirichlet energy).
If all the images in the series have similar noise level, the Dirichlet energy 
provides a direct way of ordering the images according to  sharpness.  
(i.e., large Dirichlet energy implies sharpness).
 
Let $v_i$ $i=1,\ldots,M$, be a series of registered images of the same scene,
the per-pixel Dirichlet energy weights are~\cite{haro2012photographing}
\begin{align}
w^i_\text{dirichlet} (\bx) = \int_{\Omega_\bx} | \nabla v_i (\bx) |^2 d\bx,
\label{eq:weightsGloria}
\end{align}
where $\Omega_\bx$ is a block of $(100\times100)$ pixels around the pixel $\bx$.
In practice, the Dirichlet weights vary poorly with camera shake blur, so although blurry images 
will contribute less to the final image, their contribution will be still significant.

Joshi and Cohen~\cite{joshi2010seeing} propose to use a combination of sharpness and selectivity per-pixel weights
to determine the contribution of each pixel to the restored image. 
The  sharpness weight is built from the local intensity of the image Laplacian and it is
pondered by a local selectivity term. The selectivity term enforces more noise reduction
in smooth/flat areas. The final sharpness--selectivity weights are\footnote{In the present analysis we did not consider
the terms due to the resampling error nor the sensor dust that were originally included in the
formulation presented in~\cite{joshi2010seeing}.}
\begin{align}
w^i_\text{sharp-sel}(\bx) =   w^i_\text{tex} (\bx) ^ {\gamma (\bx)}
\label{eq:joshiWeights}
\end{align}
where $w^i_\text{tex}(\bx) = \frac{ | \Delta v_i (\bx)| }{ \max_\bx  | \Delta v_i (\bx)| }$ is the local sharpness measure and
$\gamma (\bx) = \lambda \nicefrac{ | \Delta \bar{v} (\bx)| }{ \max_\bx  | \Delta \bar{v} (\bx)| }$ is the selectivity term controlled
by a parameter $\lambda$.  We have denoted by $\bar{v}$ the average of all input frames.

In all cases, the final image is computed as a per-pixel weighted average of the input images,
$$
v_\text{lucky} (\bx) = \frac{\sum_{i=1}^M w^i  (\bx) \cdot v_i (\bx)}{ \sum_{i=1}^n w^i(\bx)}.
$$

Figure~\ref{fig:LuckyParking} shows a comparison of these two approaches to the proposed Fourier Burst Accumulation.
We did not include a final sharpening step to faithfully compare all the approaches, as this last step could be included
in all of them (see Section~\ref{sec:algodetails}).
Since the weighted averaged image using~\eqref{eq:weightsGloria} did not show any difference with respect to the arithmetic average, we
instead used the total Dirichlet energy to rank all the input images and then average only the top $K$  (the sharpest ones).
We named this method Lucky frame average (\textsc{lfa}) and tested different values of $K=1, 2, 3$.

The weights given by~\eqref{eq:joshiWeights} lead to an over-smoothed image with significant less noise. 
The sharpest frame detected (\textsc{lfa}, $K=1$) is still blurred and noisier
 than the result given by the Fourier Burst Accumulation with $p=11$. 
As more lucky frames are averaged (\textsc{lfa} 2 and \textsc{lfa} 3), more noise is eliminated at the expense of introducing blur in the final image.

\section{Algorithm Implementation}
\label{sec:algodetails}

The proposed burst restoration algorithm is built on three main blocks:
Burst Registration, Fourier Burst Accumulation, and Noise Aware Sharpening as a post-processing.
These are described in what follows. \\

\noindent \textbf{Burst Registration.}
There are several ways of registering images (see \cite{zitova2003image} for a survey). In this work, we use image 
correspondences to estimate the dominant homography relating every image of the burst and a reference 
image (the first one in the burst). The homography assumption is valid if
the scene is planar (or far from the camera) or the viewpoint location is fixed, e.g., the camera only rotates around its optical center.
Image correspondences are found using SIFT features~\cite{Lowe2004}  and then filtered out through
the \textsc{orsa} algorithm~\cite{moisan2012ipol}, a variant of the so called \textsc{ransac} method~\cite{fischler1981random}.
To mitigate the effect of the camera shake blur we only detect SIFT features having a larger scale than $\sigma_\text{min} = 1.8$.

Recall that as in prior art, e.g.,~\cite{park2014gyro}, the registration can be done with the gyroscope and accelerometer information from the camera.\\

\noindent \textbf{Fourier Burst Accumulation.}
Given the registered images $\{v_i\}_{i=1}^M$ we directly compute the corresponding Fourier transforms  $\{\hat{v}_i\}_{i=1}^M$.
Since camera shake motion kernels have a small spatial support, their Fourier spectrum magnitudes vary very smoothly. 
Thus,  $|\hat{v}_i|$ can be lowpass filtered before computing the weights, that is,
$
|\bar{\hat{v}}_i| =  G_\sigma  |\hat{v}_i|,
$
where $G_\sigma$  is a Gaussian filter of standard deviation $\sigma$. 
The strength of the low pass filter (controlled by the parameter $\sigma$)  should depend
on the assumed motion kernel size (the smaller the kernel the more regular its Fourier spectrum magnitude). 
In our implementation we set $\sigma =  \nicefrac{\text{min}(m_h,m_w)}{k_s}$, where $k_s=50$~pixels 
and the image size is $m_h \times m_w$ pixels.  Although this low pass filter is important, the results are not too sensitive 
to the exact value of $\sigma$ \mdG{(see Figure~\ref{fig:weightsmoothing}).}

The final Fourier burst aggregation is (note that the smoothing is only applied to the weights calculation)
\begin{align}
u_p =  \mathcal{F}^{-1} \left ( \sum_{i=1}^M w_i \cdot \hat{v}_i \right),  \qquad  w_i   = \frac{ |\bar{ \hat{v}}_i |^p }{\sum_{j=1}^M | \bar{\hat{v}}_j |^p}.
\label{eq:fourierWeights}
\end{align}

The extension to color images is straightforward. The accumulation is done channel by channel using
the same Fourier weights for all channels. The weights are computed by arithmetically averaging the Fourier magnitude of the channels 
before the low pass filtering.\\

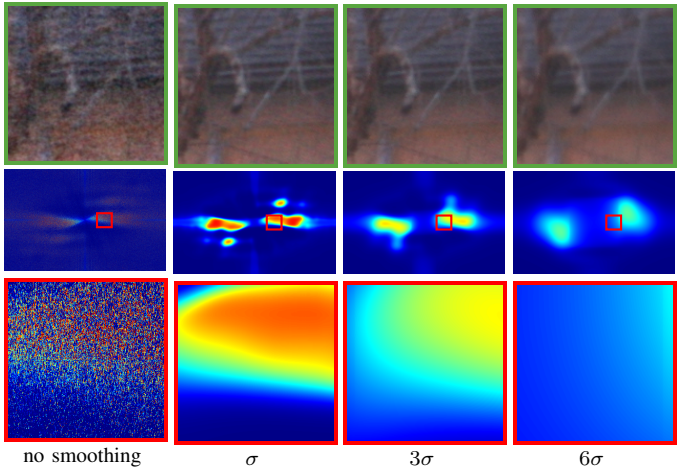
\begin{figure}[hptb]
\centering

\begin{minipage}[c]{.235\columnwidth}
\begin{center}

\cfboxR{1.3pt}{red1}{\includegraphics[width=1.00\textwidth]{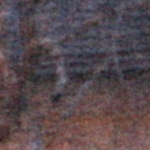}} \vspace{-0.9em}

\begin{tikzpicture}
    \node[anchor=north west,inner sep=0] (image) at (0,0){
              \includegraphics[width=1.03\textwidth]{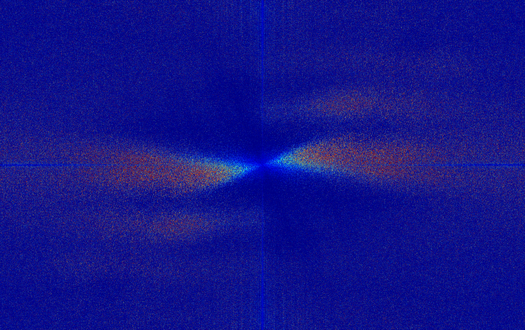}};
    \begin{scope}[x={(image.north east)},y={(image.south west)}]
        \draw[red, thick] (0.5758,0.4318) rectangle (0.6667,0.5682); %
    \end{scope}
\end{tikzpicture} \vspace{-0.9em}

\cfboxR{1.3pt}{red}{\includegraphics[width=1.00\textwidth, height=1.00\textwidth]{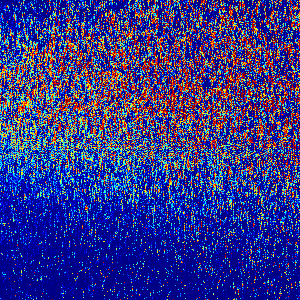}}

\footnotesize no smoothing
\end{center}
\end{minipage}\hspace{.15em}
\begin{minipage}[c]{.235\columnwidth}
\centering
\cfboxR{1.3pt}{red1}{\includegraphics[width=1.00\textwidth]{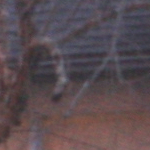}} \vspace{-0.9em}

\begin{tikzpicture}
    \node[anchor=north west,inner sep=0] (image) at (0,0){
              \includegraphics[width=1.03\textwidth]{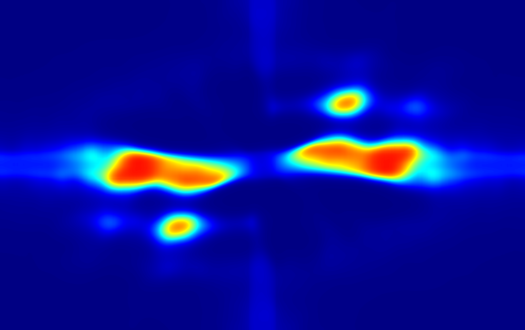}};
    \begin{scope}[x={(image.north east)},y={(image.south west)}]
        \draw[red, thick] (0.5758,0.4318) rectangle (0.6667,0.5682); %
    \end{scope}
\end{tikzpicture} \vspace{-0.9em}

\cfboxR{1.3pt}{red}{\includegraphics[width=1.00\textwidth, height=1.00\textwidth]{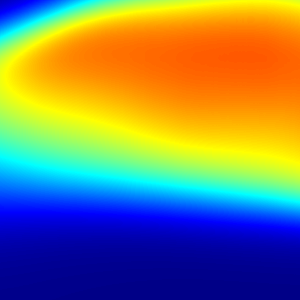}}

\footnotesize $\sigma$
\end{minipage}\hspace{.15em}
\begin{minipage}[c]{.235\columnwidth}
\centering
\cfboxR{1.3pt}{red1}{\includegraphics[width=1.00\textwidth]{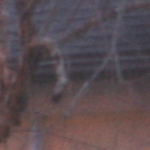}} \vspace{-0.9em}

\begin{tikzpicture}
    \node[anchor=north west,inner sep=0] (image) at (0,0){
              \includegraphics[width=1.03\textwidth]{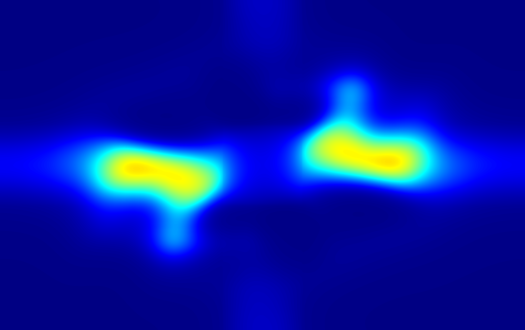}};
    \begin{scope}[x={(image.north east)},y={(image.south west)}]
        \draw[red, thick] (0.5758,0.4318) rectangle (0.6667,0.5682); %
    \end{scope}
\end{tikzpicture} \vspace{-0.9em}

\cfboxR{1.3pt}{red}{\includegraphics[width=1.00\textwidth, height=1.00\textwidth]{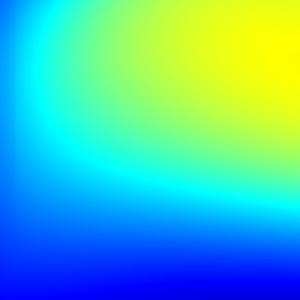}}

\footnotesize $3\sigma$
\end{minipage}\hspace{.15em}
\begin{minipage}[c]{.235\columnwidth}
\centering
\cfboxR{1.3pt}{red1}{\includegraphics[width=1.00\textwidth]{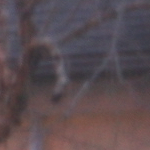}} \vspace{-0.9em}

\begin{tikzpicture}
    \node[anchor=north west,inner sep=0] (image) at (0,0){
              \includegraphics[width=1.03\textwidth]{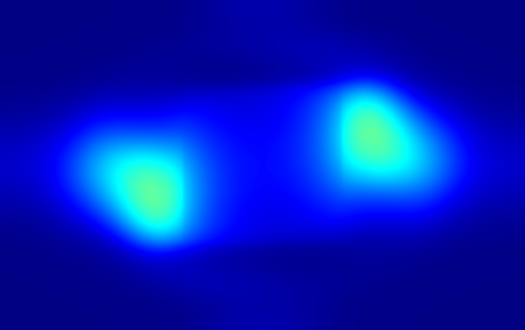}};
    \begin{scope}[x={(image.north east)},y={(image.south west)}]
        \draw[red, thick] (0.5758,0.4318) rectangle (0.6667,0.5682); %
    \end{scope}
\end{tikzpicture} \vspace{-0.9em}

\cfboxR{1.3pt}{red}{\includegraphics[width=1.00\textwidth, height=1.00\textwidth]{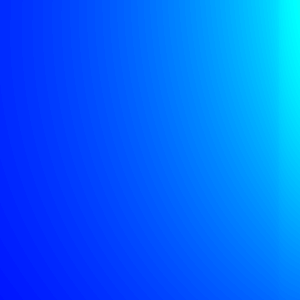}}

\footnotesize $6\sigma$

\end{minipage}

\mdG{\caption{Impact of smoothing the Fourier weights.
To eliminate image artifacts and noise,  $|\hat{v}_i|$ are smoothed before computing the weights $w_i$.
Top row shows the results of the \textsc{fba} average (for the burst shown in Figure~\ref{fig:LuckyParking}), 
middle row the Fourier weights, and the bottom row
a crop of the Fourier weights, when considering different levels of Gaussian smoothing (no smoothing, $\sigma$, $3\sigma$, $6\sigma$).
The strength of the low pass filter is controlled by the parameter 
$\sigma =  \nicefrac{\text{min}(m_h,m_w)}{k_s}$, where $k_s=50$~pixels  and the image size is $m_h \times m_w$ pixels.
As shown in the left column (no smoothing), this filtering step is very important. It provides stabilization to the Fourier weights and also helps to remove noise.
The results are stable for a large range of smoothing levels.}
\label{fig:weightsmoothing}
}
\end{figure}

\noindent \textbf{Noise Aware Sharpening.}
While the results of the Fourier burst accumulation are already very good, 
considering that the process so far has been  computationally non-intensive, one
can optionally apply a final sharpening step if resources are still available.
The sharpening must contemplate that the reconstructed image may have some remaining noise. Thus, we first
apply a denoising algorithm (we used the \textsc{nlbayes} algorithm~\cite{nlbayes2013ipol}\footnote{A variant of this is already available on camera phones, so we stay at the level of potential on-board implementations.}), then on  the filtered image we apply a Gaussian sharpening. To avoid removing fine details
we finally add back a percentage of what has been removed during the denoising step. \vspace{.6em}

The complete method is detailed in Algorithm~\ref{algo:pAggregation}.\\

\begin{algorithm}

\caption{Aggregation of  Blurred Images}
\label{algo:pAggregation}

\Input{A series of images $\tilde{v}_1,\tilde{v}_2,\ldots,\tilde{v}_n$ of size $m\times n \times c^\ast$. An integer value $p$.} \vspace{.5em}
\Output{The aggregated image $u_p$}
\BlankLine

$w = \text{zeros}(m,n)$; $\hat{u}_p = \text{zeros}(m,n,c)$\;
  
\BlankLine
              
\For{image $i$ in $\{1,\ldots,n\}$} 
  {
 \BlankLine
  \Comment{Burst Registration}                 
  $M_i = \textsc{sift}(\tilde{v}_i,\tilde{v}_1)$ \Comment*{\footnotesize  $M_i$ set of corresponding points.}
  $H_i = \textsc{orsa}(M_i)$ \Comment*{\footnotesize $H_i$ dominant homography in $M_i$.}
  $v_i = \tilde{v}_i \circ H_i$ \Comment*{\footnotesize Image resampling.}
\BlankLine

 \Comment{Fourier Burst Accumulation}
 $\hat{v}_i = \textsc{fft}(v_i)$\;
 $w_i = \frac{1}{c}\sum_{j=1}^c |\hat{v}^j_i|$\Comment*{\footnotesize Mean over color channels}
 $w_i = G_\sigma w_i$ \Comment*{\footnotesize Gaussian smoothing}

 $\hat{u}_p =  \hat{u}_p + w^p_i \cdot \hat{v}_i$\Comment*{\footnotesize  Weighted Fourier accumulation}
 $w = w + w^p_i$\;

}

$u_p =  \textsc{ifft} ( \hat{u}_p / w)$\;

\BlankLine

\Comment{Noise Aware Sharpening (Optional)}
$\bar{u}_p = \textsc{denoise}(u_p)$\;
$\bar{u}^s_p =   2\bar{u}_p - G_\rho \bar{u}_p$\Comment*{\footnotesize Gaussian sharpening,  $\rho \in [1,3]$} 
$u_p = \bar{u}^s_p + \delta ( u_p - \bar{u}_p)$\Comment*{\footnotesize Add a fraction of removed noise, $\delta=0.4$} 

\vspace{1em}

\Comment{[*] \footnotesize $c$ is the number of color channels, typically 3. 
The color channels of $v$ are denoted by $v^j$, for $j=1,\ldots,c$.
All the regular operations (e.g, $+, /, \cdot$)  are  point-wise. }
\end{algorithm}

\noindent \textbf{Memory and Complexity Analysis.}
Once the images are registered, the algorithm runs in $O(M \cdot m \cdot \log m)$, where $m = m_h \times m_w$ is the number of image pixels and $M$ the number
of images in the burst. The heaviest part of the algorithm is the computation of $M$ \textsc{fft}s, very suitable and popular in \textsc{vlsi} implementations. This is the reason why the method has a very low complexity.
Regarding memory consumption, the algorithm does not need to access all the images simultaneously and can proceed in an online fashion.
This keeps the memory requirements to only three buffers: one for the current image, one for the
current average, and one for the current weights sum.

\section{Experimental Results}
\label{sec:results}

\begin{figure}[thbp]
\centering
\begin{minipage}[c]{\columnwidth}
\centering

\begin{minipage}[c]{0.32\textwidth}
\centering
{\includegraphics[width=\textwidth]{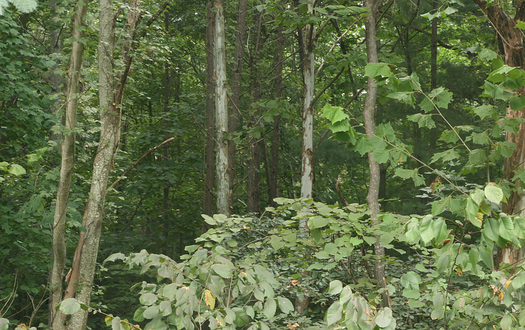}}

\tiny \textsf{woods} 13 imgs \\ \textsc{iso}~1600, $\nicefrac{1}{8}$'' \\ Canon 400D
\end{minipage}
\begin{minipage}[c]{0.32\textwidth}
\centering

{\includegraphics[width=\textwidth]{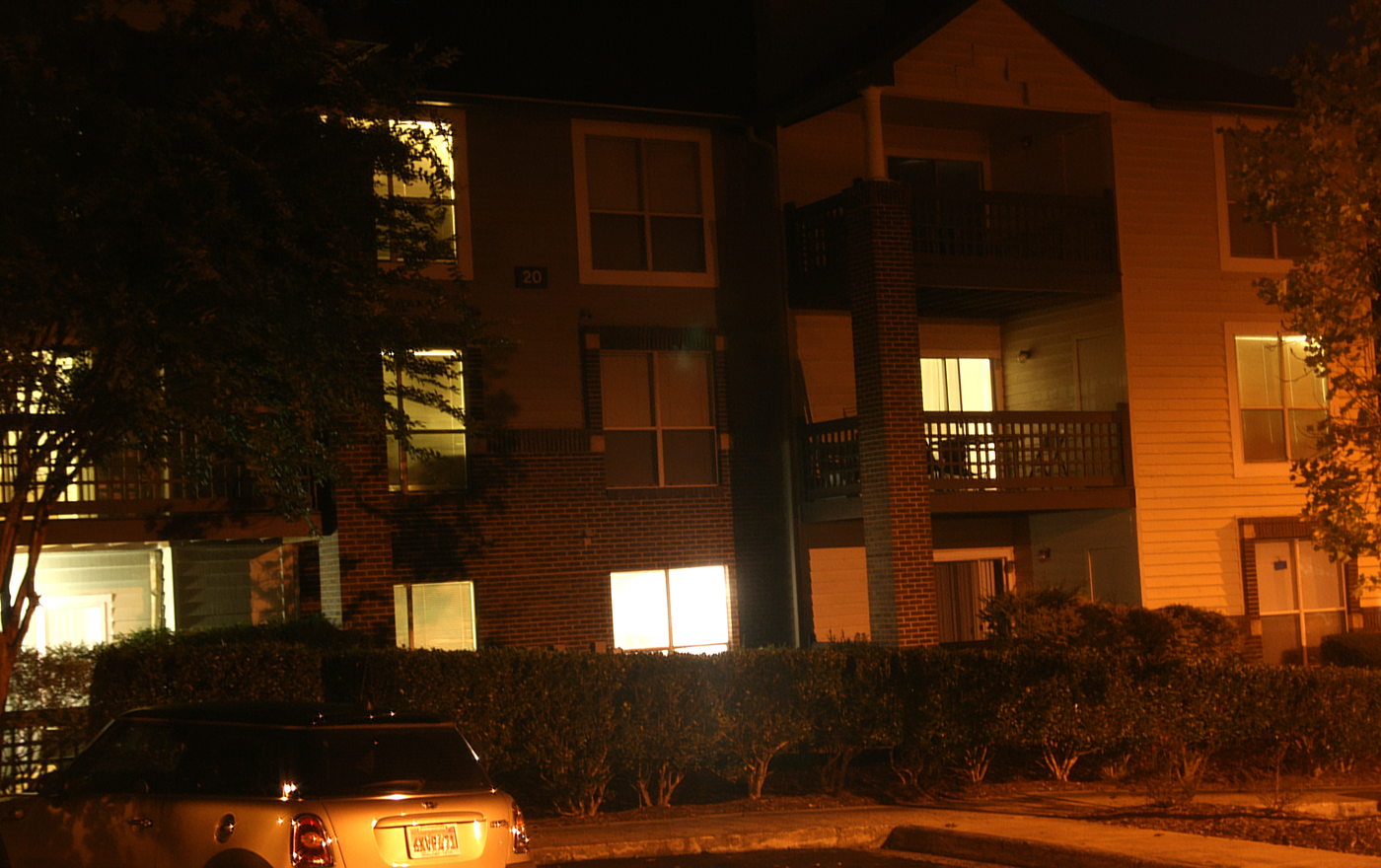}}

\tiny \textsf{parking night} 10 imgs \\ \textsc{iso}~1600, $\nicefrac{1}{3}$'' \\ Canon 400D
\end{minipage}
\begin{minipage}[c]{0.32\textwidth}
\centering

{\includegraphics[width=\textwidth]{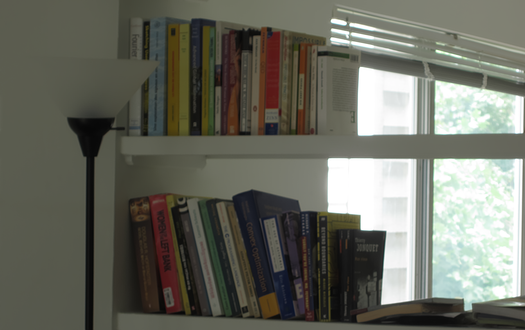}}

\tiny \textsf{bookshelf}  10 imgs \\ \textsc{iso}~100, $\nicefrac{1}{6}$'' \\ Canon 400D
\end{minipage}
\end{minipage}

\vspace{.5em}

\begin{minipage}[c]{\columnwidth}
\centering

\begin{minipage}[c]{0.185\textwidth}
\centering

{\includegraphics[width=\textwidth]{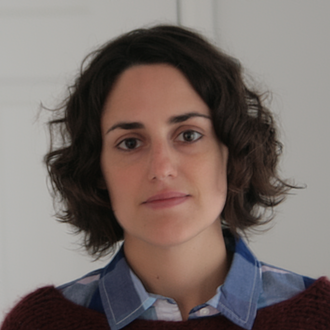}}

\mdG{\tiny \textsf{portrait}  10 imgs \\ \textsc{iso}~800, $\nicefrac{1}{8}$'' \\\vspace{-.05em} Canon 400D}
\end{minipage}
\begin{minipage}[c]{0.25\textwidth}
\centering

{\includegraphics[width=\textwidth]{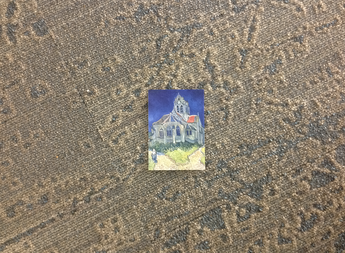}}

\tiny \textsf{auvers} 12 imgs \\ \textsc{iso} 400, $\nicefrac{1}{2}$'', iPad \vspace{1em}
\end{minipage}
\begin{minipage}[c]{0.25\textwidth}
\centering

{\includegraphics[width=\textwidth]{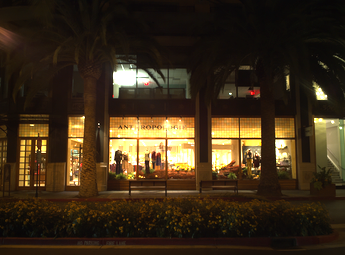}}

\tiny \textsf{anthropologie}~\cite{park2014gyro} 8 imgs \\ \textsc{iso}~100,  353~ms \vspace{1em}
\end{minipage}
\begin{minipage}[c]{0.25\textwidth}
\centering

{\includegraphics[width=\textwidth]{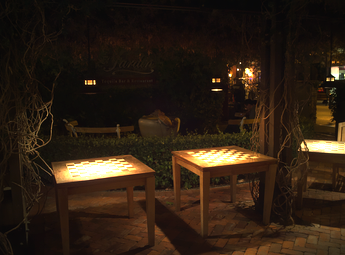}}

\tiny \textsf{tequila}~\cite{park2014gyro} 8 imgs \\ \textsc{iso}~100,  177~ms \vspace{1em}
\end{minipage}
\end{minipage}

\caption{Restoration of image bursts captured with two different cameras. The number of frames, the \textsc{iso} sensitivity and the exposure time is indicated for each burst. Note that in the case of the iPad tablet, the exposure time is the one indicated by the app., and there is  no guarantee that this is the real exposure time.  Full images are available at the project's website.}
\label{fig:burstsFull}

\end{figure}

We captured several handheld bursts with different number of images using a Canon 400D \textsc{dslr} camera and the back camera of an iPad tablet.
The full restored images and the details of the camera parameters are shown in Figure~\ref{fig:burstsFull} and Figure~\ref{fig:hdrResults}.
The photographs contain complex structure, texture, noise and saturated pixels, and were acquired
under different lighting conditions. %
All the results were computed using $p=11$. The full high resolution images are available at the project's website.\footnote{\url{http://dev.ipol.im/~mdelbra/fba/}}

\subsection{Comparison to Multi-image Blind Deblurring}
Since this problem is typically addressed by multi-image blind deconvolution techniques,
we selected two state-of-the-art algorithms for comparison~\cite{zhang2013multi,sroubek2012robust}.
Both algorithms are built on variational formulations and estimate first  the blurring kernels using all the frames in the burst and then
do a step of multi-image non-blind deconvolution, requiring significant memory for normal operation. 
We used the code provided by the authors.\footnote{%
{\scriptsize 
\url{http://zoi.utia.cas.cz/files/fastMBD.zip}\\
\url{https://drive.google.com/file/d/0BzoBvkfRHe5bUF9jQ1ZsWXRYSkk/}}
}
The algorithms rely on parameters that were manually tuned to get the best possible results.
We also compare to the simple \textsf{align and average} algorithm (which indeed is the particular case $p=0$). 

Figures~\ref{fig:AllComparison2}, \ref{fig:AllComparison3} and \ref{fig:AllComparison1} show some crops of the restored 
images by all the algorithms. In addition, we show two input images for each burst:
the best one in the burst and a typical one in the series.
The proposed algorithm obtains similar or better results than the one by 
Zhang {\it et al.}~\cite{zhang2013multi}, at a significantly lower computational and memory cost. Since this algorithm explicitly seeks to deconvolve the sequence, if the convolution model is not
perfectly valid or there is misalignment, the restored image will have deconvolution artifacts. This is clearly 
observed in the \textsf{bookshelf} sequence where \cite{zhang2013multi} produces a slightly sharper restoration
but having ringing artifacts (see \emph{Jonquet} book). Also, it is hard to read
the word ``Women'' in the spine of the red book. Due to the strong assumed priors, \cite{zhang2013multi} generally leads to 
very sharp images but it may also produce overshooting/ringing in some regions like in the brick wall (\textsf{parking night}).

The proposed method clearly outperforms \cite{sroubek2012robust} in all the sequences. This algorithm
 introduces strong artifacts that degraded the performance in most of the tested bursts. 
Tuning the parameters was not trivial since this algorithm relies on 4 parameters that the authors
have linked to a single one (named $\gamma$).  We swept the parameter $\gamma$ to get the best possible performance.

Our approach is conceptually similar to a regular \textsf{align and average} algorithm, 
but it produces significantly sharper images while keeping the noise reduction power of the average principle.
In some cases with numerous images in the burst (e.g., see the \textsf{parking night} sequence), there might already 
be a relatively sharp image in the burst (lucky imaging). 
Our algorithm does not need to explicitly detect such ``best'' frame, and naturally uses the others to denoise
the frequencies not containing image information but  noise. \vspace{.5em}

\subsection{Execution Time} 
Once the images are registered, the proposed approach runs in only a few seconds in our Matlab experimental code, 
while \cite{zhang2013multi} needs several hours for bursts of 8-10 images. Even if the estimation
of the blurring kernels is done in a cropped version (i.e., $200\times200$ pixels region), the multi-image non-blind deconvolution
step is very slow, taking several hours for  6-8~megapixel images. 

\subsection{Multi-image Non-blind Deconvolution}
Figure~\ref{fig:resultsAllComparisonGyroscope} shows the algorithm results  in two sequences provided in~\cite{park2014gyro}. 
The algorithm proposed in~\cite{park2014gyro}  uses gyroscope information present in new camera devices
to register the burst and also to have an estimation of the local blurring kernels.
Then a more expensive multi-image non-blind deconvolution algorithm is applied to recover the latent sharp image. 
Our algorithm produces  similar results without explicitly solving any inverse problem nor using any information about the motion kernels.

\begin{figure}[hptb]
\centering
\begin{minipage}[c]{.158\textwidth}
\centering
\cfboxR{1.3pt}{black}{\includegraphics[width=\textwidth]{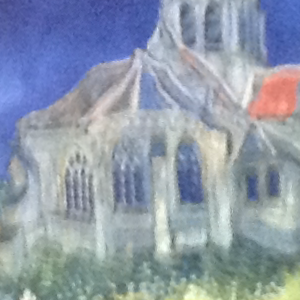}} \vspace{-0.95em}

\cfboxR{1.3pt}{black}{\includegraphics[width=\textwidth]{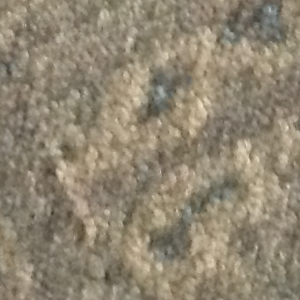}}
\scriptsize Typical Shot \vspace{.9em}
\end{minipage}
\hspace{-0.25em}
\begin{minipage}[c]{.156\textwidth}
\centering
\cfboxR{1.3pt}{black}{\includegraphics[width=\textwidth]{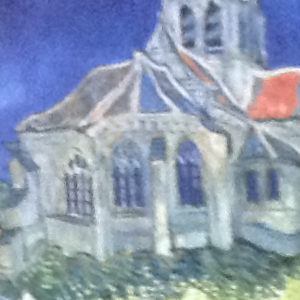}} \vspace{-0.95em}

\cfboxR{1.3pt}{black}{\includegraphics[width=\textwidth]{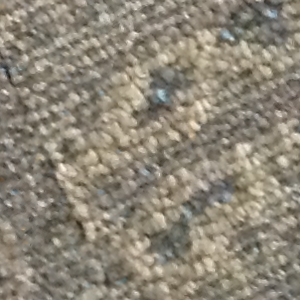}}
\scriptsize Best Shot \vspace{.9em}
\end{minipage}
\hspace{-0.25em}
\begin{minipage}[c]{.156\textwidth}
\centering
\cfboxR{1.3pt}{black}{\includegraphics[width=\textwidth]{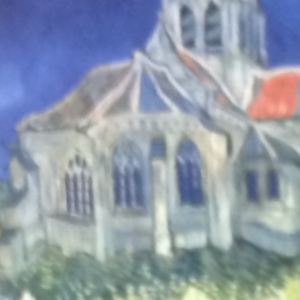}} \vspace{-0.95em}

\cfboxR{1.3pt}{black}{\includegraphics[width=\textwidth]{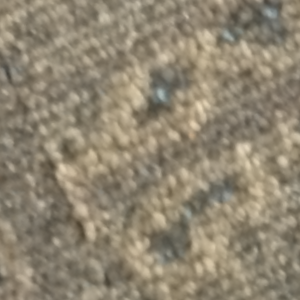}}
\scriptsize Align and average \vspace{.9em}
\end{minipage}

\begin{minipage}[c]{.156\textwidth}
\centering
\fbox{\includegraphics[width=\textwidth]{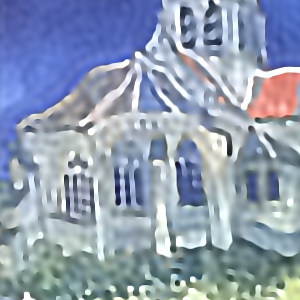}} \vspace{-0.9em}

\fbox{\includegraphics[width=\textwidth]{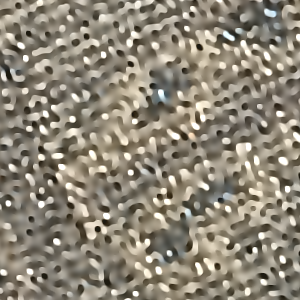}}
\scriptsize \v{S}roubek \& Milanfar~\cite{sroubek2012robust} 
\end{minipage}
\hspace{-0.25em}
\begin{minipage}[c]{.156\textwidth}
\centering
\cfboxR{1.3pt}{black}{\includegraphics[width=\textwidth]{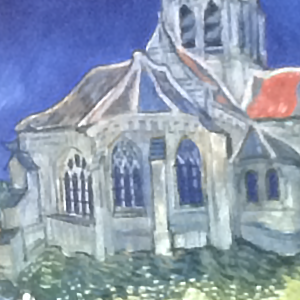}} \vspace{-0.95em}

\cfboxR{1.3pt}{black}{\includegraphics[width=\textwidth]{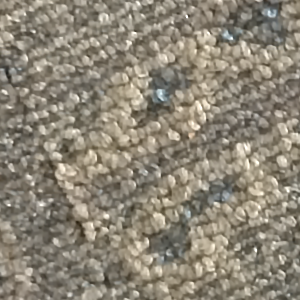}}
\scriptsize  Zhang {\it et al.}~\cite{zhang2013multi} 
\end{minipage}
\hspace{-0.25em}
\begin{minipage}[c]{.156\textwidth} 
\centering
\cfboxR{1.3pt}{blue1}{\includegraphics[width=\textwidth]{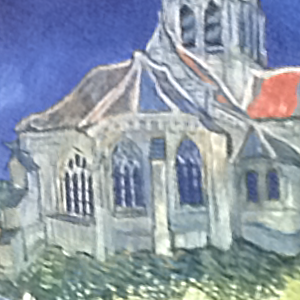}} \vspace{-0.95em}

\cfboxR{1.3pt}{blue1}{\includegraphics[width=\textwidth]{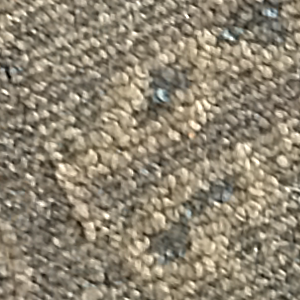}}
\scriptsize \algoname %
\end{minipage}

\caption{Real data burst deblurring results and comparison with multi-image blind deconvolution methods (\textsf{auvers}). }

\label{fig:AllComparison2}
\end{figure}

 \begin{figure}[hptb]
\centering
\begin{minipage}[c]{.158\textwidth}
\centering
\cfboxR{1.3pt}{black}{\includegraphics[width=\textwidth]{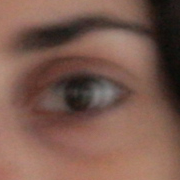}} \vspace{-0.95em}

\cfboxR{1.3pt}{black}{\includegraphics[width=\textwidth]{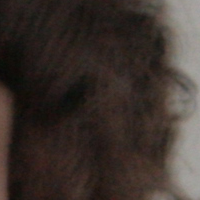}}
\scriptsize Typical Shot \vspace{.9em}
\end{minipage}
\hspace{-0.25em}
\begin{minipage}[c]{.156\textwidth}
\centering
\cfboxR{1.3pt}{black}{\includegraphics[width=\textwidth]{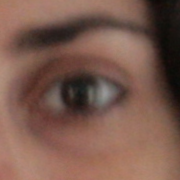}} \vspace{-0.95em}

\cfboxR{1.3pt}{black}{\includegraphics[width=\textwidth]{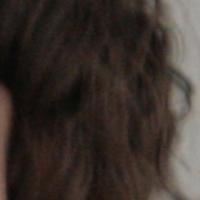}}
\scriptsize Best Shot \vspace{.9em}
\end{minipage}
\hspace{-0.25em}
\begin{minipage}[c]{.156\textwidth}
\centering
\cfboxR{1.3pt}{black}{\includegraphics[width=\textwidth]{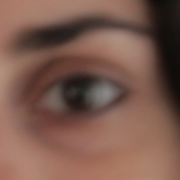}} \vspace{-0.95em}

\cfboxR{1.3pt}{black}{\includegraphics[width=\textwidth]{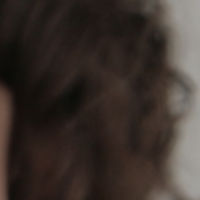}}
\scriptsize Align and average \vspace{.9em}
\end{minipage}

\begin{minipage}[c]{.156\textwidth}
\centering
\fbox{\includegraphics[width=\textwidth]{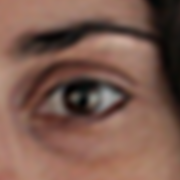}} \vspace{-0.9em}

\fbox{\includegraphics[width=\textwidth]{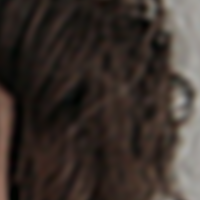}}
\scriptsize \v{S}roubek \& Milanfar~\cite{sroubek2012robust} 
\end{minipage}
\hspace{-0.25em}
\begin{minipage}[c]{.156\textwidth}
\centering
\cfboxR{1.3pt}{black}{\includegraphics[width=\textwidth]{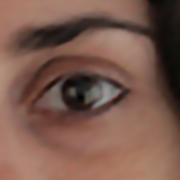}} \vspace{-0.95em}

\cfboxR{1.3pt}{black}{\includegraphics[width=\textwidth]{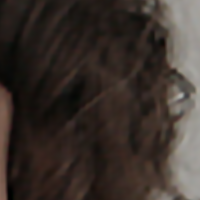}}
\scriptsize  Zhang {\it et al.}~\cite{zhang2013multi} 
\end{minipage}
\hspace{-0.25em}
\begin{minipage}[c]{.156\textwidth} 
\centering
\cfboxR{1.3pt}{blue1}{\includegraphics[width=\textwidth]{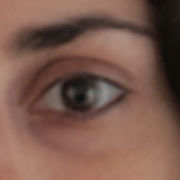}} \vspace{-0.95em}

\cfboxR{1.3pt}{blue1}{\includegraphics[width=\textwidth]{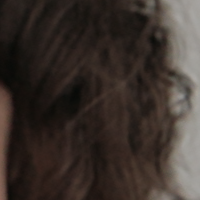}}
\scriptsize \algoname %
\end{minipage}

\mdG{\caption{Real data burst deblurring results and comparison with multi-image blind deconvolution methods (\textsf{portrait}).}
\label{fig:AllComparison3}
}

\end{figure}

\begin{figure*}[hptb]
\centering

\begin{minipage}[c]{.137\textwidth}
\centering
\cfboxR{1.3pt}{black}{\includegraphics[width=\textwidth]{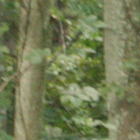}} \vspace{-.95em}

\cfboxR{1.3pt}{black}{\includegraphics[width=\textwidth]{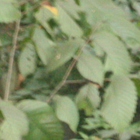}} \vspace{-.95em}

\cfboxR{1.3pt}{black}{\includegraphics[width=\textwidth]{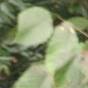}} \vspace{-.95em}

\end{minipage}
\begin{minipage}[c]{.137\textwidth}
\centering
\cfboxR{1.3pt}{black}{\includegraphics[width=\textwidth]{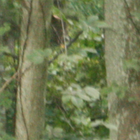}} \vspace{-.95em}

\cfboxR{1.3pt}{black}{\includegraphics[width=\textwidth]{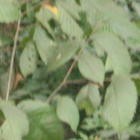}} \vspace{-.95em}

\cfboxR{1.3pt}{black}{\includegraphics[width=\textwidth]{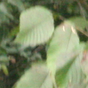}} \vspace{-.95em}

\end{minipage}
\begin{minipage}[c]{.137\textwidth}
\centering
\cfboxR{1.3pt}{black}{\includegraphics[width=\textwidth]{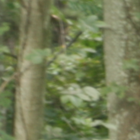}} \vspace{-.95em}

\cfboxR{1.3pt}{black}{\includegraphics[width=\textwidth]{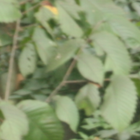}} \vspace{-.95em}

\cfboxR{1.3pt}{black}{\includegraphics[width=\textwidth]{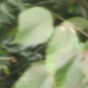}} \vspace{-.95em}

\end{minipage}
\begin{minipage}[c]{.137\textwidth}
\centering
\cfboxR{1.3pt}{black}{\includegraphics[width=\textwidth]{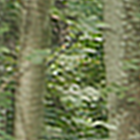}} \vspace{-.95em}

\cfboxR{1.3pt}{black}{\includegraphics[width=\textwidth]{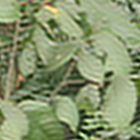}} \vspace{-.95em}

\cfboxR{1.3pt}{black}{\includegraphics[width=\textwidth]{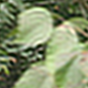}} \vspace{-.95em}

\end{minipage}
\begin{minipage}[c]{.137\textwidth}
\centering
\cfboxR{1.3pt}{black}{\includegraphics[width=\textwidth]{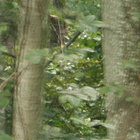}} \vspace{-.95em}

\cfboxR{1.3pt}{black}{\includegraphics[width=\textwidth]{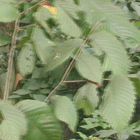}} \vspace{-.95em}

\cfboxR{1.3pt}{black}{\includegraphics[width=\textwidth]{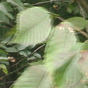}} \vspace{-.95em}

\end{minipage}
\begin{minipage}[c]{.137\textwidth}
\centering
\cfboxR{1.3pt}{black}{\includegraphics[width=\textwidth]{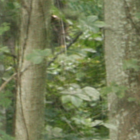}} \vspace{-.95em}

\cfboxR{1.3pt}{black}{\includegraphics[width=\textwidth]{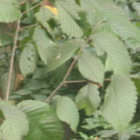}} \vspace{-.95em}

\cfboxR{1.3pt}{black}{\includegraphics[width=\textwidth]{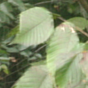}} \vspace{-.95em}

\end{minipage}
\begin{minipage}[c]{.137\textwidth}
\centering
\cfboxR{1.3pt}{blue1}{\includegraphics[width=\textwidth]{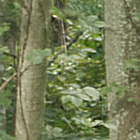}} \vspace{-.95em}

\cfboxR{1.3pt}{blue1}{\includegraphics[width=\textwidth]{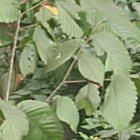}} \vspace{-.95em}

\cfboxR{1.3pt}{black}{\includegraphics[width=\textwidth]{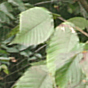}} \vspace{-.95em}

\end{minipage}

\vspace{.3em}

\begin{minipage}[c]{.137\textwidth}
\centering
\cfboxR{1.3pt}{black}{\includegraphics[width=\textwidth]{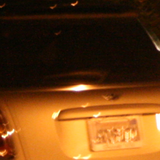}} \vspace{-.95em}

\cfboxR{1.3pt}{black}{\includegraphics[width=\textwidth]{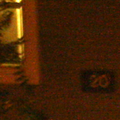}} \vspace{-.95em}

\cfboxR{1.3pt}{black}{\includegraphics[width=\textwidth]{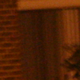}} \vspace{-.95em}

\end{minipage}
\begin{minipage}[c]{.137\textwidth}
\centering
\cfboxR{1.3pt}{black}{\includegraphics[width=\textwidth]{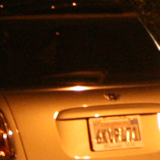}} \vspace{-.95em}

\cfboxR{1.3pt}{black}{\includegraphics[width=\textwidth]{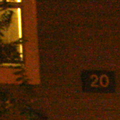}} \vspace{-.95em}

\cfboxR{1.3pt}{black}{\includegraphics[width=\textwidth]{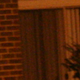}} \vspace{-.95em}
\end{minipage}
\begin{minipage}[c]{.137\textwidth}
\centering
\cfboxR{1.3pt}{black}{\includegraphics[width=\textwidth]{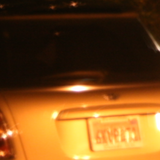}} \vspace{-.95em}

\cfboxR{1.3pt}{black}{\includegraphics[width=\textwidth]{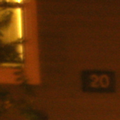}} \vspace{-.95em}

\cfboxR{1.3pt}{black}{\includegraphics[width=\textwidth]{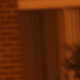}} \vspace{-.95em}
\end{minipage}
\begin{minipage}[c]{.137\textwidth}
\centering
\cfboxR{1.3pt}{black}{\includegraphics[width=\textwidth]{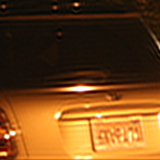}} \vspace{-.95em}

\cfboxR{1.3pt}{black}{\includegraphics[width=\textwidth]{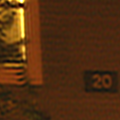}} \vspace{-.95em}

\cfboxR{1.3pt}{black}{\includegraphics[width=\textwidth]{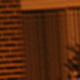}} \vspace{-.95em}
\end{minipage}
\begin{minipage}[c]{.137\textwidth}
\centering
\cfboxR{1.3pt}{black}{\includegraphics[width=\textwidth]{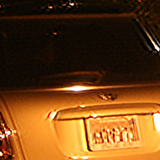}} \vspace{-.95em}

\cfboxR{1.3pt}{black}{\includegraphics[width=\textwidth]{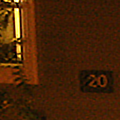}} \vspace{-.95em}

\cfboxR{1.3pt}{black}{\includegraphics[width=\textwidth]{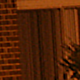}} \vspace{-.95em}
\end{minipage}
\begin{minipage}[c]{.137\textwidth}
\centering
\cfboxR{1.3pt}{black}{\includegraphics[width=\textwidth]{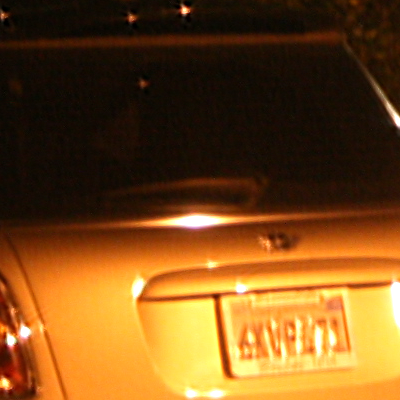}} \vspace{-.95em}

\cfboxR{1.3pt}{black}{\includegraphics[width=\textwidth]{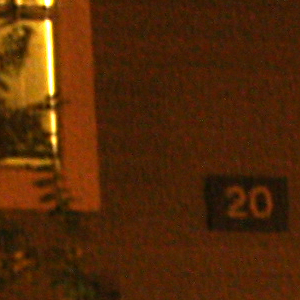}} \vspace{-.95em}

\cfboxR{1.3pt}{black}{\includegraphics[width=\textwidth]{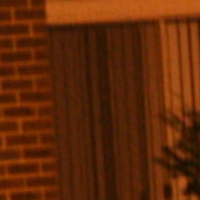}} \vspace{-.95em}

\end{minipage}
\begin{minipage}[c]{.137\textwidth}
\centering
\cfboxR{1.3pt}{blue1}{\includegraphics[width=\textwidth]{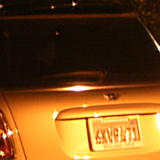}} \vspace{-.95em}

\cfboxR{1.3pt}{blue1}{\includegraphics[width=\textwidth]{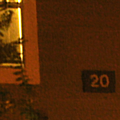}} \vspace{-.95em}

\cfboxR{1.3pt}{blue1}{\includegraphics[width=\textwidth]{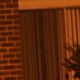}} \vspace{-.95em}
\end{minipage}

\vspace{.3em}

\begin{minipage}[c]{.137\textwidth}
\centering
\cfboxR{1.3pt}{black}{\includegraphics[width=\textwidth]{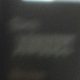}} \vspace{-.95em}

\cfboxR{1.3pt}{black}{\includegraphics[width=\textwidth]{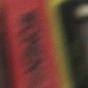}} \vspace{-.95em}

\ssmall Typical Shot \vspace{1.1em}
\end{minipage}
\begin{minipage}[c]{.137\textwidth}
\centering
\cfboxR{1.3pt}{black}{\includegraphics[width=\textwidth]{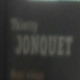}} \vspace{-.95em}

\cfboxR{1.3pt}{black}{\includegraphics[width=\textwidth]{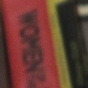}} \vspace{-.95em}

\ssmall Best Shot \vspace{1.2em}

\end{minipage}
\begin{minipage}[c]{.137\textwidth}
\centering
\cfboxR{1.3pt}{black}{\includegraphics[width=\textwidth]{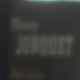}} \vspace{-.95em}

\cfboxR{1.3pt}{black}{\includegraphics[width=\textwidth]{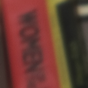}} \vspace{-.95em}

\ssmall Align and average \vspace{1.1em}
\end{minipage}
\begin{minipage}[c]{.137\textwidth}
\centering
\cfboxR{1.3pt}{black}{\includegraphics[width=\textwidth]{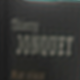}} \vspace{-.95em}

\cfboxR{1.3pt}{black}{\includegraphics[width=\textwidth]{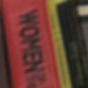}} \vspace{-.95em}

\ssmall \v{S}roubek \& Milanfar~\cite{sroubek2012robust}  \vspace{1.1em}
\end{minipage}
\begin{minipage}[c]{.137\textwidth}
\centering
\cfboxR{1.3pt}{black}{\includegraphics[width=\textwidth]{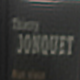}} \vspace{-.95em}

\cfboxR{1.3pt}{black}{\includegraphics[width=\textwidth]{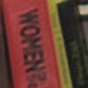}} \vspace{-.95em}

\ssmall Zhang {\it et al.}~\cite{zhang2013multi} \vspace{1.0em}
\end{minipage}
\begin{minipage}[c]{.137\textwidth}
\centering
\cfboxR{1.3pt}{black}{\includegraphics[width=\textwidth]{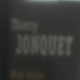}} \vspace{-.95em}

\cfboxR{1.3pt}{black}{\includegraphics[width=\textwidth]{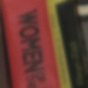}} \vspace{-.95em}

\ssmall Proposed Method\\ (no final sharp.)\vspace{-.1em}
\end{minipage}
\begin{minipage}[c]{.137\textwidth}
\centering
\cfboxR{1.3pt}{blue1}{\includegraphics[width=\textwidth]{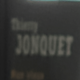}} \vspace{-.95em}

\cfboxR{1.3pt}{blue1}{\includegraphics[width=\textwidth]{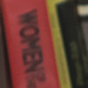}} \vspace{-.95em}

\ssmall \algoname \vspace{1.1em}
\end{minipage}

\caption{Real data burst deblurring results and comparison to state-of-the-art multi-image blind deconvolution algorithms~(\textsf{woods} rows 1-3, \textsf{parking night} rows 4-6, \textsf{bookshelf} rows 7-8,  sequences). }
\label{fig:AllComparison1}

\end{figure*}

\begin{figure}[hptb]
\centering
\begin{minipage}[c]{.32\columnwidth}
\centering
\cfboxR{1.3pt}{black}{\includegraphics[width=\textwidth]{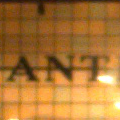}} \vspace{-0.9em}

\cfboxR{1.3pt}{black}{\includegraphics[width=\textwidth]{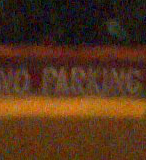}}

\end{minipage}
\hspace{-0.22em}
\begin{minipage}[c]{.32\columnwidth}
\centering
\cfboxR{1.3pt}{black}{\includegraphics[width=\textwidth]{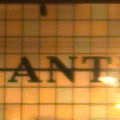}} \vspace{-0.9em}

\cfboxR{1.3pt}{black}{\includegraphics[width=\textwidth]{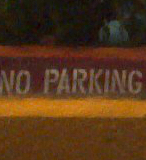}}

\end{minipage}
\hspace{-0.22em}
\begin{minipage}[c]{.32\columnwidth}
\centering
\cfboxR{1.3pt}{blue1}{\includegraphics[width=\textwidth]{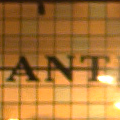}} \vspace{-0.9em}

\cfboxR{1.3pt}{blue1}{\includegraphics[width=\textwidth]{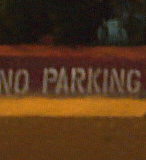}}

\end{minipage}

\vspace{0.2em}

\begin{minipage}[c]{.32\columnwidth}
\centering
\cfboxR{1.3pt}{black}{\includegraphics[width=\textwidth]{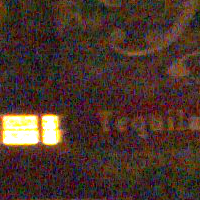}} \vspace{-0.9em}

\cfboxR{1.3pt}{black}{\includegraphics[width=\textwidth]{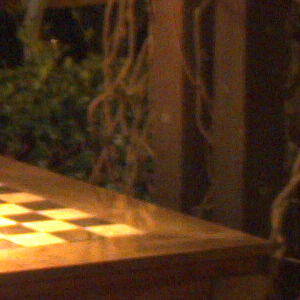}}
\scriptsize Input 1 %
\end{minipage}
\hspace{-0.22em}
\begin{minipage}[c]{.32\columnwidth}
\centering
\cfboxR{1.3pt}{black}{\includegraphics[width=\textwidth]{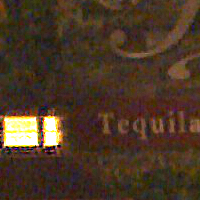}} \vspace{-0.9em}

\cfboxR{1.3pt}{black}{\includegraphics[width=\textwidth]{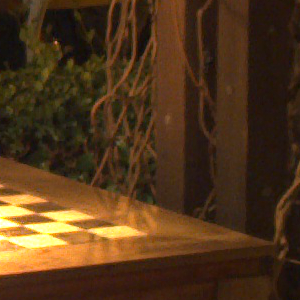}}
\scriptsize Park \& Levoy~\cite{park2014gyro}
\end{minipage}
\hspace{-0.22em}
\begin{minipage}[c]{.32\columnwidth}
\centering
\cfboxR{1.3pt}{blue1}{\includegraphics[width=\textwidth]{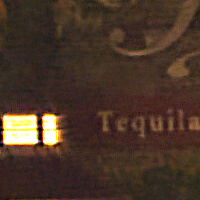}} \vspace{-0.9em}

\cfboxR{1.3pt}{blue1}{\includegraphics[width=\textwidth]{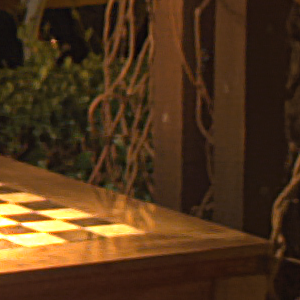}}
\scriptsize \algoname \vspace{0.15em}
\end{minipage}

\caption{Restoration results with the data provided in~\cite{park2014gyro} (\textsf{anthropologie} and \textsf{tequila} sequences).}
\label{fig:resultsAllComparisonGyroscope}

\end{figure}

\subsection{HDR imaging: Multi-Exposure Fusion}
\label{sec:hdr}
In many situations, the dynamic range of the photographed scene is larger than the
one of the camera's image sensor. A popular solution to this problem is to capture
several images with varying exposure settings and  combine them to produce a
single high dynamic range high-quality image (see e.g., \cite{reinhard10,banterle2011book}). 

However, in dim light conditions, large exposure times are needed, leading to
the presence of image blur when the images are captured without a tripod.
This presents an additional challenge.
A direct extension of the present work is to capture several bursts, 
each one covering a different exposure level. Then, each of the bursts
is processed with the \textsc{fba} procedure leading to
a clean sharper representation of each burst. The obtained sharp images
can then be merged to produce a high quality image using any existent 
exposure fusion algorithm. 

Figure~\ref{fig:hdrResults} shows the results of taking two image bursts with 
two different exposure times, 
and separately aggregating them using
the proposed algorithm. We then applied the exposure fusion algorithm
of~\cite{zhang2010gradient} to get a clean tone mapped image from the two burst representations.
The fusion is much sharper and cleaner when using the Fourier Burst Accumulation for combining
each burst than the one given by the arithmetic average or the best frames in each burst.

\begin{figure*}
\begin{center}
\centering

\begin{minipage}[c]{.49\textwidth}
\centering
\ssmall \textsf{building}, $2\times12$ imgs, $\nicefrac{1}{8}$'' and $\nicefrac{1}{40}$'', \textsc{iso} 800, Canon 400D 
\end{minipage}
\begin{minipage}[c]{.49\textwidth}
\centering
\ssmall \textsf{indoors}, $2\times12$ imgs, $\nicefrac{1}{4}$'' and $\nicefrac{1}{80}$'', \textsc{iso} 250, iPad 
\end{minipage}

\vspace{.1em}

\hspace{-1em}
\begin{minipage}[c]{.16\textwidth}
\centering
\cfboxR{1.0pt}{black}{\includegraphics[width=0.99\textwidth]{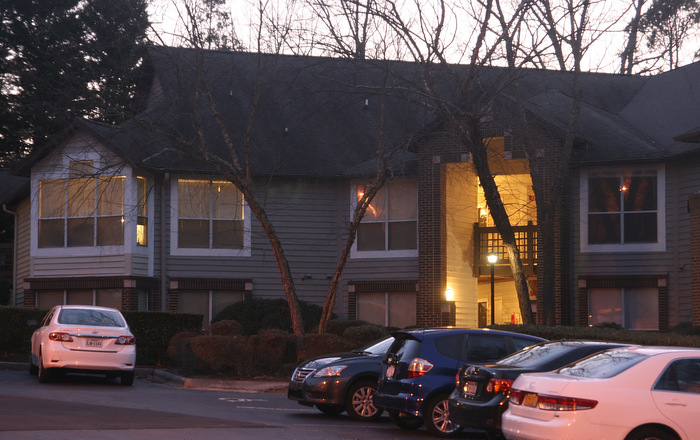}}

\ssmall FBA - long exposure (LE)
\end{minipage}
\begin{minipage}[c]{.16\textwidth}
\centering
\cfboxR{1.0pt}{black}{\includegraphics[width=.99\textwidth]{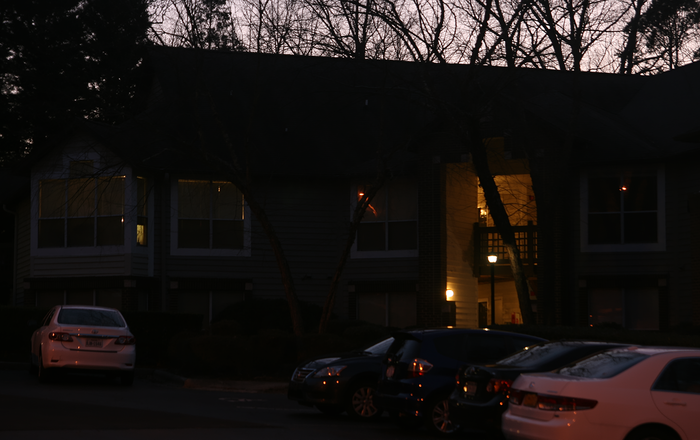}}

\ssmall FBA - short exposure (SE)
\end{minipage}
\begin{minipage}[c]{.16\textwidth}
\centering
\cfboxR{1.0pt}{black}{\includegraphics[width=0.99\textwidth]{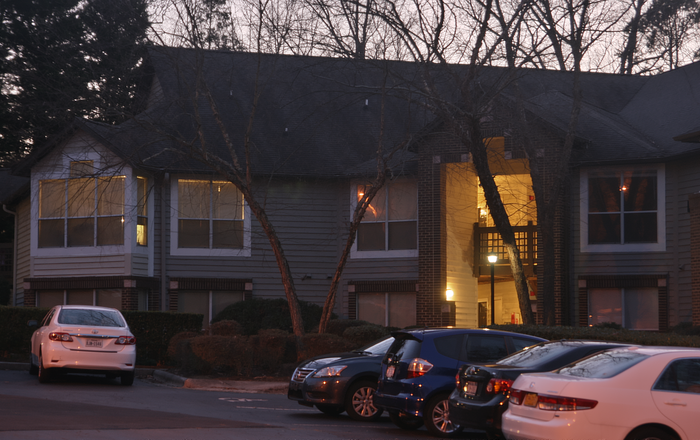}}

\ssmall FBA - exposure fusion
\end{minipage}
\hspace{.4em}
\begin{minipage}[c]{.15\textwidth}
\centering
\cfboxR{1.0pt}{black}{\includegraphics[width=0.99\textwidth]{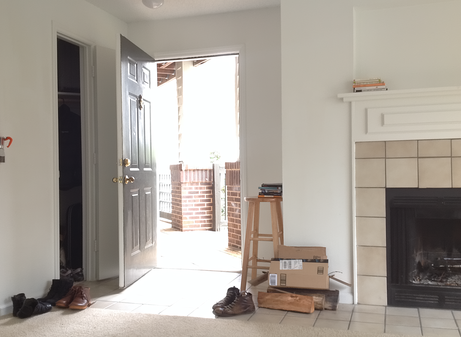}}

\ssmall FBA - long exposure (LE)
\end{minipage}
\begin{minipage}[c]{.15\textwidth}
\centering
\cfboxR{1.0pt}{black}{\includegraphics[width=.99\textwidth]{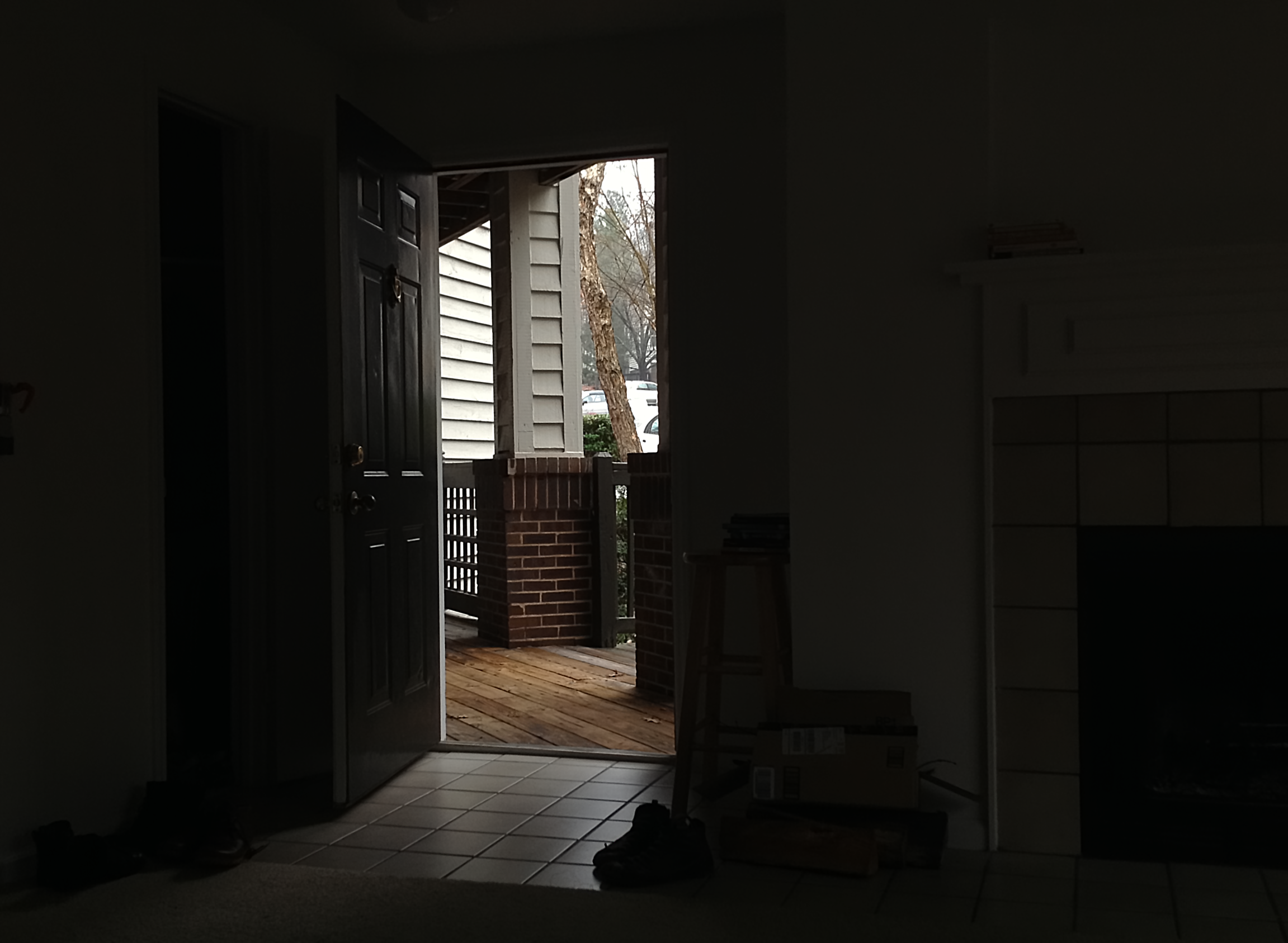}}

\ssmall FBA - short exposure (SE)
\end{minipage}
\begin{minipage}[c]{.15\textwidth}
\centering
\cfboxR{1.0pt}{black}{\includegraphics[width=0.99\textwidth]{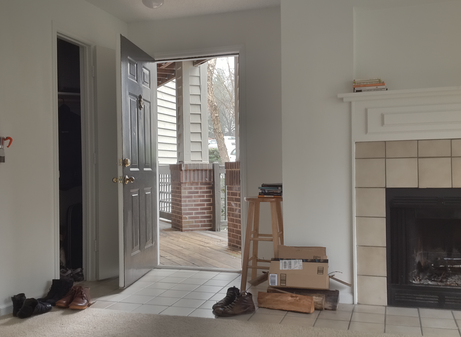}}

\ssmall FBA - exposure fusion
\end{minipage}

\vspace{.8em}

\begin{minipage}[c]{\textwidth}
\centering
\hspace{-0.55em}
\begin{sideways}{\quad \quad {\scriptsize best frame  SE}}\end{sideways}\hspace{-.15em}
\cfboxR{1.3pt}{black}{\includegraphics[width=.153\textwidth]{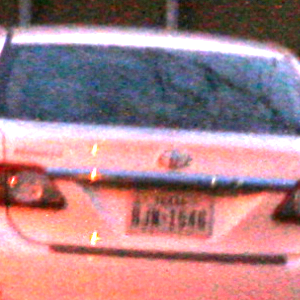}}\hspace{-.1em}
\cfboxR{1.3pt}{black}{\includegraphics[width=.153\textwidth]{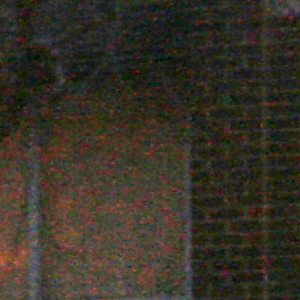}}\hspace{-.1em}
\cfboxR{1.3pt}{black}{\includegraphics[width=.153\textwidth]{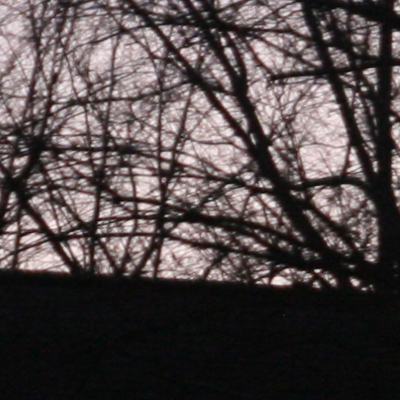}}\hspace{0.3em}
\cfboxR{1.3pt}{black}{\includegraphics[width=.153\textwidth]{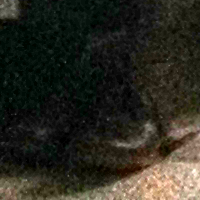}}\hspace{-.1em}
\cfboxR{1.3pt}{black}{\includegraphics[width=.153\textwidth]{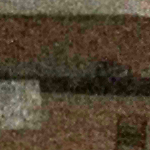}}\hspace{-.1em}
\cfboxR{1.3pt}{black}{\includegraphics[width=.153\textwidth]{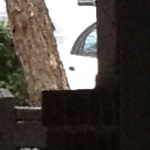}} 
\end{minipage}

\vspace{.1em}

\begin{minipage}[c]{\textwidth}
\centering
\hspace{-0.65em}
\begin{sideways}{\quad  {\scriptsize  typical frame SE}}\end{sideways}\hspace{-.15em}
\cfboxR{1.3pt}{black}{\includegraphics[width=.153\textwidth]{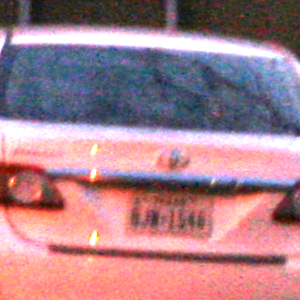}}\hspace{-.1em}
\cfboxR{1.3pt}{black}{\includegraphics[width=.153\textwidth]{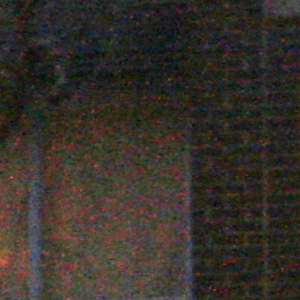}}\hspace{-.1em}
\cfboxR{1.3pt}{black}{\includegraphics[width=.153\textwidth]{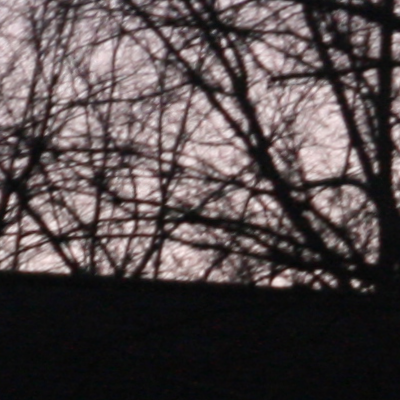}}\hspace{0.3em}
\cfboxR{1.3pt}{black}{\includegraphics[width=.153\textwidth]{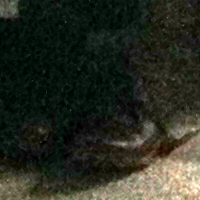}}\hspace{-.1em}
\cfboxR{1.3pt}{black}{\includegraphics[width=.153\textwidth]{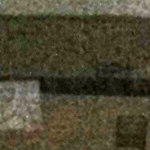}}\hspace{-.1em}
\cfboxR{1.3pt}{black}{\includegraphics[width=.153\textwidth]{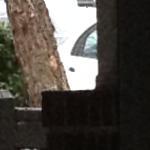}} 
\end{minipage}

\vspace{.1em}

\begin{minipage}[c]{\textwidth}
\begin{sideways}{\qquad  {\scriptsize best frame LE}}\end{sideways}
\cfboxR{1.3pt}{black}{\includegraphics[width=.153\textwidth]{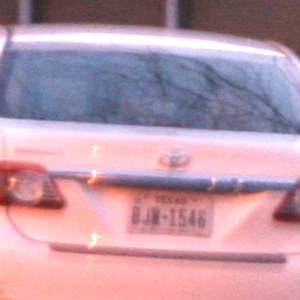}}\hspace{-.1em}
\cfboxR{1.3pt}{black}{\includegraphics[width=.153\textwidth]{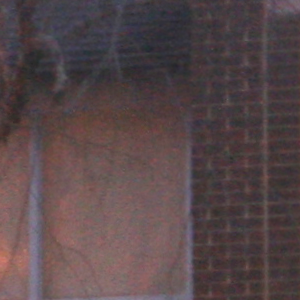}}\hspace{-.1em}
\cfboxR{1.3pt}{black}{\includegraphics[width=.153\textwidth]{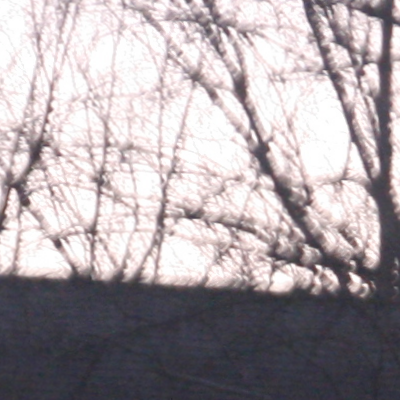}}\hspace{0.3em}
\cfboxR{1.3pt}{black}{\includegraphics[width=.153\textwidth]{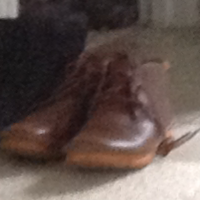}}\hspace{-.1em}
\cfboxR{1.3pt}{black}{\includegraphics[width=.153\textwidth]{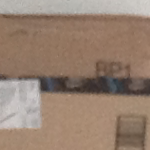}}\hspace{-.1em}
\cfboxR{1.3pt}{black}{\includegraphics[width=.153\textwidth]{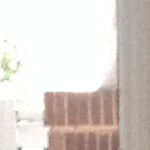}} 
\end{minipage}

\vspace{.1em}

\begin{minipage}[c]{\textwidth}
\begin{sideways}{\,\,\quad {\scriptsize typical frame LE}}\end{sideways}\hspace{-.15em}
\cfboxR{1.3pt}{black}{\includegraphics[width=.153\textwidth]{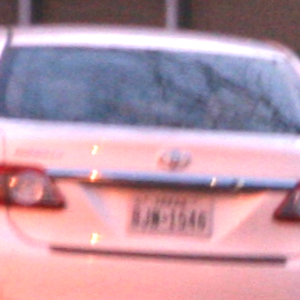}}\hspace{-.1em}
\cfboxR{1.3pt}{black}{\includegraphics[width=.153\textwidth]{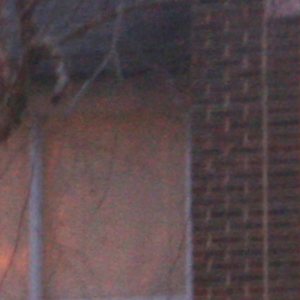}}\hspace{-.1em}
\cfboxR{1.3pt}{black}{\includegraphics[width=.153\textwidth]{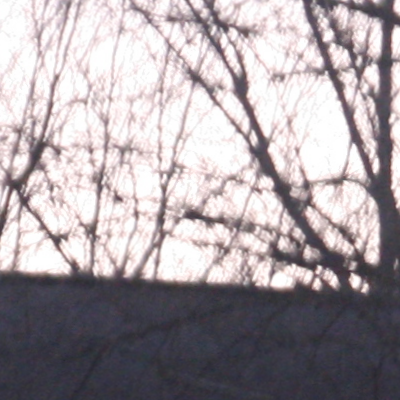}}\hspace{0.3em}
\cfboxR{1.3pt}{black}{\includegraphics[width=.153\textwidth]{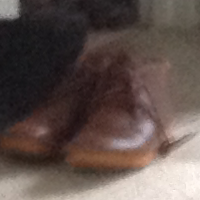}}\hspace{-.1em}
\cfboxR{1.3pt}{black}{\includegraphics[width=.153\textwidth]{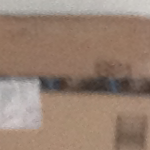}}\hspace{-.1em}
\cfboxR{1.3pt}{black}{\includegraphics[width=.153\textwidth]{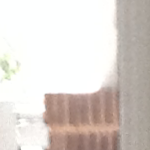}}
\end{minipage}

\vspace{.1em}

\begin{minipage}[c]{\textwidth}
\begin{sideways}{ {\scriptsize align \& average  fusion}}\end{sideways}\hspace{-.15em}
\cfboxR{1.3pt}{black}{\includegraphics[width=.153\textwidth]{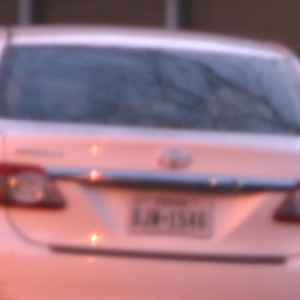}}\hspace{-.1em}
\cfboxR{1.3pt}{black}{\includegraphics[width=.153\textwidth]{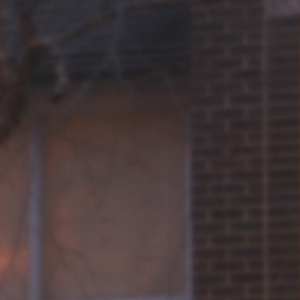}}\hspace{-.1em}
\cfboxR{1.3pt}{black}{\includegraphics[width=.153\textwidth]{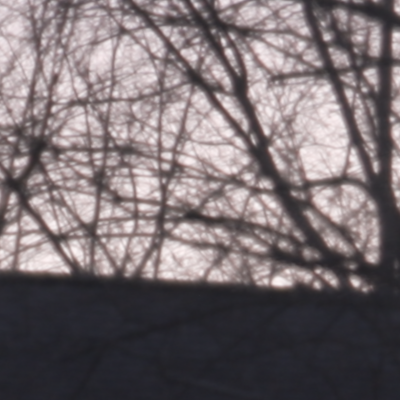}}\hspace{0.3em}
\cfboxR{1.3pt}{black}{\includegraphics[width=.153\textwidth]{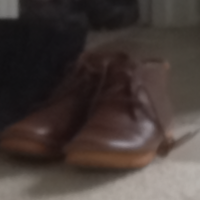}}\hspace{-.1em}
\cfboxR{1.3pt}{black}{\includegraphics[width=.153\textwidth]{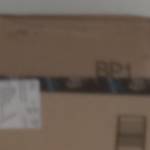}}\hspace{-.1em}
\cfboxR{1.3pt}{black}{\includegraphics[width=.153\textwidth]{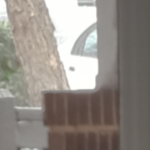}}
\end{minipage}

\vspace{.1em}

\begin{minipage}[c]{\textwidth}
\begin{sideways}{\quad  {\scriptsize best frames fusion}}\end{sideways}
\cfboxR{1.3pt}{black}{\includegraphics[width=.153\textwidth]{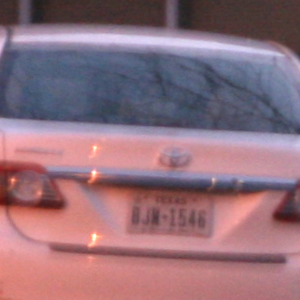}}\hspace{-.1em}
\cfboxR{1.3pt}{black}{\includegraphics[width=.153\textwidth]{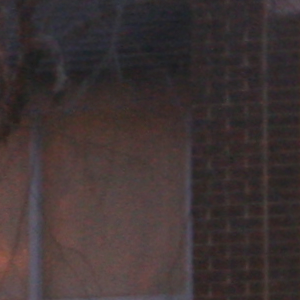}}\hspace{-.1em}
\cfboxR{1.3pt}{black}{\includegraphics[width=.153\textwidth]{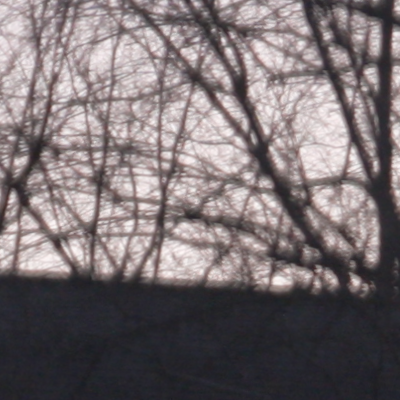}}\hspace{0.3em}
\cfboxR{1.3pt}{black}{\includegraphics[width=.153\textwidth]{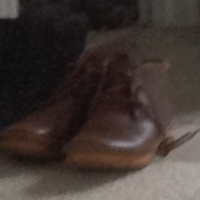}}\hspace{-.1em}
\cfboxR{1.3pt}{black}{\includegraphics[width=.153\textwidth]{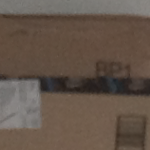}}\hspace{-.1em}
\cfboxR{1.3pt}{black}{\includegraphics[width=.153\textwidth]{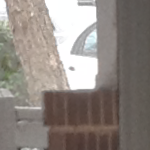}}
\end{minipage}

\vspace{.1em}

\begin{minipage}[c]{\textwidth}
\begin{sideways}{\textbf{\quad\quad \scriptsize FBAs fusion}}\end{sideways}
\cfboxR{1.3pt}{blue1}{\includegraphics[width=.153\textwidth]{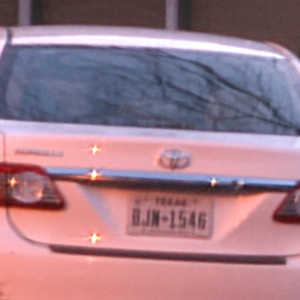}}\hspace{-.1em}
\cfboxR{1.3pt}{blue1}{\includegraphics[width=.153\textwidth]{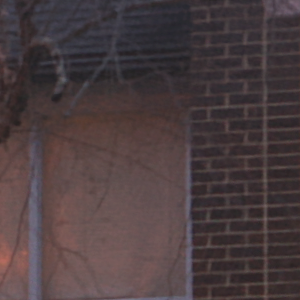}}\hspace{-.1em}
\cfboxR{1.3pt}{blue1}{\includegraphics[width=.153\textwidth]{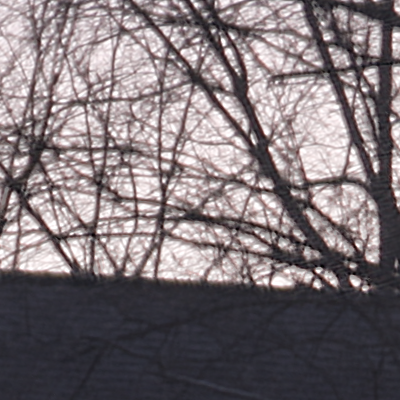}}\hspace{0.3em}
\cfboxR{1.3pt}{blue1}{\includegraphics[width=.153\textwidth]{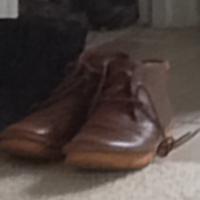}}\hspace{-.1em}
\cfboxR{1.3pt}{blue1}{\includegraphics[width=.153\textwidth]{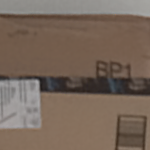}}\hspace{-.1em}
\cfboxR{1.3pt}{blue1}{\includegraphics[width=.153\textwidth]{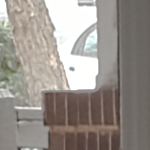}}
\end{minipage}

\end{center}

\caption{HDR burst exposure fusion. Two different examples of fusion of two 12-image bursts captured with two different exposure levels. 
On top, the \textsc{fba} average of each burst  (left and middle) and the HDR fusion of these two images using~\cite{zhang2010gradient} (right). 
Below, image crops of the best shot (sharpest, manually selected) in each burst
and another typical in the series; the exposure fusion using the regular align and average to combine all the images of a burst, the exposure fusion using 
the best shot in each series, and the fusion using the proposed Fourier weighted average. Some of the crops have been rescaled to improve their contrast.}
\label{fig:hdrResults}

\end{figure*}

\section{Conclusions}
\label{sec:conclusions}
We presented an algorithm to remove the camera shake blur in an image burst. The algorithm
is built on the idea that each image in the burst is generally differently blurred; this being a consequence 
of the random nature of hand tremor. By doing a weighted average in the Fourier domain, we  reconstruct 
an image combining the least attenuated frequencies in each frame.
Experimental results showed that the reconstructed image is sharper and less noisy than the original ones.

This algorithm has several advantages. First, it does not introduce typical ringing or overshooting artifacts present in most deconvolution algorithms. 
This is avoided by not formulating the  deblurring problem as an inverse problem of deconvolution. 
The algorithm produces similar or better results than the state-of-the-art multi-image deconvolution while being significantly faster and with lower memory footprint.

We also presented a direct application of the Fourier Burst Accumulation algorithm to HDR imaging with a hand-held camera.
As a future work, we would like to incorporate a gyroscope registration technique, e.g.,~\cite{park2014gyro}, to create
a real-time system for removing camera shake in image bursts.

A very related problem is how to determine the best capture strategy. Giving a total exposure time, would it be more convenient to take 
several pictures with a short exposure (i.e., noisy) or only a few with a larger exposure time (i.e., blurred)? Variants of these questions 
have been previously tackled~\cite{park2014gyro,boracchi2012modeling,zhang2010denoising} in the context of denoising / deconvolution tradeoff.
We would like to explore this analysis using the Fourier Burst Accumulation principle.

\appendix
\label{sec:appendix}
We consider that a burst is correctly aligned if each $v_i$ satisfies 
$v_i = u \star k_i + n_i$,
with the blurring kernel $k_i$ having vanishing first moment. That is,
$\int k_i(\bx) \bx d\bx = 0$. 
This constraint on the kernel implies that the kernel does not drift the image $u$, so each $v_i$ is aligned to~$u$.

An intuitive way to motivate this requirement is by analyzing the result of iteratively applying the blurring kernel a large number of times. 
Let $k(\bx)$ be a non-negative blurring kernel, with $\boldsymbol\mu = \int k(\bx)\bx d\bx$ and 
$\Sigma = \int k(\bx)(\bx - \boldsymbol\mu)(\bx - \boldsymbol\mu)^t d\bx$. Then, 
$$
k^{\star n} := k \star k \star \cdots \star k  \longrightarrow G \left(n\boldsymbol\mu, n\Sigma\right),
$$
where $G \left(n\boldsymbol\mu, n\Sigma\right)$ is a Gaussian function with mean $n\boldsymbol \mu$ and variance $n\Sigma$. 
This is a direct consequence of the
Central Limit Theorem. This Gaussian function can be decomposed into two different components: a centered Gaussian kernel 
(the low pass filter component) and a shifting kernel given by a Dirac delta function, that is,
$$
G \left(n\boldsymbol\mu, n\Sigma\right) = G\left(0,n\Sigma\right) \star \delta_{n\boldsymbol\mu}.
$$
This means that iteratively applying  $n$ times the kernel $k$  is (asymptotically) equivalent to 
applying a Gaussian blur with variance $n\Sigma$ and then shifting the image an amount $n\boldsymbol\mu$.
Thus, we can say that the original kernel $k(\bx)$ drifts the image ``in average'' an amount $\boldsymbol\mu$. 
Therefore, if we do not want the image
to be shifted, the blurring kernel should have zero first moment.
Following this argument all the simulated kernels were centered by forcing $\boldsymbol \mu = \int k(\bx) \bx d\bx = 0$. \vspace{.5em}

\section*{Acknowledgment}
The authors would like to thank Cecilia Aguerrebere and Tomer Michaeli for fruitful comments and discussions. 

\ifCLASSOPTIONcaptionsoff
  \newpage
\fi

\bibliographystyle{IEEEtran}
\bibliography{burst_deblurring_clean_short}

\begin{IEEEbiography}[{\includegraphics[width=1in,height=1.25in,clip,keepaspectratio]{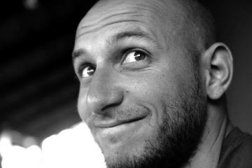} \vspace{4em}}]{Mauricio Delbracio}
received the graduate degree from the Universidad
de la Rep\'{u}blica, Uruguay, in electrical engineering in 2006, the MSc
and PhD degrees in applied mathematics from \'{E}cole normale sup\'{e}rieure de Cachan, France, in 2009
and 2013 respectively. He currently has a postdoctoral position at
the Department of Electrical and Computer Engineering, Duke
University. His research interests include image and signal
processing, computer graphics, photography and computational
imaging.
\end{IEEEbiography}
\vfill
\begin{IEEEbiography}[{\includegraphics[width=1in,height=1.25in,clip,keepaspectratio]{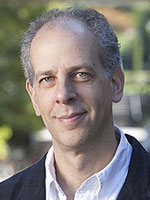}}]{Guillermo Sapiro}
was born in Montevideo, Uruguay, on April 3, 1966. He received his B.Sc. (summa cum laude), M.Sc., and
Ph.D. from the Department of Electrical Engineering at the Technion, Israel Institute of Technology, in 1989, 1991, and 1993 
respectively. After post-doctoral research at MIT, Dr. Sapiro became Member of Technical Staff at the research facilities of HP 
Labs in Palo Alto, California. He was with the Department of Electrical and Computer Engineering at the University of Minnesota, 
where he held the position of Distinguished McKnight University Professor and Vincentine Hermes-Luh Chair in Electrical and 
Computer Engineering. Currently he is the Edmund T. Pratt, Jr. School Professor with Duke University.

G. Sapiro works on theory and applications in computer vision, computer graphics, medical imaging, image analysis, and machine
 learning. He has authored and co-authored over 300 papers in these areas and has written a book published by Cambridge 
 University Press, January 2001.

G. Sapiro was awarded the Gutwirth Scholarship for Special Excellence in Graduate Studies in 1991, 
the  Ollendorff Fellowship for Excellence in Vision and Image Understanding Work in 1992, 
the Rothschild Fellowship for Post-Doctoral Studies in 1993, the Office of Naval Research Young Investigator Award in 1998, 
the Presidential Early Career Awards for Scientist and Engineers (PECASE) in 1998, the National Science Foundation Career Award in 1999, and
the National Security Science and Engineering Faculty Fellowship in 2010.
He received the test of time award at ICCV 2011.

G. Sapiro is a Fellow of IEEE and SIAM.

G. Sapiro was the founding Editor-in-Chief of the SIAM Journal on Imaging Sciences.
\end{IEEEbiography}

\vfill

\end{document}